\documentclass[twoside,11pt]{article}

\usepackage{blindtext}

\usepackage{jmlr2e}
\usepackage{microtype}
\usepackage{graphicx}
\usepackage{subfigure}
\usepackage{booktabs} 
\usepackage{amsmath}
\usepackage{multirow}
\usepackage{caption2}
\usepackage{color}
\usepackage{bm}
\makeatletter
\renewcommand*{\@opargbegintheorem}[3]{\trivlist
	\item[\hskip \labelsep{\bfseries #1\ #2}] \textbf{(#3)}\ \itshape}
\makeatother

\usepackage{hyperref}
\usepackage{amsfonts}
\usepackage{algorithm}
\usepackage{algorithmicx}

\newtheorem{assumption}{Assumption}

\ShortHeadings{learning hyper-parameter prediction function conditioned on tasks}{Shu, Meng and Xu}

\usepackage{lastpage}
\jmlrheading{24}{2023}{1-\pageref{LastPage}}{7/21; Revised
	2/23}{6/23}{21-0742}{Jun Shu, Deyu Meng and Zongben Xu}

\firstpageno{1}

\begin{document}

\title{Learning an Explicit Hyper-parameter Prediction Function Conditioned on Tasks}

\author{\name Jun Shu \email xjtushujun@gmail.com \\
	\addr School of Mathematics and Statistics and Ministry of  Education Key Lab of  Intelligent Networks and Network Security, Xi’an Jiaotong University, Xi'an, Shaan'xi Province, P. R. China\\
	Pazhou Lab (Huangpu), Guangzhou, Guangdong Province, P. R. China
	\AND
	\name Deyu Meng\thanks{Corresponding author} \email dymeng@mail.xjtu.edu.cn \\
	\addr School of Mathematics and Statistics and Ministry of  Education Key Lab of  Intelligent Networks and Network Security, Xi’an Jiaotong University, Xi'an, Shaan'xi Province, P. R. China\\
	Pazhou Lab (Huangpu), Guangzhou, Guangdong Province, P. R. China\\
	Macau Institute of Systems Engineering, Macau University of Science and	Technology\\
	Taipa, Macau, P. R. China.
	\AND
	\name Zongben Xu \email zbxu@mail.xjtu.edu.cn \\
	\addr School of Mathematics and Statistics and Ministry of  Education Key Lab of  Intelligent Networks and Network Security, Xi’an Jiaotong University, Xi'an, Shaan'xi Province, P. R. China\\
	Pazhou Lab (Huangpu), Guangzhou, Guangdong Province, P. R. China
}

\editor{Samuel Kaski}

\maketitle

\begin{abstract}
	Meta learning has attracted much attention recently in machine learning community. Contrary to conventional machine learning aiming to learn inherent prediction rules to predict labels for new query data, meta learning aims to learn the learning methodology for machine learning from observed tasks, so as to generalize to new query tasks by leveraging the meta-learned learning methodology. In this study, we achieve such learning methodology by learning an explicit hyper-parameter prediction function shared by all training tasks, and we call this learning process as \textit{Simulating Learning Methodology} (SLeM). Specifically, this function is represented as a parameterized function called meta-learner, mapping from a training/test task to its suitable hyper-parameter setting, extracted from a pre-specified function set called meta learning machine. Such setting guarantees that the meta-learned learning methodology is able to flexibly fit diverse query tasks, instead of only obtaining fixed hyper-parameters by many current meta learning methods, with less adaptability to query task's variations. Such understanding of meta learning also makes it easily succeed from traditional learning theory for analyzing its generalization bounds with general losses/tasks/models. The theory naturally leads to some feasible controlling strategies for ameliorating the quality of the extracted meta-learner, verified to be able to finely ameliorate its generalization capability in some typical meta learning applications, including few-shot regression, few-shot classification and domain generalization. The source code of our method is released at \url{https://github.com/xjtushujun/SLeM-Theory}.
\end{abstract}


\newpage

\section{Introduction}\label{introduction}

The core goal of machine learning is to learn the inherent generalization rule underlying data from some empirical observations, so as to make label predictions for new query samples. Such a generalization rule is generally modeled as a parameterized function (i.e., learner), and extracted from a pre-specified function set (i.e., learning machine) \citep{jordan2015machine}. In the recent decade, deep learning approaches, equipped with highly parameterized learning machine (deep neural network architectures), have made great successes in a variety of fields \citep{he2016deep,silver2016mastering,devlin2018bert}, strongly substantiating the validity of this learning framework.

However, nowadays gradually more deficiencies have been emerging for this conventional machine learning framework. On the one hand, its success largely relies on vast quantities of pre-collected annotated data, and simultaneously huge computation resources. However, most applications in real world have intrinsically rare or expensive data, or limited computation resources. This inclines to largely degenerate the capability of conventional machine learning, especially deep learning, expected by general users. On the other hand, current deep learning methods are always designed with complicated architectures and possess huge amounts of hyper-parameters, making them easily trapped into the overfitting issues, and possibly perform poorly on the test domain.

The above limitations can be mainly attributed to highly complicated hyper-parametric configurations involved in almost every single component of the learning process in machine learning \citep{jordan2015machine}, e.g., data screening, model constructing, loss function presetting and algorithm designing, etc (see Section \ref{machinelearning} and Table \ref{tablexx}), for handling a practical learning task. The conventional assumption is that these hyparameters are pre-specified when executing machine learning algorithm. However, the specification of hyparameters can drastically affect performance measures like accuracy or data efficiency.
Therefore, seeking a hyper-parameter configuration where the machine learning algorithm will produce a model that generalizes well to new data attracts much resent research attention.

Conventional machine learning literatures have raised many elegant approaches for such hyper-parameter selection issue, like Akaike Information Criterion (AIC) \citep{akaike1974new}, Bayesian Information Criterion (BIC) \citep{schwarz1978estimating}, Minimum Description Length (MDL) \citep{barron1991minimum}, and validation set based approaches \citep{stone1974cross}. Their availability, however, is mostly restricted to the problems with relatively small-scale hyper-parametric structures. When facing complicated and massive hyper-parametric configurations, especially those related to a deep neural network, these concise manners are generally incapable of taking effect. Instead, in most cases the task still highly relies on manual attempts by human experts to the problem. By deeply understanding the problem and accumulating heuristic experience of hyper-parameter tuning, such human-designed manner does be able to make effect to specific tasks. However, when the task is with dynamic variations, the learning method is always required to be re-designed from scratch based on the new understandings of humans to the varying task. Especially, a learning method with carefully modulated hyper-parameters might be excessively good to the investigated learning tasks, while hardly to be readily generalized to a new query task with insightful correlations but also evident variations, like data modalities, network architectures, loss formulations and utilized algorithms, with the trained tasks. Then such laborsome process for hyper-parameter tuning has to be started again, which then inclines to result in most core issues encountered by current machine learning. It has been gradually more widely recognized nowadays that appropriately specifying these hyper-parameters contained in the entire learning process of machine learning algorithm has become significantly more difficult, time-consuming and laborious than directly running a well designed learning process itself for a machine learning algorithm \citep{hospedales2020meta}.

\subsection{Background of Meta-learning Research}

Meta learning, or learning to learn, seeks to improve conventional machine learning by nesting two search problems \citep{hospedales2020meta,franceschi2018bilevel}: at the inner level (or with-task level) we seek a good task-specific model (as in standard machine learning), while at the outer level (or meta level) we seek a good hyper-parameters configuration rather than assuming it is pre-specified and fixed, that ensures the produced model at the inner level to generalize well. The realization strategy is to learn the common hyper-parameter specification principle among a set of training tasks, instead of training samples as conventional machine learning. The aim is to get the hyper-parameter setting rule shared by training tasks, which is thus expected to improve future task learning performance. Such a learning manner can lead to a variety of benefits such as improved data and compute efficiency of machine learning, and being better aligned with human learning, that can learn new concepts quickly from few examples \citep{hospedales2020meta}.
In the recent years, many meta learning studies have been raised, constructed on specific problems, e.g., AutoML \citep{yao2018taking} and algorithm selection \citep{vanschoren2018meta}, or particular applications, e.g., few-shot learning \citep{finn2017model,wang2020generalizing,shu2018small}, neural architecture search (NAS) \citep{elsken2019neural}, or hyper-parameter optimization \citep{franceschi2018bilevel}. Typical researches along this line are listed in Table \ref{tablexx}.

\begin{table*}[t] \vspace{-0mm}
	\caption{Taxonomy of some typical recent literatures on meta learning based on their specified hyper-parameters to learn.}
	\label{tablexx}\vspace{-8mm}
	\centering
	\vskip 0.1in
	\begin{center}
		\begin{tiny}
			\begin{tabular}{l|l}
				\toprule
				Taxonomy	& \hspace{10em} Hyper-parameters to learn in the learning process \\
				\midrule
				\multirow{5}*{Data collection}	& \multirow{5}{*}{\shortstack[l]{data simulator \citep{ruiz2018learning}, dataset distillation \citep{wang2018dataset}, \\  instance weights \citep{shu2019meta,shu2020meta, shu2022cmw,shu2023cmw}, label corrector \citep{wu2020learning,zheng2021meta},\\
						exploration policy \citep{xu2018learning,garcia2019meta}, noise generator \citep{madaan2020learning}\\data annotator \citep{konyushkova2017learning}, data augmentation policy \citep{cubuk2018autoaugment}
				}}   \\
				&           \\
				&           \\
				&              \\
				&              \\
				\midrule
				\multirow{5}*{Model construct}	&  \multirow{5}*{\shortstack[l]{neural architecture \citep{zoph2016neural,liu2018darts}, activation function \citep{ramachandran2017searching}, \\ feature modulation function \citep{ryu2020metaperturb,wang2020meta}, neural modules \citep{alet2019neural}\\ dropout \citep{lee2019meta},
						neural processes \citep{garnelo2018conditional,requeima2019fast},  \\attention \citep{madan2021meta}, batch normalization \citep{bronskill2020tasknorm,du2021metanorm}}} \\
				&           \\
				&           \\
				&           \\
				&            \\
				\midrule
				\multirow{5}*{Loss function preset}	& \multirow{5}*{\shortstack[l]{   loss predictor \citep{houthooftevolved,huang2019addressing,gonzalez2020improved}\\ metric \citep{sung2018learning,lee2018gradient}, auxiliary loss \citep{li2019feature,veeriah2019discovery}, \\ critic predictor \citep{sung2017learning}, regularization \citep{balaji2018metareg,denevi2020advantage}, \\ robust loss \citep{shu2020learning,shu2020meta,ding2023improve,rui2022hyper}}}\\
				&           \\
				&          \\
				&          \\
				&          \\
				\midrule
				\multirow{5}*{Algorithm design}	& \multirow{5}*{\shortstack[l]{gradient generator
						\citep{andrychowicz2016learning,wichrowska2017learned,ravi2016optimization} \\
						initialization \citep{finn2017model,fakoor2019meta,song2019maml,wang2020structured},  \\
						preconditioning matrix \citep{flennerhag2019meta}, curvature \citep{park2019meta}, \\ learning rate schedule \citep{li2017meta,shu2022meta}, minimax optimizer \citep{shen2021learning} }}  \\
				&           \\
				&          \\
				&           \\
				&           \\
				\bottomrule
			\end{tabular}
		\end{tiny}
	\end{center}
		\vspace{-8mm}
\end{table*}

Recent studies on meta learning, however, still have two major issues. Firstly, many of current meta learning works put emphasis on learning fixed hyper-parameters shared from the training tasks, and then directly applying them to adapt to new query tasks. Typical works include MAML \citep{finn2017model} and its variants \citep{li2017meta,nichol2018first,antoniou2018train,finn2019online}, AutoML \citep{yao2018taking}, hyper-parameter optimization \citep{franceschi2018bilevel}, and also the meta learning framework summarized in the recent excellent survey work in \citep{hospedales2020meta}.
Relying on the shared hyper-parameters is challenging for complex task distributions (e.g., those with distribution shift), since different tasks may require to set substantially different hyper-parameters. This makes it always infeasible to find a set of common hyper-parameters performable for all tasks, which may fail to adapt to heterogenous environments of tasks.

Secondly, although some solid theoretical justifications have been presented along this research line, most existing meta learning theories are developed under conventional machine learning framework and put emphasis on evaluating the generalization capability of the traditional learning model (i.e. the learner). And the
most theoretical results are seldom well-aligned with the current meta learning practice with innovatory support/query episodic training mode \citep{vinyals2016matching,finn2017model}\footnote{More details are presented in Section \ref{matheory}.}. More importantly, the current theoretical results have less mentioned theory-inspired controlling strategies for meta-learner to improve the capability of meta learning, and thus are less functioned to feedback meta learning models for helping improve their practical generalization performance.

\subsection{Our Contributions}

To alleviate the aforementioned deficiencies of current meta learning, we can revisit the hyper-parameter setting for machine learning. In fact, if we take the entire learning process as an implicit function\footnote{In many conventional literatures on learning theory, this implicit function is often called a learning algorithm \citep{maurer2005algorithmic,chen2020closer}. Yet it intrinsically represents the whole learning process from input data to output learner. In current machine learning, this process should contain more general learning component settings besides the pre-specified learning algorithm itself, like the designing of network architecture and loss function. To avoid possible clutters of readers, in this paper we call it as the function of learning method, or $\mathcal{LM}(\cdot)$ briefly.} mapping from an input training dataset to a decision model (i.e., a learner extracted from the learning machine by regular machine learning), as shown in Fig.\ref{fig1}(a), these hyper-parameters involved in the entire learning process then constitute the parameters of this function. Note that such hyper-parametric configurations are involved in almost every single component of the learning process in machine learning \citep{jordan2015machine}, e.g., data screening, model constructing, loss function presetting and algorithm designing, etc (see Section \ref{machinelearning} and Table \ref{tablexx}), containing broader essence compared with simple hyper-parameters needed to be tuned in conventional machine learning, like regularization strength or learning rate.
From this perspective, we can interpret such hyper-parameter configuration as the ``learning method'' of the learning task at hand, since their proper settings essentially determine the ultimate capability, especially the generalization, of the extracted learner by this function through implementing the corresponding learning process equipped with the hyper-parameter configuration.

Compared with many current meta learning methods aiming to learn hyper-parameter configurations themselves, we propose to learn an explicit hyper-parameter setting function for predicting proper hyper-parameter configurations when adapted to new query tasks, as shown in Fig. \ref{fig1}(b).
Specifically, equipped with such an explicit function, a novel machine learning system can automatically produce proper learning method for different tasks without need of much extra human intervention. More strictly speaking, what we want is a mapping from the learning task space to the hyper-parameter space covering the whole learning process. From this perspective, learning an automatic hyper-parameter setting function conditioned on tasks is expected to get the ``learning methodology" shared among various learning tasks, and directly be used to allocate learning method for new tasks. In this way, such a ``learning methodology function" is hopeful to be employed to finely adapt varying query tasks from heterogeneous environment with less computation/data costs, as well as fewer human interventions \citep{biggs1985role,schrier1984learning,schmidhuber1987evolutionary,thrun2012learning}. We call this learning process as \textit{Simulating Learning Methodology} (SLeM).

\begin{figure}[t]  \vspace{-10mm}
	\vskip -0.15in
	\centering
	\subfigure[A learning method producing a decision model from an input dataset.]{\label{fig1a}    \includegraphics[width=0.48\columnwidth]{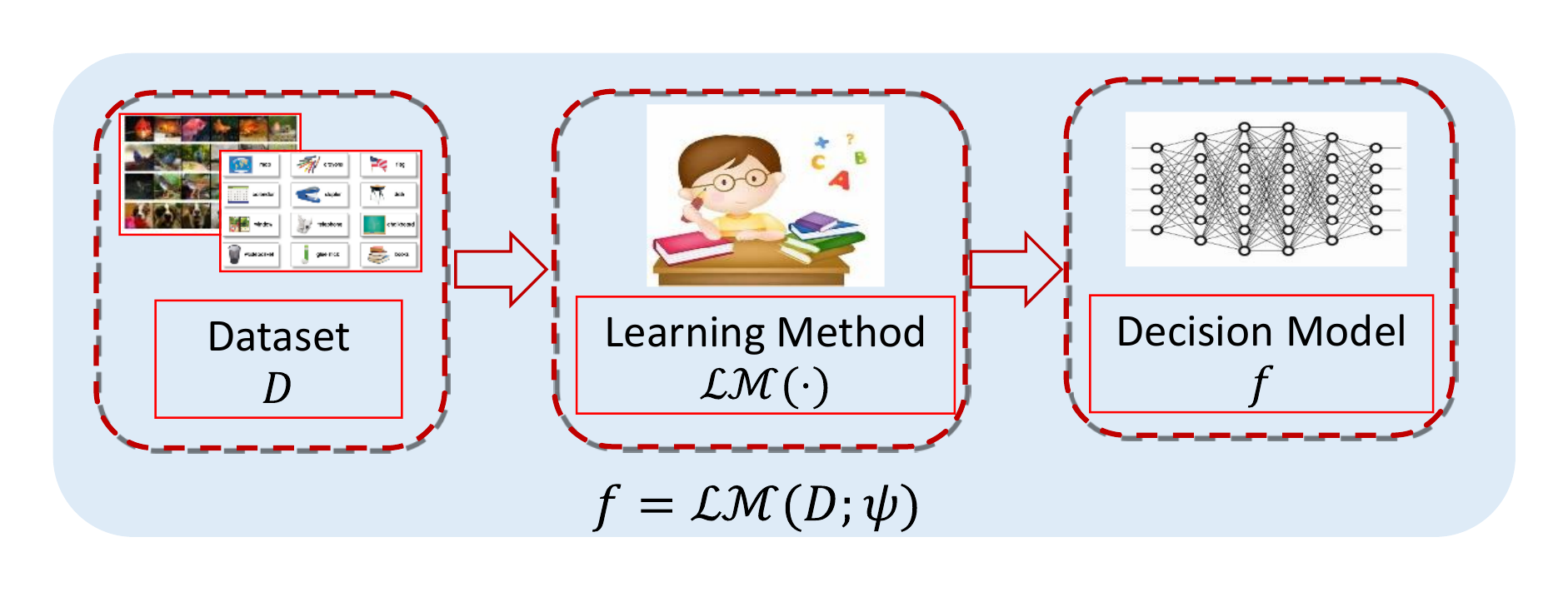}}
	\subfigure[A learning methodology producing a learning method from a query task.]{	\label{fig1b} \includegraphics[width=0.48\columnwidth]{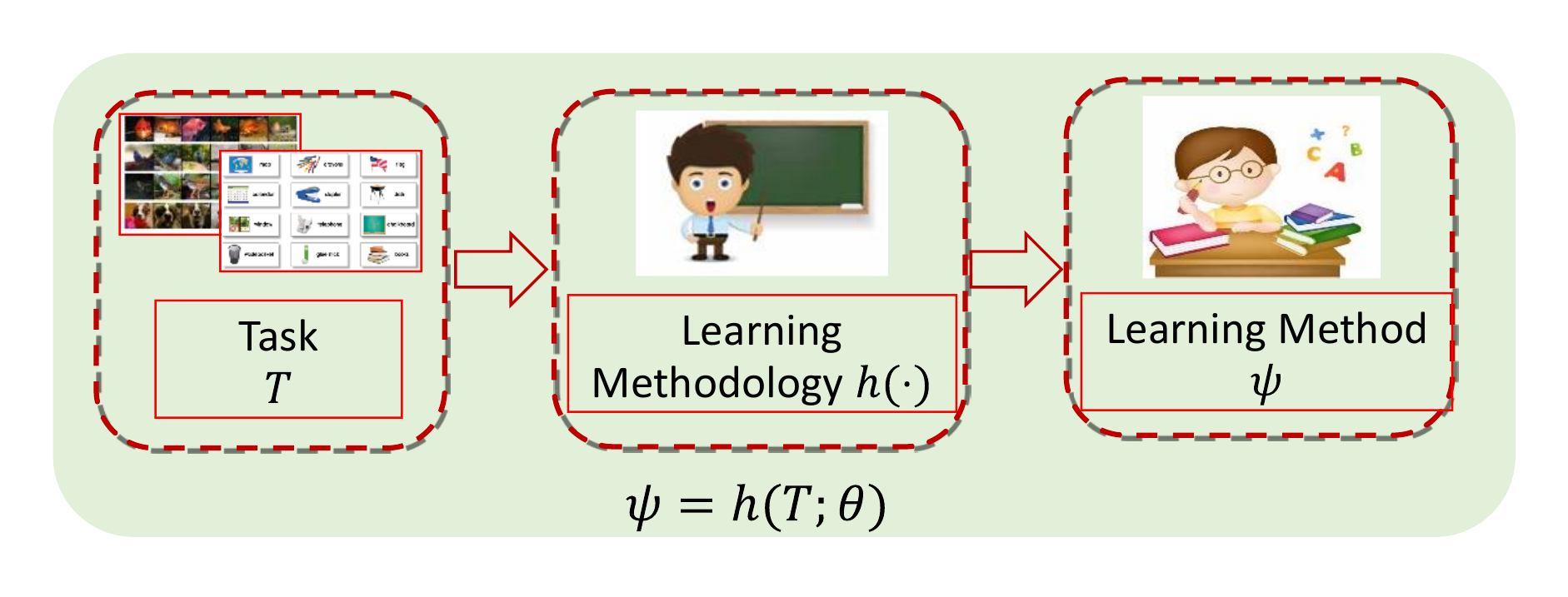}} \vspace{-4mm}
	\caption{(a) depicts the executive process when we deal with a machine learning problem. The learning method can be seen as an implicit function mapping from an input training dataset to a decision model. This function represents the entire learning process equipped with proper hyper-parameter configurations, guaranteeing the machine learning system to produce a decision function by executing it. (b) is the principle of how to determine the learning method. There exists a common learning methodology among various learning tasks, which can be seen as a function that maps an input task to the corresponding learning method. We propose a SLeM framework to learn the learning methodology function.
		Once we achieve the function, it can allocate a proper learning method for a specific query task. The machine learning system can then automatically produce the decision model for the task by running learning pipeline (a), equipped with the learning method produced by the learning methodology function.}  \vspace{-6mm}
	\label{fig1}
\end{figure}

Formally, we can study the statistical guarantee of our novel SLeM framework to understand the generalization capability of such hyper-parameter prediction function.
Especially, just succeeded from the conventional machine learning framework, this function is mathematically modelled as a parameterized function, called meta-learner, mapping from a learning task to its proper hyper-parameter configuration, and extracted from a pre-specified function set, called meta learning machine. On the one hand, such meta-learner setting facilitates a better flexibility of the learned hyper-parameters adaptable to diverse query tasks than fixed hyper-parameters. Specifically, the extracted meta-learner can be readily used to produce proper hyper-parameters conditioned on new query tasks.
On the other hand, such SLeM framework with explicit hyper-parameter prediction function conditioned on learning tasks can be seen as a substantial but homologous extension from the conventional machine learning framework imposed on the pre-specified learning machine, and the corresponding statistical learning theory can be subsequently derived. Especially, similar to the structural risk minimization (SRM) principle in the conventional statistical learning theory \citep{vapnik1999overview,shawe1998structural}, some beneficial theoretical results can then be achieved for revealing and controlling the intrinsic generalization capability of the preset meta learning machine. In summary, the main contributions of this work are as follows:

(1) We propose a new meta learning method aiming at learning an explicit hyper-parameter prediction function conditioned on learning tasks. This formulation provides a general framework to understand meta learning, and reveals that the essential character of meta learning is to construct a meta-learner for simulating the learning methodology, and transferably use it to help fulfill new query tasks.

(2) We introduce a problem-agnostic definition of meta-learner to realize learning the hyper-parameter prediction function conditioned on learning tasks. The parameterized meta-learner can be easily integrated into the traditional machine learning framework to provide a fresh understanding and extension of the original machine learning framework. We further provide generalization bounds for the new SLeM framework with general losses, tasks, and models.

(3) We provide general-purpose bounds for SLeM with decoupling the complexity of learning the task-specific learner from that of learning the task-transferrable meta-learner. The meta-learner is extracted from a pre-specified function set (i.e., meta learning machine), and the theoretical results facilitate some feasible controlling strategies on the meta-learner, capable of being easily embedded into current off-the-shelf meta learning programs and yet helping generally improve its generalization capability.

(4) We highlight the utility of our SLeM framework for obtaining the learning guarantees of some typical meta learning applications, including few-shot regression, few-shot classification and domain generalization. The theory-induced meta-regularization control effects of the meta-learner are empirically verified to be effective for consistently improving its generalization capability on new query tasks.	

The rest of the paper is organized as follows. The proposed SLeM meta learning framework is formulated in Section \ref{framework}, and in Section \ref{matheory}, we derive the generalization bounds of this meta learning manner during the meta-training and meta-test stages, respectively. Related works are introduced in Section \ref{relatedwork}. We further instantiate our general theoretical framework on few-shot regression in Section \ref{application}, few-shot classification in Section \ref{classification}, and domain generalization in Section \ref{domain}. Then Section 8 discusses the online version of our method and its theoretical rationality. The conclusion and discussion are finally made.

\section{Exploring a Task-Transferable Meta-learner for Meta-learning}	\label{framework}
For clarity, in the following we focus on supervised learning. We firstly recall the conventional machine learning framework, and then discuss the deficiencies of existing meta learning methods. Finally, we present the proposed SLeM meta-learning framework and methods for addressing these deficiencies, which can be seen as a substantial but homologous extension from the conventional machine learning framework.

Let's first introduce some necessary notations. $\mathbf{X}^{\mathsf{T}}$ denotes the transpose of a matrix $\mathbf{X}$.  Let $\mathcal{D} = \mathcal{X}\times \mathcal{Y}$ be the data space, where $\mathcal{X} \subset \mathbb{R}^d$ and $\mathcal{Y}\subset \mathbb{R}$ (regression) or $\mathcal{Y}= \{0,1,\cdots,K-1\}$ (multi-class classification) are the input and output spaces, respectively. We use the bracketed notation $[k]=\{1,2,\cdots,k\}$ as shorthand for index sets. The norm $\|\cdot\|$ appearing on a vector or a matrix refers to its $\ell_2$ norm or Frobenius norm.

\subsection{Machine Learning: Learning a Learner from Learning Machine} \label{machinelearning}
Machine learning \citep{jordan2015machine} aims to extract a learner (decision model) from a pre-specified learning machine based on a set of training observations. It generally includes the following ingredients.

\textbf{Training dataset.}
It is usually composed of a finite paired dataset $D=\{(x_i,y_i),i\in[n]\}$ drawn i.i.d. from a task (probability distribution) $\mu$, which simulates the input and output of the learner to be estimated.

\textbf{Learner and learning machine.} The learner corresponds to a mapping  $f(x;\omega):\mathcal{X}\rightarrow\mathcal{Y}$, which is requested to produce a prediction rule from $x \in \mathcal{X}$ to its label $y \in \mathcal{Y}$, and $\omega \in \Omega$ represents the parameters of $f$ to be estimated. $f$ is chosen from a preset learning machine (i.e., hypothesis space) $\mathcal{F}$, constituting a set of candidate learners.

\textbf{Performance measure.} The performance is measured by a loss function $\ell: \mathcal{Y} \times \mathcal{Y} \rightarrow \mathbb{R}_+$, which reflects the extent of how well the learner fits training data. The expected risk $R_{\mu}(f)$ and empirical risk $\hat{R}_{D}(f)$ are respectively defined as:

\begin{equation}\label{emrisk}
R_{\mu}(f) = \mathbb{E}_{(x,y) \sim \mu} \ell(f(x),y),\ \ \
\hat{R}_{D}(f) = \frac{1}{n} \sum_{i=1}^{n} \ell(f(x_i),y_i).
\end{equation}

\textbf{Optimization algorithm.} The learning algorithm $A$ is employed from optimization toolkits to extract a learner $f$ from $\mathcal{F}$ guided by the performance measure.
E.g., the commonly used algorithm for deep learning methods is SGD or Adam \citep{goodfellow2016deep}.

The overall learning process of machine learning is shown in Fig.\ref{fig2a}.
Actually, before executing machine learning for a given learning task in practice, it needs to firstly determine complicated hyper-parametric configurations involved in all components of the learning process.
Then the whole machine learning process can be executed to produce the decision model. If we take the entire learning process as an implicit function mapping from an input training dataset to a decision model, all involved hyper-parameters then constitute the parameters of this function. We can represent this function as $\mathcal{LM}(D; \psi): \mathcal{D} \rightarrow \mathcal{F}$. We denote $\psi$ as all hyper-parameters involved in the learning method\footnote{
In practice, the hyper-parameters to be learned by meta learning generally only contain a certain subset of hyper-parameters among all involved ones in the learning process. Some typical hyper-parameters meta learning aims to learn are listed in Table \ref{tablexx}.
}. For example, we can write $\psi = (\psi_D,\psi_f,\psi_{\ell},\psi_A)$ to record hyper-parameters in the learning configurations with respect to the data $D$, the learner $f$, the loss function $\ell$ and the optimization algorithm $A$, respectively. The ultimate capability, especially the generalization, of the decision model then essentially depends on the proper hyper-parameter settings $\psi$.

After the learning process outputs the decision model, it can produce the label predictions for new query samples. 
Current great successes of deep learning have strongly substantiated the validity of this learning framework for summarizing the underlying label prediction rules from data.
Assume that the optimal learner underlying the task $\mu$ is $f^*= \arg\min_{f \in \mathcal{F}} R_{\mu}(f)$, and the estimated learner through minimizing the empirical risk from the training data $D$ is $\hat{f}= \arg\min_{f \in \mathcal{F}} \hat{R}_{D}(f)$. Then we generally use $R_{\mu}(\hat{f}) - R_{\mu}(f^*)$ to measure the ``closeness'' between the estimated $\hat{f}$ and the optimal $f^*$, whose upper bound has been proved as follows \citep{mohri2018foundations}:
\begin{theorem} \label{s}
Let $\mathcal{D} = \mathcal{X}\times \mathcal{Y}$ be the data space, hypothesis (learning machine) $\mathcal{F}$ be the function class of the mapping $f:\mathcal{X}\rightarrow \mathcal{Y}$, and $\ell: \mathcal{Y}\times \mathcal{Y} \rightarrow [0,B]$ be the loss function. Assume that the loss $\ell(\cdot,y)$ is $L$-Lipschitz for any $y\in\mathcal{Y}$. Then for any $\delta>0$,  with probability at least $1-\delta$, we have
\begin{align}
	R_{\mu}(\hat{f}) - R_{\mu}(f^*) \leq 6 L \mathcal{G}_n (\mathcal{F}) + 2B\sqrt{\frac{\ln(1/\delta)}{2n}},
\end{align}
where $\mathcal{G}_n (\mathcal{F}) := \mathbb{E}_{D\sim\mu^n} [\hat{\mathcal{G}}_{D} (\mathcal{F})] $ is the Gaussian complexity, and $\hat{\mathcal{G}}_{D} (\mathcal{F})$ is the empirical Gaussian complexity of $\mathcal{F}$ w.r.t $D$ defined as
\begin{align*}
	\hat{\mathcal{G}}_{D} (\mathcal{F}) := \mathbb{E}_{\mathbf{g}} \left[\sup_{f\in \mathcal{F}} \frac{1}{n}\sum_{i=1}^n g_i f(x_i)\right], \ g_{i}\sim \mathcal{N}(0,1)
\end{align*}
where $\mathcal{N}(0,1)$ is the Gaussian distribution with zero mean and unit variance.
\end{theorem}
Since $\mathcal{G}_n (\mathcal{F}) \sim \sqrt{C(\mathcal{F})/n}$ \citep{mohri2018foundations}, where $C(\cdot)$ measures the intrinsic complexity of the function class (e.g., VC dimension), Theorem \ref{s} delivers the information that the distance between $\hat{f}$ and $f^*$ will decrease as $n$ becomes large for a given learning machine $\mathcal{F}$. In other word, given enough data, model $\hat{f}$ obtained by minimizing $\hat{R}_{D}(f)$ is arbitrarily close to optimal model $f^*$.

\begin{figure}[t]
\vskip -0.15in
\subfigcapskip=-2mm
\centering
\subfigure[Machine learning framework.]{	\label{fig2a}    \includegraphics[width=0.75\columnwidth]{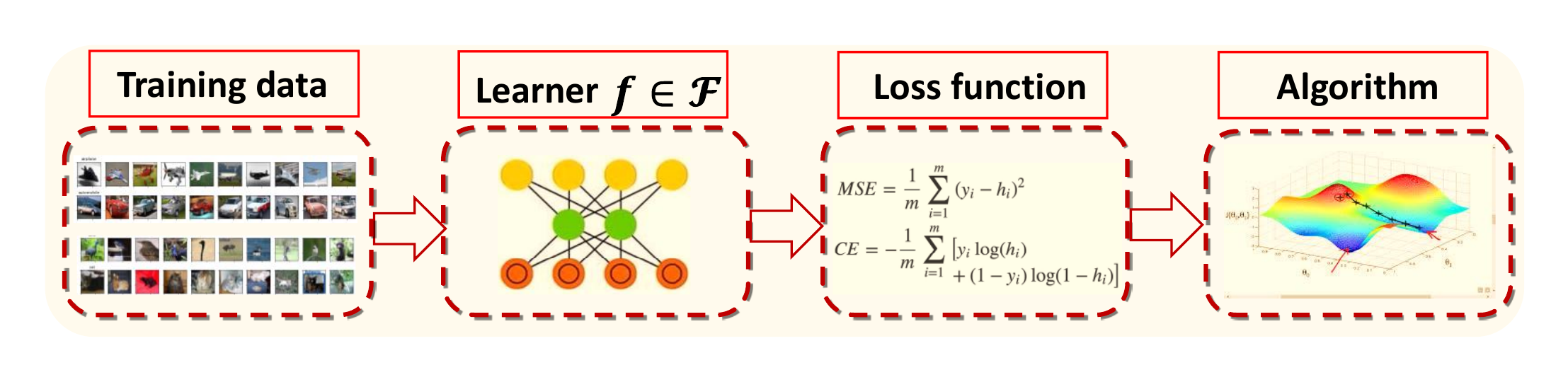} } \\ \vspace{-5mm}
\subfigure[Task-Transferable Meta-learner for Meta Learning.]{	\label{fig2b} \includegraphics[width=0.75\columnwidth]{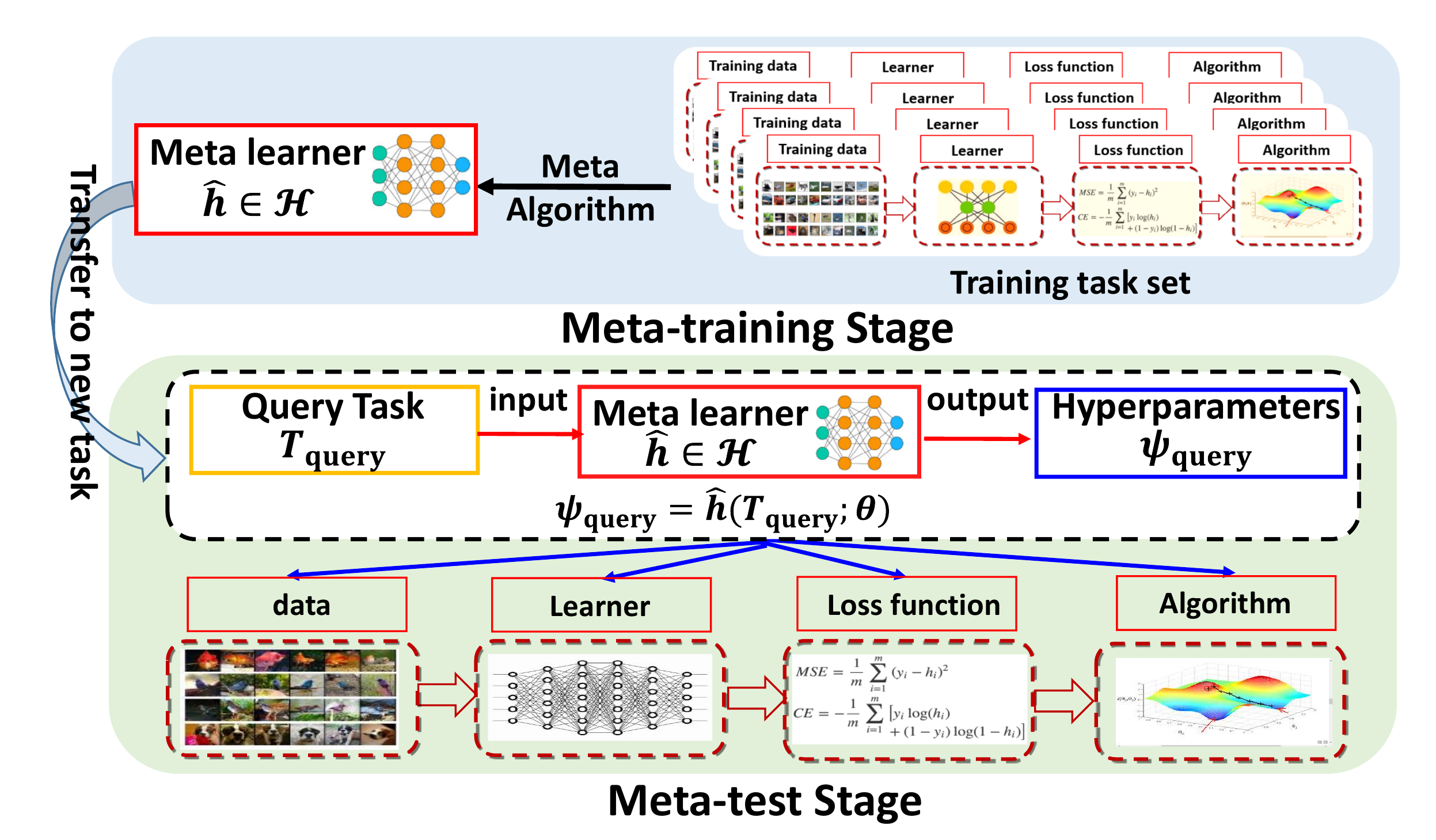}} \vspace{-4mm}
\caption{(a) A general machine learning framework. (b) The meta-training and meta-test stages of the proposed meta learning framework. }
\label{fig2}
\vspace{-6mm}
\end{figure}

\subsection{Deficiencies of Existing Meta Learning Methods}
Despite their success, the effect of contemporary machine learning models, especially deep learning models, largely relies on vast quantities of pre-collected annotated data, and simultaneously huge computation resources. This inclines to significantly degenerate the capability of machine learning for real-world applications with intrinsically rare or expensive data, or limited computation resources. Besides, deep learning methods with complicated architectures and huge amounts of hyper-parameters tend to easily suffer from the overfitting issue, and possibly perform poorly on the test domain.

Recently, meta learning methods are proposed to improve conventional machine learning by distilling the common hyper-parameter specification rule among a set of training tasks, and then using this rule to improve future task learning performance. It can be described as a nested optimization scheme: at the with-task level, an inner algorithm performs task-specific machine learning algorithm with the current hyper-parameter configuration $\psi$, i.e., $\mathcal{LM}(D; \psi)$, where $D$ is the task-specific training dataset. And at the meta-level, a meta-algorithm updates the aforementioned hyper-parameter configuration $\psi$ by leveraging the experience accumulated from the tasks observed so far.
Such a learning manner can lead to a variety of benefits such as improved data and compute efficiency of machine learning, and it is better aligned with human learning, that can learn new concepts quickly from few examples \citep{hospedales2020meta}.
To achieve this goal, recently a large quantity of meta learning methods have been raised, that have helped machine learning improve the data efficiency \citep{shu2018small}, algorithm automation \citep{andrychowicz2016learning,liu2018darts,cubuk2018autoaugment,shu2022cmw,shu2020meta}, and generalization \citep{finn2017model,tripuraneni2020theory}. Typical researches along this line are listed in Table \ref{tablexx}.

Despite their success, most meta learning methods may fail to adapt to heterogenous environments of tasks, in which the complexity of the tasks' distribution cannot be captured by a single fixed hyper-parameter configuration shared from the training tasks. In fact, heterogenous tasks may require to substantially different hyper-parameters. To address this limitation, \cite{denevi2020advantage,deneviconditional,wang2020structured} propose conditional meta learning aiming at learning a conditioning function that maps a task's dataset into a hyper-parameter vector that is appropriate for the task at hand. While they are mainly delicately designed for specific problems (e.g., task conditional model initialization, biased regularization and linear representation, etc), and their conditioning functions are simply set as linear functions, which may limit the applicability to other meta-learning problems.

From the theoretical perspective, most existing meta learning theories along above research line are developed under conventional machine learning framework and put emphasis on evaluating the generalization capability of the traditional learning model (i.e. the learner). And the most theoretical results are seldom well-aligned with the current meta learning practice with innovatory support/query episodic training mode \citep{vinyals2016matching,finn2017model}.
Besides, we notice that the regularization techniques have been effective in improving capability of machine learning, which often impose controlling strategies on the task-specific model (learner) developed from statistical learning theory via structural risk minimization principle. However, in the meta learning framework, most previous theoretical works have seldom emphasized theory-inspired meta-regularization strategies for meta-learner to improve the capability of meta learning, and thus are less functioned to feedback meta learning models for helping improve their practical generalization performance.

\subsection{Proposed Method: Learning a Meta-learner from Meta-learning Machine} \label{metaframework}

To address the deficiencies of existing meta learning methods, as shown in Fig.\ref{fig1b}, we propose to learn an explicit hyper-parameter prediction function for predicting proper hyper-parameter configurations when adapted to new query tasks.
We can mathematically model the function as $h (T; \theta): \mathcal{T} \rightarrow \Psi$, called the meta-learner, which maps from the learning task space $\mathcal{T}$ to the hyper-parameter space $\Psi$. Then the extracted meta-learner can be readily used to produce proper hyper-parameter configurations conditioned on new query tasks.
Therefore, the goal of our novel SLeM meta learning method is to determine the parameter $\theta \in \Theta$ contained in the meta-learner $h$, instead of the hyper-parameter configuration $\psi$ itself like many previous meta learning methods, e.g.,  weights initialization \citep{finn2017model, nichol2018first,rusu2018meta,antoniou2018train,finn2019online}, learning rate \citep{li2017meta,baydin2018online}, and other hyper-parameters \citep{franceschi2018bilevel}. These works have been comprehensively introduced in the recent excellent survey paper  \citep{hospedales2020meta}. Relying on the shared and unique hyper-parameters is challenging for adapting complex heterogenous task distributions, especially those with distribution shift. Comparatively, our hyper-parameter prediction function facilitates a better flexibility of the learned hyper-parameters adaptable to diverse query tasks, so as to effectively address the issue of single fixed hyper-parameters.
In fact, when we set the meta-learner as a constant function, these hyper-parameter-fixed approaches can be categorized as special cases of this more general SLeM meta learning framework.

Different from the learner $f(x;\omega)$ in machine learning, attempting to extract the label prediction rule among data, the meta-learner $h(T; \theta)$ is requested to extract the hyper-parameter prediction rule for learning process shared by training tasks, which is then expected to finely generalize to new query tasks to specify their hyper-parameters for a learning method. Nevertheless, such meta learning framework with explicit hyper-parameter prediction function conditioned on learning tasks can be seen as a substantial but homologous extension from the conventional machine learning framework imposed on the pre-specified learning machine.
To achieve such a meta-learner in practice, we usually assume access to the following ingredients corresponding to the conventional machine learning in Section \ref{machinelearning}:

\textbf{Training task set.} It usually assumes access to a set of source tasks $\Gamma = \{\mu_t=\{\mu_t^s,\mu_t^q\}, t \in [T]\}$, to learn the meta-learner mapping, where $\mu_t$ denotes the probability distribution for the $t$-th task and i.i.d. sampled from an environment distribution $\eta$ \citep{baxter2000model}. Here $\mu_t^s$ and $\mu_t^q$ denote the distributions sampling support/training set $D_t^{tr}$ and query/validation dataset $D_t^{val}$ for the task, respectively. This follows the support/query meta learning setting proposed by \citep{vinyals2016matching}. Specifically, for the $t$-th task, it is composed of $D_t = (D_t^{tr},D_t^{val})$,
where $D_t^{tr} = \{(x_{ti}^{(s)}, y_{ti}^{(s)})\}_{i=1}^{m_t} \sim (\mu_t^s)^{m_t}, D_t^{val} = \{(x_{tj}^{(q)}, y_{tj}^{(q)})\}_{j=1}^{n_t} \sim (\mu_t^q)^{n_t}$, and $m_t$ and $n_t$ are the sizes of training set and validation set for the $t$-th task, respectively. We denote $\bm{D} = \{D_t \}_{t=1}^T$ as the entire training task dataset.

\textbf{Meta-learner and meta learning machine.} Formally, the meta-learner $h (T; \theta): \mathcal{T} \rightarrow \Psi$ is required to extract the rule from $T\in \mathcal{T}$ to its hyper-parameter setting $\psi \in \Psi$ for the learning method, and $\theta \in \Theta$ is the parameter contained in $h$ to be estimated.
$h$ is selected from a pre-designed parameterized function set $\mathcal{H}$, called meta learning machine, constituting a function class of candidate meta-learners $h$.

Generally, the input of a meta-learner $h$ is a representation conveyed by the task $T \in \mathcal{T}$. It can be some static information for task, e.g., directly obtained from the provided dataset $D$ for the task, or dynamic knowledge extracted from the learning process for handling task $T$, e.g., gradient information \citep{andrychowicz2016learning}, sample loss information \citep{shu2019meta}, feature information \citep{li2019feature}, etc. Table \ref{tableexample} shows some typical examples of existing meta learning methods to illustrate the practical parametrization of task representations and the hyper-parameter prediction functions. It is seen that conditional meta learning \citep{denevi2020advantage,deneviconditional} could be a special case under our general meta-learner formulation, where the meta-learner is chose from linear functions class.
In this paper, we mainly focus on the mathematical formulation and the statistical learning theory to understand the proposed meta learning framework, and thus we do not especially emphasize a specific practical parametrized hyper-parameter setting function.
For simplicity and convenience, we formulate the hyper-parameter prediction function as $h({D;\theta})$ in our study, where we assume the representation of task is directly obtained from the provided dataset $D$ for the task at hand, and the hyper-parameter prediction function $h$ is selected from a pre-designed parameterized function set $\mathcal{H}$. For simplicity, we often briefly denote it as $h$.

\textbf{Performance measure.} The goal of SLeM is to learn a common hyper-parameter prediction function suitable for a family of learning tasks. For intuitive understanding, we rewrite the empirical risk $\hat{R}_D(f)$ as $\bm{L}(f,D):\mathcal{F}\times \mathcal{D} \rightarrow \mathbb{R}^{+}$ to characterize the performance of the decision model on the task. Our aim is then to seek the best hyper-parameter prediction function $h$, which is employed to set the with-task level learning process on $D^{tr}$, so that the obtained learner can achieve optimal generalization performance on $D^{val}$. Specifically, the risk so described can be formally written as:
\begin{align} \label{transferrisk}
\begin{split}
R_{\eta}(h) = \mathbb{E}_{\mu \sim \eta}\mathbb{E}_{D^{val}\sim {(\mu^q)}^n} \mathbb{E}_{D^{tr}\sim {(\mu^s)}^m} L(\mathcal{LM}(D^{tr}; h), D^{val}),
\end{split}
\end{align}
where $\mu=(\mu^s,\mu^q)$ is the task distribution sampled from an environment distribution $\eta$,
and $D^{tr}$ and $D^{val}$ are training and validation sets  i.i.d. generated from $\mu^s$ and $\mu^q$, respectively.

In practice, we have access to a set of training task set $\bm{D}$, and the meta-learner is often learned by (approximately) minimizing the following expected empirical risk on the meta learning task:
\begin{align}\label{eqrh}
\hat{R}_{\bm{D}}(h) = \frac{1}{T}\sum_{t=1}^{T} \bm{L} (\mathcal{LM}(D_t^{tr};h), D_t^{val}).
\end{align}

\begin{table*}[t]  \vspace{-4mm}
\caption{ Illustration of the parametrization of tasks and the hyper-parameter prediction function from some typical examples of existing meta learning methods. }
\label{tableexample}\vspace{-8mm}
\centering
\vskip 0.1in
\begin{center}
\begin{tiny}
	\begin{tabular}{l|l|l|l}
		\toprule
		\multirow{2}*{Methods} & \multirow{2}*{Specified hyper-parameters to learn} & \multirow{2}*{Parametrization of tasks} &  Parametrization of \\
		&   &         &     hyper-parameter prediction function \\ \hline
		Conditional meta learning& \multirow{2}*{Bias vector} & \multirow{2}*{Task side information} & \multirow{2}*{Linear functions} \\
		\cite{denevi2020advantage} &   &      & \\ \hline
		MetaOptNet \cite{lee2019metaopt}    &   Feature represention        &Sample original input &  Deep neural networks  \\ \hline
		R2D2 \cite{bertinetto2018meta}&   Feature represention        &Sample original input &  Deep neural networks  \\ \hline
		MMAML \cite{vuorio2019multimodal} & Task modulation vector  & Task embedding & MLP \\ \hline
		LSTM optimizer  & \multirow{2}*{Dynamic update rules of optimizer}  & \multirow{2}*{Gradient information} & \multirow{2}*{LSTM} \\
		\cite{andrychowicz2016learning} &   &   &  \\ \hline
		MW-Net \cite{shu2019meta} & Sample weights & Sample loss information & MLP \\ \hline
		Feature-Critic \cite{li2019feature} & Loss function & Feature information & MLP \\ \hline
		LEO \cite{rusu2018meta} & Model adapted parameters & Training dataset  & Encoding \& decoding network \\ 	\bottomrule
	\end{tabular}
\end{tiny}
\end{center}
\vspace{-4mm}
\end{table*}

\textbf{Optimization algorithm.}
The goal then is to achieve a meta-learner $h$ from $\mathcal{H}$ by minimizing the empirical task average risk $\hat{R}_{\bm{D}}(h)$. There are multiple effective optimization algorithms designed for solving the problem, e.g., the bilevel optimization techniques \citep{finn2017model,shu2019meta,hospedales2020meta}.

Contrary to conventional machine learning whose performance can be directly computed on the given annotated data, meta learning algorithm is relatively hard to directly compute the performance on given tasks. Generally, it needs a two-stage procedure to evaluate the performance of a meta-learner \citep{finn2017model,shu2019meta}. Firstly, a meta-learner predicts the configuration of the hyper-parameters after observing the training set $D_t^{tr}$ in each task of the task set $\bm{D}$.
Then the machine learning system equipped with such hyper-parameter configurations can be executed to automatically produce the proper decision model for $D_t^{tr}$. Then in the second stage, the task loss can be computed as the average over all errors calculated on the validation set $D_t^{val},t\in[T]$ with respect to the corresponding decision model obtained from $D_t^{tr}$, reflecting the hyper-parameter prediction capability of such meta-learner. This process actually is exactly expressed as the empirical task average risk $\hat{R}_{\bm{D}}(h)$ as formulated above.
As shown in Fig.\ref{fig2b}, after achieving the meta-learner $h$, we can further transfer it to new query tasks, which can help set the learning process (i.e., parameter of learning method) for new query tasks.

The above SLeM meta learning framework is evidently succeeded from but also with evident difference with conventional machine learning framework in Section \ref{machinelearning}. To better clarify this point, we list the main notations of both frameworks to easily compare their main differences as introduced in Table \ref{table11}.
It should be noted that in the previous researches, meta learning is interpreted to improve performance by learning `how to learn' \citep{thrun2012learning}. The main study of this work is expected to solidify such  `learning to learn' manner as learning an explicit hyper-parameter prediction function (i.e., a meta-learner) conditioned on learning tasks. From this view, a meta-learner helps formulate the learning method to tell the machine learning system how to automatically learn the decision model. Thus the insight of this formulation is to simulate humans to master the tricks for the learning methodology (thus we call it simulating learning methodology (SLeM)), and most importantly, further transfer it to help fulfill new tasks in the future.
Therefore, achieving the meta-learner is hopeful to be employed to finely adapt varying query tasks from heterogeneous environment with less computation/data costs, as well as fewer human interventions \citep{biggs1985role,schrier1984learning,schmidhuber1987evolutionary,thrun2012learning}.


\begin{table*}[t] \vspace{-4mm}
\caption{Comparison of the main notations used in conventional machine learning and the proposed meta learning frameworks.}
\label{table11}
\vspace{-6mm}
\begin{center}
\begin{tiny}
	\begin{tabular}{l|c|c}
		\toprule
		Compared concepts & Machine Learning & Meta Learning \\
		\midrule
		Source distribution&   Task distribution $\mu$         &    Environment $\eta$        \\ \midrule
		\multirow{3}*{Training instance} &    \multirow{3}*{$D = \{(x_i,y_i)\}_{i=1}^n \sim \mu^n $ }                &   \multirow{3}*{\shortstack[l]{$\bm{D} = \{D_t = (D_t^{tr},D_t^{val})\}_{t=1}^T, D_t^{tr}\sim (\mu_t^s)^{m_t},$\\ $ D_t^{val}\sim (\mu_t^q)^{n_t}, \mu_t = (\mu_t^s,\mu_t^q) \sim \eta$  }}      \\
		&  &  \\
		&  &  \\ \midrule
		Test instance &    $(x,y)\sim \mu $    &  $D_{\mu}^{tr}\sim (\mu^s)^{m_{\mu}}, (x,y) \in \mu^q, \mu = (\mu^s,\mu^q) \sim \eta$                     \\ \midrule
		Learning objective	& Learner $f:\mathcal{X}\rightarrow \mathcal{Y}$&  Meta-learner $h:\mathcal{T} \rightarrow \Psi$\\  \midrule
		Output of the (meta)-learner	& $y = f(\mathbf{x};\omega)$  & $\psi=h(T;\theta)$ \\ \midrule
		\multirow{3}*{Performance measure}	& \multirow{3}*{\shortstack{ $R_{\mu}(f) = \mathbb{E}_{(x,y) \sim \mu} \ell(f(x),y)$}} & \multirow{3}*{\shortstack{$R_{\eta}(h) = \mathbb{E}_{\mu \sim \eta}\mathbb{E}_{D^{val}\sim {(\mu^q)}^n} \mathbb{E}_{D^{tr}\sim {(\mu^s)}^m}$\\ $ \bm{L}(\mathcal{LM}(D^{tr}; h), D^{val})$}} \\
		&        &  \\
		&        &  \\ \midrule
		Goal of the generalization	& Predicting label $y$ for query sample $x$ & Predicting learning method $\psi$ for query task $T$        \\
		\bottomrule
	\end{tabular}
\end{tiny}
\end{center}
\vspace{-6mm}
\end{table*}


\section{Learning Theory}\label{matheory}
We use two phases to characterize the learning paradigm of the above introduced meta learning framework. The first is the meta-training phase, aiming to learn the hyper-parameter prediction function, and the second is the meta-test phase, purposing to generalize this function for setting learning methods on new query tasks. We firstly present some preliminaries, and then present learning theory results for meta-training and meta-test phases, respectively.

\subsection{Preliminaries}\label{Preliminaries}
We use the Gaussian complexity to measure the complexity of a function class. For a generic vector-valued function class $\mathcal{E}$, containing a set of functions $e: \mathbb{R}^d \rightarrow \mathbb{R}^r$. Given a data set $\mathbf{U} = \{\mathbf{u}_i\in \mathbb{R}^d,i \in [N]\}$ i.i.d. drawn from a task $\mu$, the empirical Gaussian complexity is then defined as \citep{bartlett2002rademacher}:

\begin{align*}
\hat{\mathcal{G}}_{\mathbf{U}}(\mathcal{E}) = \mathbb{E}_{\mathbf{g}} \left[ \sup_{e\in \mathcal{E}} \frac{1}{N} \sum_{i=1}^N \sum_{k=1}^{r} g_{ik} e_k(\mathbf{u}_i)   \right], \ g_{ik}\sim \mathcal{N}(0,1),
\end{align*}

\noindent where $\mathbf{g} = \{g_{ik}\}_{i\in[N],k\in[r]}$, and $e_k(\cdot)$ is the $k$-th coordinate of the vector-valued function $e(\cdot)$. The corresponding population Gaussian complexity is defined as $\mathcal{G}_N(\mathcal{E}) = \mathbb{E}_{\mathbf{U}\sim\mu^N}[\hat{\mathcal{G}}_{\mathbf{U}}(\mathcal{E})]$, where the expectation is taken over the distribution of $\mathbf{U}$. In comparison, for empirical Rademacher complexity $\hat{\mathfrak{R}}_{\mathbf{U}}(\mathcal{E})$, $g_{ik}$ are replaced by uniform $\{-1,1\}$-distributed variables, and the corresponding population Rademacher complexity is $\mathfrak{R}_N(\mathcal{E}) = \mathbb{E}_{\mathbf{U}\sim\mu^N}[\hat{\mathfrak{R}}_{\mathbf{U}}(\mathcal{E})]$.

To prove the main learning theory results, we require the following assumptions, all being usually satisfied.
\begin{assumption}[Bounded Inputs] \label{assumption1}
$\mathcal{X} \subset \mathcal{B}(0,R)$, for $R>0$, where $\mathcal{B}(0,R) = \{x\in\mathbb{R}^d: \|x\|\leq R\}$.
\end{assumption}	
\begin{assumption}[Bounded and Lipschitz Loss Function]\label{assumption2}
The loss function $\ell(\cdot,\cdot)$ is $B$-bounded, and $\ell(\cdot,y)$ be $L$-Lipschitz for any $y\in\mathcal{Y}$.
\end{assumption}

\begin{algorithm}[t]
\caption{Meta-Training: Learning the Methodology}
\label{alg1:meta-training}
\begin{algorithmic}
\State {\bfseries Input:} Training task set $\Gamma = \{D_t, t \in [T]\}$.
\State {\bfseries Do:}
\State \ \ \ \ \ {\bfseries (1)} Run algorithm to obtain $\hat{\mathbf{f}}^{(h)}=(\hat{f}_1^{(h)},\cdots, \hat{f}_T^{(h)})$, where $\hat{f}_t^{(h)}$ is obtained by calculating $\mathcal{LM}(D_t^{tr};h(D_t^{tr}))$.
\State \ \ \ \ \ \ \ {\bfseries (2)} Put $\hat{\mathbf{f}}^{(h)}$ into Eq.(\ref{eqrh}), and then minimize the error $\hat{R}_{\bm{D}}({\hat{\mathbf{f}}}^{(h)})$ to obtain $\hat{h}$.
\State {\bfseries Return:} The meta-learner $\hat{h}$.
\end{algorithmic}
\end{algorithm}\vspace{-4mm}

\subsection{Meta-Training Stage: Learning the Hyper-parameter Prediction Function} \label{meta-train}
In the meta-training stage, the training task dataset $\bm{D} = \{D_t, t \in [T]\}$ is available for learning, where $D_t = (D_t^{tr}, D_t^{val})$, $D_t^{tr} \sim (\mu_t^s)^{m_t}, D_t^{val} \sim (\mu_t^q)^{n_t}, \mu_t = (\mu_t^s, \mu_t^q) \sim \eta$. We aim to learn the hyper-parameter prediction function from this task dataset. The overall learning process is listed in Algorithm \ref{alg1:meta-training}.
The quantity of interest in meta-training is the expectation of task average generalization error as follows:
\begin{align} \label{Rtrainrisk}
R_{train}(h) &= \frac{1}{T}\sum_{t=1}^{T}\mathbb{E}_{D_t^{val}\sim {(\mu_t^q)}^{n_t}} \mathbb{E}_{D_t^{tr}\sim {(\mu_t^s)}^{m_t}} \bm{L} (\mathcal{LM}(D_t^{tr};h), D_t^{val}).
\end{align}
Denote the optimal meta-learner obtained by minimizing the theoretical risk $R_{\eta}(h)$ and empirical risk $\hat{R}_{\bm{D}}$ in meta learning as:
\begin{align} \label{tworisk}
h^*=\arg\min_{h \in \mathcal{H}}R_{\eta}(h);\ \ \ \hat{h}=\arg\min_{h \in \mathcal{H}}\hat{R}_{\bm{D}}(h),
\end{align}
where $R_{\eta}(h)$ and $\hat{R}_{\bm{D}}$ are defined in Eqs. (\ref{transferrisk}) and(\ref{eqrh}), respectively. Then we can naturally use $R_{train}( \hat{h})-R_{train}( h^*)$ to measure the ``closeness'' between the estimated $\hat{h}$ and the true underlying $h^*$, which captures the extent of how much the two meta-learners $\hat{h}$ and $h^*$ differ in aggregating over the $T$ training tasks.

We can then present a theoretical upper bound estimation for this ``closeness" measure. Before showing the theorem, we first introduce an important notation, $d_{\mathcal{F}} ({\mu}_t^{s},{\mu}_t^{q})$, necessary for the theoretical result. $d_{\mathcal{F}} ({\mu}_t^{s},{\mu}_t^{q})$ denotes the discrepancy divergence \citep{ben2010theory} between support and query data with respect to their sampled probability distributions ${\mu}_t^{s}$ and ${\mu}_t^{q}$ imposed on  the hypothesis class $\mathcal{F}$:
\begin{align*}
d_{\mathcal{F}} ({\mu}_t^{s},{\mu}_t^{q}) =   \sup_{f \in \mathcal{F}} \left| \mathbb{E}_{D_{t}^{val}\sim (\mu_t^q)^{n_t}}\bm{L} (f, D_t^{val}) - \mathbb{E}_{D_{t}^{tr}\sim (\mu_t^s)^{m_t}}\bm{L} (f, D_t^{tr})  \right|.
\end{align*}

\begin{theorem}\label{the1}
If Assumptions \ref{assumption1} and \ref{assumption2} hold, for any $\delta>0$,  with probability at least $1-\delta$, we have
\begin{align*}
R_{train}( \hat{h})-R_{train}( h^*) & \leq 768L \log(4\sum_{t=1}^T n_t)\cdot L(\mathcal{F})\cdot \hat{\mathcal{G}}_{D^{val}}(\mathcal{H})+ \frac{6L}{T}\sum_{t=1}^{T} \hat{\mathcal{G}}_{D_t^{tr}}(\mathcal{F})\\
+ \frac{4}{T} \sum_{t=1}^{T} d_{\mathcal{F}} (\mu_t^{s},\mu_t^{q}) + & 6\frac{B}{T} \sqrt{\sum_{t=1}^T \frac{1}{n_t}} \sqrt{\frac{\log \frac{2}{\delta}}{2}}  + \frac{6B}{T}\sum_{t=1}^{T} \sqrt{\frac{\log \frac{2}{\delta}}{m_t}} + \frac{12L{Dis(D^{val})}}{(\sum_{t=1}^T n_t)^2} \sqrt{\sum_{t=1}^T \frac{1}{\beta_t T} },
\end{align*}
where $Dis(D^{val})= \sup_{h,h'} \rho_{D^{val}} (\mathbf{f}^{(h)},\mathbf{f}^{(h')}),$
$\rho_{D^{val}} (\mathbf{f}^{(h)},\mathbf{f}^{(h')})=\frac{1}{\sum_{t=1}^T n_t} \sum_{t=1}^{T}\sum_{j=1}^{n_t} (f_t^{(h)}(x_{tj}^{(q)})-f_t^{(h')}(x_{tj}^{(q)}))^2,$
$D^{val} = \{D_t^{val}\}_{t=1}^T$, $\mathbf{f}^{(h)} = (f_1^{(h)}, f_2^{(h)}, \cdots, f_T^{(h)}), f_t^{(h)} = \mathcal{LM}(D_t^{tr};h)$, and $L(\mathcal{F})$ is the Lipschitz constant of $\mathbf{f}^{(h)}$ with respect to $h$, $\beta_t = \frac{n_t T}{\sum_{t=1}^T n_t}$.
\end{theorem}

The above bound for excess task average risk can be more concisely formulated by dominating its three main terms as follows:
\begin{align*}
R_{train}( \hat{h})-R_{train}( h^*)  \lesssim \tilde{\mathcal{O}} \left(\sqrt{ \frac{C(\mathcal{H})}{\sum_{t=1}^T n_t}} +  \frac{1}{T}\sum_{t=1}^T\sqrt{ \frac{C(\mathcal{F})}{m_t}} + \frac{1}{T} \sum_{t=1}^{T} d_{\mathcal{F}} (\mu_t^{s},\mu_t^{q})  \right).
\end{align*}
They can be interpreted as: the complexity of learning the meta-learner $h$ (the first term), the complexity of learning the task-specific learners $\mathbf{f}$ (the second term), and the distribution shift between support and query sets (the third term).
As aforementioned, it generally holds that $\hat{\mathcal{G}}_{D^{val}}(\mathcal{H})\sim\sqrt{C(\mathcal{H})/\sum_{t=1}^T n_t}$, and $\hat{\mathcal{G}}_{D_t^{tr}}(\mathcal{F})\sim \sqrt{C(\mathcal{F})/m_t}$, where $C(\cdot)$ measures the intrinsic complexity of the function class (e.g., VC dimension). Thus the first term on the right hand side above is of the order $\tilde{\mathcal{O}}(1/\sqrt{\sum_{t=1}^T n_t})$, and the second term above is of the order $1/T\sum_{t=1}^T\mathcal{O}(1/\sqrt{m_t}) \leq \mathcal{O}(1/\sqrt{m}), m= \min\{m_1,\cdots,m_T\}$ , where $\tilde{\mathcal{O}}$ denotes an expression that hides polylogarithmic factors in all problem parameters. Note that the leading term capturing the complexity of learning the meta-learner $h$ decays in terms of
the number of query samples ($\sum_{t=1}^T n_t$) among all training tasks. This implies that even with insufficient number of training tasks, improving the amount of query samples in the tasks can also be helpful to the final performance of the extracted meta-learner.
Comparatively, such bounds deduced in some previous works, e.g., \citep{maurer2016benefit}, have mostly not involved this amount, and thus could not reflect such common sense fact that increasing the number of query samples in training tasks should be helpful for
improving generalization of meta learning on new tasks. Actually, in many current meta learning applications, good performance can be obtained on the basis of not very large training task set (even only with one training task), e.g., hyper-parameter learning \citep{franceschi2018bilevel}, neural architecture search (NAS) \citep{elsken2019neural}, etc. Therefore, our result could provide a more rational and comprehensive theoretical explanation for these methods.

Another point is necessary to be illustrated. Many of existing meta learning theories employ traditional empirical error to develop the error bounds \citep{maurer2016benefit,tripuraneni2020theory}, whose training strategy is not episodic (i.e., support/query training strategy) \citep{vinyals2016matching}. This will be somehow inapplicable for modern meta learning algorithm, e.g., gradient-based meta-algorithms \citep{finn2017model,franceschi2018bilevel}.
Besides, some recent theoretical investigations \citep{denevi2019online,balcan2019provable} study model-agnostic meta-algorithm \citep{finn2017model},  and explore the convergence guarantees for gradient-based meta learning. Specifically, \citep{denevi2018learning,denevi2019learning,khodak2019adaptive} pay attention to specific meta-algorithms and propose the corresponding generalization guarantees. However, they study the case that the loss function is convex or the mapping is linear, which might be relatively hard to make analysis on deep neural network. Also, the training strategy studied is not episodic, which tends not to be easily used to train practical popular meta-algorithms. The most related work \citep{yin2019meta,chen2020closer} considered the support/query episodic training strategy. Albeit beneficial to illustrate some generalization insight of meta learning, their theory has not been used to conduct some feasible meta-regularization terms that can help improve the meta-learner.
Different from the aforementioned works, we developed the theoretical results followed by the support/query training strategy. Besides, the theoretical guarantees contain the complexities of learning meta-learner $h$ and learner $\mathbf{f}$, which can easily induce control effects of the meta-learner for improving generalization capability of meta-algorithms.

%

Particularly, it is seen that there exists a term $d_{\mathcal{F}} (\mu_t^{s},\mu_t^{q})$ in the error bound, which describes the distribution shift between training/suppot set and validation/query set among tasks. For applications without such domain shift, this term is zero and can be omitted. However, there are many meta learning applications with such support/query shift.
A typical example is domain generalization (as demonstrated in Section \ref{domain}), aiming to learn how to set a learning method implemented on the source domain data (i.e., support data), but able to yield a learner which could perform well on different target domain data (i.e., query data). To this purpose, various training tasks need to be collected, each containing the simulated training-to-testing domain shift by splitting source domains into virtual training and testing (i.e., validation) domains. And then meta-learning can be executed to improve model training, e.g., MetaReg \cite{balaji2018metareg},  Feature-Critic \cite{li2019feature}.
We believe that our theory with support/query discrepancy can then be used to well explain the theoretical insight of meta learning in such cases.

\subsection{Meta-Test Stage: Generalizing to New Query Tasks}\label{meta-test}
In the meta-test stage, we aim to transfer the extracted meta-learner $\hat{h}$ to help set the learning method on new query tasks. As shown in Algorithm \ref{alg2:meta-testing}, in the meta-test stage, we have the training set $D_{\mu}^{tr} \sim (\mu^s)^{m_{\mu}}$, and for evaluating the performance we also assume to have a validation set
$D_{\mu}^{val}\sim (\mu^q)^{n_{\mu}}$. The two sets are sampled from the query task $\mu =(\mu^{s},\mu^{q}) $ drawn from the environment distribution $\eta$. Apply $\hat{h}$ obtained in the meta-training stage to set the learning process on $D_{\mu}^{tr}$ to obtain the decision model, and the test performance on $D_{\mu}^{val}$ reflects the generalization performance of the meta-learner $\hat{h}$ on the new task $\mu$.
In a nutshell, the test error $R_{test}(h)$ can be calculated as\footnote{ Here $R_{test}$ denotes the meta-test error with respect to the meta-learner, and the missing expectation w.r.t. $D_{tr}$ means that we use the empirical estimation of the task-specific learner to compute the meta-test error on $D_{\mu}^{val}$. Comparatively, $R_{\eta}$ denotes that we use the ideal expected estimation of the task-specific learner to compute the generalization performance.}

\begin{align}
R_{test}(h) = \mathbb{E}_{\mu \sim \eta}\mathbb{E}_{D_{\mu}^{val}\sim (\mu^q)^{n_{\mu}}} \bm{L} (\mathcal{LM}(D_{\mu}^{tr}; h)),  D_{\mu}^{val}).
\end{align}
We further employ the following excess transfer risk to measure the ``closeness'' between the estimated $\hat{h}$ and the true underlying $h^*$, as defined in Eq. (\ref{tworisk}), in terms of helping fulfill new query task $\mu$:
\begin{align*}
R_{test}(\hat{h}) - R_{test}(h^*).
\end{align*}

To build a bridge between the meta-training and meta-test processes, similar as recent works \citep{du2020few,tripuraneni2020theory}, we make the following assumption:
\begin{assumption}[Task diversity]\label{assumption4}
Given the meta learning machine $\mathcal{H}$, and $\hat{h}$ and $h^*$ are defined in Eq. (\ref{tworisk}), it holds that
\begin{align} \label{eqassumption}
R_{\eta}(\hat{h}) - R_{\eta}(h^*) \leq \alpha\left(R_{train}(\hat{h}) - R_{train}(h^*)\right)+\beta,
\end{align}
where $R_{\eta}(h)$ and $R_{train}(h)$ are defined in Eqs. (\ref{transferrisk}) and (\ref{Rtrainrisk}), respectively.
\end{assumption}
In fact, $R_{\eta}(\hat{h}) - R_{\eta}(h^*)$ presents to measure the closeness between the estimated meta-learner $\hat{h}$ and the underlying optimal meta-learner $h^*$ in terms of an arbitrary new query task, and $R_{train}(\hat{h}) - R_{train}(h^*)$ defines this measure by virtue of the information extracted from training tasks.
However, we do not have direct access to the underlying information of the meta-learner. One can only indirectly extract partial information from observed training tasks, and thus we assume Eq.(\ref{eqassumption}) holds, where $\alpha\in \mathbb{R}$ reflects the task similarity between source meta-training tasks and target meta-test task, and $\beta\in \mathbb{R}$ denotes the small additive error.
It is anticipated that transferring the learning methodology will not be possible when the new query tasks are entirely different from those training ones (i.e., $\alpha$ is large).  

In other words, $\alpha$ could essentially encode the ratio of these two quantities, i.e., how well the meta-training tasks can cover the space captured by the $\hat{h}$ needed to predict on new meta-test tasks. So it can be regarded as the task diversity measure of meta-training tasks.
If all training tasks were quite similar, then it could only be expected that the meta-training stage can learn about a narrow slice of the learning methodology, which makes such task-transferring aim difficult.
Generally, when the task diversity is large, then $\alpha$ is small; otherwise, large $\alpha$ implies relatively more similar training tasks.

When we instantiate our SLeM theoretical framework into typical meta learning applications, we could provide some specific value of $\alpha,\beta$.
Specifically, for the few-shot regression model (see Section \ref{application}), we can set $\alpha=\frac{M}{\sigma_{d_L}(\mathbf{K})}$ and $\beta=0$ in Proposition \ref{prop4}, where $\sigma_{d_L}(\cdot)$ denotes the $d_L$-th singular value of a matrix at a decreasing order. Also, for the few-shot classification model (see Section \ref{classification}), we can set $\alpha=\frac{M}{\sum_{k=1}^K  \sigma_{d_L}((\mathbf{K})_k)}$ and $\beta=0$ in Proposition \ref{prop3}. These examples can verify the reasonability and realizability of Assumption \ref{assumption4}.

\begin{algorithm}[t]
\caption{Meta-Test: Generalization to New Query Tasks}
\label{alg2:meta-testing}
\begin{algorithmic}
\State {\bfseries Input:} Query task $\mu=(\mu^s,\mu^q)$ drawn from environment $\eta$; meta-learned meta-learner $\hat{h}$.
\State {\bfseries Do:}
\State \ \ \ \ \ {\bfseries (1)} Draw training set $D_{\mu}^{tr}$ from $\mu^s$.
\State \ \ \ \ \ {\bfseries (2)} For given $\hat{h}$, calculate $\mathcal{LM}(D_{\mu}^{tr}; \hat{h})$ to obtain $\hat{f}_{\mu}$.
\State {\bfseries Return:} The task-specific learner $\hat{f}_{\mu}$.
\end{algorithmic}
\end{algorithm}\vspace{2mm}

We can then present the theoretical result for the meta-test phase of meta learning.
\begin{theorem}\label{theomtest}
If Assumptions \ref{assumption1} - \ref{assumption4} hold,  for any $\delta>0$,  with probability at least $1-\delta$, we have
\begin{align}\label{eqtestbound}
\begin{split}
	& 	R_{test}(\hat{h}) - R_{test}(h^*) \leq \alpha\left(R_{train}(\hat{h}) - R_{train}(h^*)\right)+ \beta \\
	&\ \ \ \ \ \ \ \ \ \ \ \ \ \ \ \ \ \ \ \ \ \ \ \ \ \ \ \ \ \ \ \ + 6L\hat{\mathcal{G}}_{D_{\mu}^{tr}}(\mathcal{F})+2 \mathbb{E}_{\mu \sim \eta} d_{\mathcal{F}} (\mu^s,\mu^q) + 6B\sqrt{\frac{\log \frac{2}{\delta}}{m_{\mu}}}.
\end{split}
\end{align}
\end{theorem}
Theorem \ref{theomtest} provides an excess transfer risk bound of the meta-learned $\hat{h}$ for query task $\mu$, which can be more concisely expressed with three dominant terms as follows:
\begin{align*}
R_{test}(\hat{h}) - R_{test}(h^*) \lesssim \alpha\left(R_{train}(\hat{h}) - R_{train}(h^*)\right) + \mathcal{O}\left(\sqrt{\frac{C(\mathcal{F})}{m_{\mu}}}\right)+ \mathbb{E}_{\mu \sim \eta} d_{\mathcal{F}} (\mu^s,\mu^q).
\end{align*}
They can be interpreted as: one upper bounded by the task average excess risk in the meta-training stage (the first term), the complexity of learning the task-specific learner $f_{\mu}^*$ for a new query task (the second term), as well as the distribution shift between training set and test set of query task (the third term).
It is seen that the excess transfer risk stems from two main components: one arises from the bias of using an imperfect meta-learner estimate to transfer knowledge across training tasks, which accounts for the meta-training risk (i.e., the excess task average risk in Theorem \ref{the1}); the other arises from the difficulty of learning task-specific learner of query task with the estimated meta-learner (i.e., the last two terms of the upper bound).

In a word, the leading-order terms of the transfer risk for learning meta-learner $h$ scales as $\tilde{\mathcal{O}}\left(\sqrt{C(\mathcal{H})/(\sum_{t=1}^T n_t)}+\sqrt{C(\mathcal{F})/m}+ \sqrt{C(\mathcal{F})/m_{\mu}}\right), m = \min\{m_1,\cdots,m_T\}$. A naive algorithm which learns the new task in isolation, ignoring the training tasks, has an excess risk scaling  $\mathcal{O}\left(\sqrt{(C(\mathcal{H})+C(\mathcal{F}))/m_{\mu}}\right)$. This theoretically explained the fact that when $\sum_{t=1}^T n_t$ are sufficiently large compared with $m_{\mu}$ (e.g., the setting of few-shot learning, $m_{\mu}$ is always relatively small), the performance of meta learning should be evidently better than the baseline of learning in isolation.

\subsection{Remarks}
In the following we present two necessary remarks on our method.

\noindent\textbf{The ERM estimation for meta-learner.} We use empirical risk minimization (ERM) principle to derive generalization bounds for our framework with general losses, tasks, and models. In fact, we just assume that the ERM within and outer tasks estimators can be exactly computed, and our theoretical results are shown for the global minimizer of empirical risks, just following previous works, e.g., \citep{maurer2016benefit,tripuraneni2020theory,xu2021representation}. And the qualitative predictions still hold true for gradient-based optimization algorithms as verified by our simulations on deep neural networks.
When the loss function is not convex, we can compute the ERM within and outer task estimators by the aid of advances of bi-level optimization algorithms, e.g., \citep{ji2021bilevel,liu2021value,liu2021towards}.
Specifically, \cite{ji2021bilevel} studied the bi-level optimization where the upper-level objective function is nonconvex, and \cite{liu2021value,liu2021towards} studied the bi-level optimization where inner-level objective function is nonconvex. Our theoretical results are hopeful to benefit from these latest bi-level optimization algorithms, and the bounds could be further improved with relatively loose conditions, which will be investigated in our future research.

\noindent\textbf{Theory-induced meta-regularization.} Since we introduce an explicit parameterized meta-learner to extract the hyper-parameter prediction function, it is easy to control and improve the learning of the meta-learner with proper meta-regularization techniques, just like regularization techniques being effective in improving capability of machine learning, which often impose controlling strategies on the task-specific model (learner) developed from statistical learning theory via structural risk minimization principle \citep{vapnik2013nature}, e.g., the large margin regularizer for learner.
To bridge the gap between meta-learning theory and its practical use, we will highlight the utility of our meta learning framework for obtaining the learning guarantees of some typical meta learning applications, including few-shot regression (Section \ref{application}), few-shot classification (Section \ref{classification}) and domain generalization (Section \ref{domain}). We can empirically substantiate that the theory-induced meta-regularization strategies are capable of helping consistently improve generalization capability of meta-learner on new query tasks.

\section{Related Works}\label{relatedwork}
While there are a large number of meta learning literatures recently emerging in the field, most of them focused on practical feasible techniques against certain problems. In this section we mainly review the related studies considering the intrinsic understanding of the fundamental meta learning concepts as well as its basic learning theories. More related papers can be referred to in the recently proposed comprehensive surveys \citep{vanschoren2018meta,hospedales2020meta}.

\noindent\textbf{Meta learning understandings and taxonomies}.
Meta learning has a long history in psychology \citep{ward1937reminiscence}, and was described by \citep{maudsley1980theory} as ``the process by which learners become aware of and increasingly in control of habits of perception, inquiry, learning, and growth that they have internalized". Then \cite{schmidhuber1987evolutionary,bengio1990learning} introduced it into computer science to train a meta-learner that learns how to update the parameters of the learner. Afterwards, this approach has been applied to learning to optimize deep networks \citep{andrychowicz2016learning}.  Besides,
\cite{vilalta2002perspective} used meta learning to improve the learning bias dynamically through experience by the continuous accumulation of meta-knowledge. \cite{lemke2015metalearning} further employed meta learning to extract meta-knowledge from different domains or problems. A recent survey paper of \citep{vanschoren2018meta} mainly attributed the capability of meta learning to its leveraging prior learning experience, and interpreted that it can learn new tasks more quickly. These literatures, however, have not presented a general mathematical formulation for meta learning, which could yet be useful to clearly help distinguish meta learning from previous related learning manners, like transfer learning and multi-task learning.
Very recently, \cite{hospedales2020meta} proposed a comprehensive survey paper, introducing a taxonomy of meta learning along three independent axes, i.e., meta-representation, meta-optimizer, and meta-objective. However, as aforementioned, this paper summarized the task of meta learning as learning fixed hyper-parameters instead of a hyper-parameter-prediction-function, making the learned methodology with insufficient flexibility adapt to new query tasks.

\noindent\textbf{Conditional meta-learning}.  Several recent works \citep{wang2020structured,denevi2020advantage,deneviconditional,rusu2018meta,wang2019tafe} addressed the issue of learning fixed hyper-parameters by conditioning hyper-parameters on target tasks. \cite{wang2020structured} used a task-specific objective functions by weighting meta-training data on target tasks based on structured prediction, to achieve task adaptive model initialization. TASML \citep{wang2020structured} is a non-parametric approach and requires access to training tasks at test time for the task-specific objectives on target tasks. Considering the limitation for non-convex formulation of TASML, \cite{denevi2020advantage} proposed another conditional meta-learning framework for biased regularization and fine tuning, which learns a conditioning function mapping task's side information into a task's specific bias vector. \cite{deneviconditional} further proposed to learn a conditioning function mapping task's side information into a linear representation for better performance in more scenarios. To sum up, they are against specific problems (task conditional model initialization, biased regularization and linear representation, respectively), which can be actually seen as special hyper-parameter configurations cases under the investigated general hyper-parameter setting framework (as can be seen in Table 1). Besides, our meta-learner is a general parameterized structure, e.g., we consider deep neural networks as the meta-learner in our experiments. Thus their meta-learner could be a special case under our general meta-learner formulation, e.g., the meta-learner is chose from linear functions class.

Our SLeM framework can be easily integrated into the traditional machine learning framework to provide a fresh understanding and extension of the original machine learning framework. This facilitates us capable of readily establishing learning theory upon the traditional statistical learning theory, and provide generalization bounds for the new framework with general losses, tasks, and models. Comparatively, \citep{wang2020structured,denevi2020advantage,deneviconditional} have not specifically emphasized the relationship with conventional machine learning regime, and they mainly pay attention to the advantageous performance of conditional meta-learning approach compared with the standard unconditional	counterpart. Specifically, similar to the structural risk minimization (SRM) principle in the conventional statistical learning theory, we can further analyze general theory-inducted control strategies for meta-learner to ameliorate its generalization capability according to the derived generalization bounds. While \citep{wang2020structured,denevi2020advantage,deneviconditional} does not develop corresponding theory-inducted controlling strategies based on their formulation and theory.

\noindent\textbf{Theory of gradient-based meta-learning}. GBML stems from the Model-Agnostic Meta-Learning (MAML) algorithm \citep{finn2017model} and has been widely used in practice. \cite{finn2019online} showed that MAML meta-initialization is learnable via follow the meta leader method under strong-convexity and smoothness assumptions. \cite{balcan2019provable,denevi2019learning} provided finite-sample meta-test-time performance guarantees in the convex setting for SGD-based Reptile algorithm or biased regularization via online learning. ARUBA \citep{khodak2019adaptive} further improved upon the bound of \citep{balcan2019provable,denevi2019learning}, and obtained better statistical guarantees. And \cite{denevi2019online} explored Online-Within-Online meta-learning approach via online mirror descent algorithm, and derived a cumulative error bound for their method. \cite{fallah2020convergence} developed a convergence analysis for one-step MAML for a general nonconvex objective, and \cite{wang2020global,ji2022theoretical} studied the convergence of multi-step MAML cases.
\cite{wang2021bridging,kao2022maml,collins2022maml} provided theoretical analysis for two well-known GBML methods, MAML and ANIL.
Recently, \cite{fallah2021generalization} studied the generalization properties of MAML using the connections between algorithmic stability and generalization bounds of algorithms.	
\cite{huang2022provable} analyzed the generalization properties of MAML trained with SGD in the overparameterized regime under a mixed linear regression model. \cite{yao2021improving} proposed task augmentation strategies to address meta-overfitting problem. \cite{chen2022understanding} further showed that overfitting is more likely to happen in MAML and its variants than in ERM, especially when the data heterogeneity across tasks is relatively high.

For all these GBML based methods, the bounds presented in the related literatures are restricted to MAML algorithms and its variants, i.e., unconditional meta-learning, which, however, may be likely to fail to adapt to heterogenous environments of tasks. Comparatively, in this study, we learn an explicit hyper-parameter prediction function conditioned on training tasks, aiming at more flexibly dealing with dynamic and changing environments of tasks. We then provide general-purpose bounds with decoupling the complexity of learning the task-specific learner from that of learning the task-transferrable hyper-parameter prediction function. Thus such new bounds should have a more comprehensive adaptability on real complicated environments. This can also be evidently verified by the superiority of our method compared with MAML-based ones in our experiments.

\noindent\textbf{Statistical learning to learn  (LTL)}. The seminal work of \citep{baxter2000model} introduced a framework to study the statistical benefits of meta learning, subsequently used to show excess risk bounds for ERM-like methods using techniques like covering numbers \citep{baxter2000model}, algorithmic stability \citep{maurer2005algorithmic,chen2020closer} and Gaussian complexities \citep{maurer2016benefit}. Afterwards, \cite{du2020few,tripuraneni2020theory,tripuraneni2021provable} proposed a general framework for obtaining the generalization bound, and made analyses on the benefit of representation learning for reducing the sample complexity of the target task. \cite{arora2020provable} proposed a representation learning approach for imitation learning via bilevel optimization, and demonstrated the improved sample complexity brought by representation learning.  	
\cite{zhou2019efficient} statistically demonstrated the importance of prior hypothesis in reducing the excess risk via a regularization approach. \cite{chua2021fine} further provided risk bounds on predictors found by fine-tuning via gradient descent, using an initial representation from a MAML-like algorithm.
\cite{bai2021important} studied the implicit regularization effect of the practical	design choice of train-validation splitting popular in meta learning.
\cite{sun2021towards} showed that the optimal representation for representation-based meta learning is overparameterized and provides sample complexity for the method of moment estimator. \cite{bernacchia2021meta} discovered that the optimal step size of overparameterized MAML during training is negative.

The PAC-Bayes framework has been extended to understand this LTL approach,
facilitating  a PAC-Bayes meta-population risk bound \citep{pentina2014pac,pentina2015lifelong,amit2018meta,ding2021bridging,farid2021generalization,rothfuss2021pacoh,rothfuss2021meta,nguyen2022pac,rezazadeh2022unified}. These works
mostly focus on the case where the meta learning model is underparameterized; that is, the total number of meta training data from all tasks is larger than the dimension of the model parameter.  Comparatively, our meta learning model is problem-agnostic, which could be overparameterized or underparameterized.
Besides, information theoretical bounds have been recently proposed in \citep{chen2021generalization,jose2021information,jose2021transfer,hellstromevaluated}, which bound the generalization error in terms of mutual information between the input training data and the output of the meta-learning algorithms. Particularly, \cite{hellstromevaluated} applied their bounds to a representation learning setting, and yields a bound coincides with the one reported in \citep{tripuraneni2020theory}.

Compared with these bound estimation theories, the generalization bounds we derive have several differences. Firstly, previous works assume the traditional learning model to train the meta-learner, which is somehow deviated from the current meta learning practice with support/query episodic training mode such as MAML \citep{finn2017model} and ProtoNet \citep{snell2017prototypical}.
Comparatively, we derive the bounds based on the commonly used support/query meta-training manner in meta learning practice. Besides, although previous works achieve solid theoretical justifications, they seldom emphasized theory-inspired meta-regularization terms to conduct practically feasible controlling strategies for further ameliorating the meta-learner performance, especially its generalization capability across diverse testing query tasks. In our study, however, our estimated bounds could be readily employed to induce such expected generalization-controlling strategies for general meta-learning algorithms. Specifically, our theory can be used to deduce several meta-regularization strategies on soundly rectifying the learning track of the meta-learner to improve its generalization ability. The effectiveness of such strategy has been comprehensively substantiated with typical meta learning applications on few-shot regression, few-shot classification and domain generalization, as demonstrated in Sections \ref{application}, \ref{classification} and \ref{domain}, respectively.

\noindent\textbf{Hyper-parameter tuning/optimization}. The hyper-parameter tuning/optimization techniques \citep{feurer2019hyperparameter} often search some relatively small-scale hyper-parameters like regularization strength, learning rate, etc. While when faced with complicated and massive hyper-parametric configurations, especially those related to a deep neural network, conventional hyper-parameter tuning approaches are generally incapable of taking effect. Yet our framework aims to study highly complicated hyper-parametric configurations constituted in general components of the learning process, e.g., data screening, model constructing, loss function presetting and algorithm designing, etc. Please see Table \ref{tablexx} in the main manuscript, which lists some typical examples on this motivation. The hyper-parameter setting issues considered in this study thus significantly exceed those involved in conventional hyper-parameter tuning literatures.

Besides, hyper-parameter tuning techniques mainly focus on extracting hyper-parameters themselves for the investigated tasks, and often employ validation set based approaches to search proper hyper-parameters, e.g., random search \citep{bergstra2012random}, Bayesian hyper-parameter optimization \citep{snoek2012practical}, gradient-based hyper-parameter optimization \citep{franceschi2018bilevel}, etc.
While the framework emphasized in this study employs an explicit hyper-parameter prediction function (meta-learner) for predicting proper hyper-parameters, which can more flexibly and adaptively fit query tasks. Besides, it often adopts bi-level optimization tools to calculate hyper-parameters under a sound optimization framework, which is always more efficient than purely searching strategies.

Furthermore, the hyper-parameter tuning techniques usually focus on finding optimal hyper-parametric configurations for tasks at hand. Despite their success, the searched hyper-parameters should be mainly appropriate for the investigated learning tasks, but can be hardly finely generalized to wider range of learning problems with heterogenous environments, which may require to set substantially different hyper-parameters. Comparatively, the explored hyper-parameter setting function in the suggested meta-learning framework of this study aims to extract the hyper-parameter setting rule shared by training	tasks sampled from a task distribution, which is expected to facilitate a better flexibility of the learned hyper-parameters setting function performing well on new tasks sampled from this task distribution. Thus this hyper-parameter setting function should potentially own better task generalization ability, which can be readily used to predict proper hyper-parameters conditioned on new query tasks. Especially, we have provided the theoretical evidence to support such task transfer generalization capability of our hyper-parameter prediction function, to make this essential characteristics more solid and convincible.

\noindent\textbf{Multi-task/Transfer Learning}.
Here we want to shortly clarify the main difference of the meta learning framework as aforementioned from conventional multi-task/transfer learning approaches \cite{weiss2016survey,zhuang2020comprehensive,pan2010survey,zhang2018overview,ruder2017overview}.
Although the latter also makes use of the intrinsic relationship among multiple tasks for improving the generalization performance for the learning results, they still fall into the scope of conventional machine learning, which considers the problem at the data/learner level. Specifically, most conventional multi-task/transfer learning works aim to improve the task learning performance by virtue of ameliorating the parameters of multiple learners imposed on different tasks by leveraging their underlying relevance. Meta-learning, however, considers the problem at the task/meta-learner level, and aims to learn the common hyper-parameter setting function imposed on the parameters of one single meta-learner functioned on different tasks. This way inclines to better improve its flexibility, available range and potential capability in practice than conventional learning manners. In particular, a feasible meta learning method allows the learners on different training/test tasks with very different forms, like variant input data modalities \citep{cubuk2018autoaugment}, model architectures \citep{zoph2016neural}, feature dimensions \citep{ryu2020metaperturb}, etc, and only requires them to possess certain shared common knowledge in learning method setting, which yet should be hardly handled by conventional multi-task/transfer learning manners. Such consideration in the learning-methodology perspective inclines to make this learning manner be with a more widely potential usage than many previous machine learning fashions.

\section{Application \uppercase\expandafter{\romannumeral 1}: Few-shot Regression}\label{application}
In this section, we instantiate the proposed meta learning framework for the few-shot regression model.

\subsection{Basic Setting}
Here we consider the few-shot regression setting where $\mathcal{X}=\mathbb{R}^d, \mathcal{Y}=\mathbb{R}$. Let $\|x\| \leq R,|y|\leq B, \forall x \in \mathcal{X}, y\in \mathcal{Y}$. The output of the meta-learner is the representation of each sample.
The loss functions $\ell$ is chosen as the square loss $\ell(\mathbf{w}^{\mathsf{T}} {h}(x),y)=(y-\mathbf{w}^{\mathsf{T}} {h}(x))^2$.
The task-specific learning machine $\mathcal{F}$ are chosen as linear regression function maps, and the meta learning machine $\mathcal{H}$ is parameterized as a depth-$L$ vector-valued deep neural network (DNN) to extract the common representation for various regression tasks. Concretely, $h(x)$ can be written as:
\begin{align}\label{eqmlp}
	h(x) = \phi_{L}(\mathbf{W}_{L}(\phi_{L-1}(\mathbf{W}_{L-1}\cdots \phi_1 (\mathbf{W}_1 x)))),
\end{align}
where $\mathbf{W}_k,k\in [L]$ is the parameter matrix, and $\phi_k, k\in[L]$ is the activation function. Here we assume that the activation functions in all layers are element-wise $1$-Lipschitz and zero-centered as assumed in \citep{golowich2018size}
which is typically satisfied, like the ReLU. Formally, $\mathcal{F}$ and $\mathcal{H}$ are writen as:
\begin{align}\label{eqfewregression}
	\begin{split}
		\mathcal{F}&=\{f | f(z)= \mathbf{w}^{\mathsf{T}} z, \mathbf{w},z\in \mathbb{R}^{d_L}, \|\mathbf{w}\| \leq M  \},\\
		\mathcal{H} &= \{h|h(x) \in \mathbb{R}^{d_L} \text{ as defined in (\ref{eqmlp})}, x \in \mathcal{X}\}.
	\end{split}
\end{align}

Following the setting in \citep{finn2017model}, we assume that there are $T$ tasks $\bm{D} = \{D_t\}_{t=1}^T$ available for learning, and $D_t = (D_t^{tr},D_t^{val})$, where $D_t^{tr} = \{z_{ti}^{(s)}=(x_{ti}^{(s)},y_{ti}^{(s)})\}_{i=1}^{m}, D_t^{val} = \{z_{tj}^{(q)}=(x_{tj}^{(q)},y_{tj}^{(q)})\}_{j=1}^{n}$.
Here, the number of training/validation set samples for each task is identical.
The few-shot regression model is then written as:
\begin{align}\label{ffew}
	\begin{split}
		&\bm{W}^* = \arg\min_{\bm{W}} \frac{1}{nT}\sum_{t=1}^T \sum_{j=1}^n \ell(\mathbf{w}_t^{*\mathsf{T}} h(x_{tj}^{(q)}),y_{tj}^{(q)}) \\
		s.t.,\  &\mathbf{w}_t^{*} = \arg\min_{\mathbf{w}_t} \frac{1}{m}\sum_{i=1}^m \ell(\mathbf{w}_t^{T} h(x_{ti}^{(s)}),y_{ti}^{(s)}), t\in [T],
	\end{split}
\end{align}
where $\bm{W} = \{\mathbf{W}_{k}, k\in [L]\}$ represents the collection of parameter matrices of $h$ and $\mathbf{w}, h$ is chosen from $\mathcal{F},\mathcal{H}$ as defined in Eq. (\ref{eqfewregression}). For notation convenience, we denote $\mathbf{P} = (\mathbf{w}_1,\cdots,\mathbf{w}_T)^{\mathsf{T}} \in \mathbb{R}^{T\times d_L}$, and $\mathbf{P^*} = (\mathbf{w}_1^*,\cdots,\mathbf{w}_T^*)^{\mathsf{T}}$ as its theoretical optimal solution.

\subsection{Theoretical Analysis}\label{theory}
In this part, we will instantiate Theorem \ref{theomtest} for few-shot regression model as defined in Eq. (\ref{ffew}).
The $Dis(D^{val})$ can be computed as
\begin{align*}
	Dis(D^{val}) &= \sup_{h,h'} \rho_{D^{val}} (\mathbf{f}^{(h)},\mathbf{f}^{(h')}) \leq 4 \sup_{h,x \in \mathcal{X}} \|\mathbf{w}^{\mathsf{T}} h(x)\|\\
	& \leq \sup_{h,x} 4M\|h(x)\| \leq 4MD \cdot \prod_{k=1}^{L} \|\mathbf{W}_k\|,
\end{align*}
where $\rho_{D^{val}}(\cdot,\cdot) $ is defined in Theorem 2, and
$\|h(x)\|$ is bounded by (Let $ \mathbf{r}_k(\cdot)$ denote the vector-valued output of the $k$-layer for $k\in [L]$):
\begin{align} \label{eqhbound}
	\begin{split}
		&\| h(x)\| = \| \mathbf{r}_L(x)\|  = \|\phi_L (\mathbf{W}_{L} \mathbf{r}_{L-1}(x))\| \\
		\leq & \|\mathbf{W}_{L}\mathbf{r}_{L-1}(x)\| \leq \|\mathbf{W}_{L}\| \|\mathbf{r}_{L-1}(x)\| \leq D \prod_{k=1}^{L} \|\mathbf{W}_k\|.
	\end{split}
\end{align}

Now we can verify that Assumptions \ref{assumption1} - \ref{assumption4} holds.
Observe that $\nabla_{\hat{y}}\ell(\hat{y},y) = 2(\hat{y}-y) \leq 2(B+M\|h(\mathbf{x})\|)$, and thus the loss function is Lipschitz continuous with $L=2(B+M\|h(x)\|$, and $\|h(x)\|$ is bounded by Eq.(\ref{eqhbound}). Besides, since $|\ell(\hat{y},y)| \leq B^2 + 2BM\|h(x)\| +M^2\|h(x)\|^2$, the loss function is bounded.
The following result verifies that Assumption \ref{assumption4} holds,
\begin{proposition}\label{prop4}
	Consider the few-shot regression model defined in Eq.(\ref{ffew}), and the loss function $\ell(\cdot,\cdot)$ is chosen as the squared loss. The feature representation and the linear regression function are designed as in Eq.(\ref{eqfewregression}). Then Assumption \ref{assumption4} holds with $\alpha=\frac{M}{\sigma_{d_L}(\mathbf{K})}$ and $\beta=0$, where $\mathbf{K}=\mathbf{P^*}^{\mathsf{T}} \mathbf{P^*}/T$, $\sigma_{d_L}(\mathbf{K})$ denote the $d_L$-th singular value of matrix $\mathbf{K}$ at a decreasing order.
\end{proposition}

It can be seen that $\alpha$ reflects the diversity of training task set. The larger the similarity of task-specific learners is (i.e., the smaller $\sigma_{d_L}(\mathbf{K})$ is), the large $\alpha$ value is. Namely, the larger the similarity of task-specific learners is, the harder the meta-learner could be transferably used to new query tasks.

Now we can compute the leading-order terms in Eq.(\ref{eqtestbound}) for the parameterized meta-learner and learners as defined in Eq.(\ref{eqfewregression}). We firstly show the Rademacher complexity of $\mathcal{H}$ in the following theorem.
\begin{theorem}[Theorem 1 in \citep{golowich2018size}] \label{MLP}
	Let $\mathcal{H}$ be the class of real-valued DNN as defined in Eq.(\ref{eqmlp}) while requiring $d_L = 1$ over $\mathcal{X}=\{\mathbf{x}: \|\mathbf{x}\|\leq R\}$, where each parameter matrix $\mathbf{W}_i$, $i\in[L]$ has Frobenius norm at most $B_i$, and the activation function $\phi_i,i\in[L]$ is 1-Lipschitz, with $\phi_i(0)=0$, and applied element-wise. Then we have:
	\begin{align*}
		\hat{\mathfrak{R}}_{N}(\mathcal{H}) \leq \frac{R \left(\sqrt{2\log(2)L}+1\right)\prod\limits_{i=1}^L B_i}{\sqrt{N}}.
	\end{align*}
\end{theorem}

(1) For the Gaussian complexity of meta-learner $h$:
\begin{align} \label{ss}
	\begin{split}
		\hat{\mathcal{G}}_{D^{val}}(\mathcal{H})  &= \mathbb{E} \sup_{h\in \mathcal{H}} \frac{1}{nT} \sum_{t=1}^{T} \sum_{j=1}^{n}\sum_{k=1}^{d_L} g_{tjk} h_k(x_{tj}^{(q)}) \\
		&\leq\sum_{k=1}^{d_L}  \hat{\mathcal{G}}_{D^{val}}(h_k)  \leq 2\sqrt{\log(nT)}\sum_{k=1}^{d_L}  \hat{\mathfrak{R}}_{D^{val}}(h_k) \\
		&\leq 2d_L \sqrt{\log(nT)}  \cdot \frac{R \left(\sqrt{2\log(2)L}+1\right)\prod\limits_{i=1}^L B_i}{\sqrt{nT}},
	\end{split}
\end{align}
where the second inequality holds since $\hat{\mathcal{G}}_{\mathbf{X}}(\mathcal{H}) \leq 2 \sqrt{\log(N)} \hat{\mathfrak{R}}_{\mathbf{X}}(\mathcal{H})$ for any function class $\mathcal{H}$ when $\mathbf{X}$ has $N$ data points \citep{ledoux2013probability}.

(2) For the Gaussian complexity of the task-specific learner:
\begin{align}
	\begin{split}
		&\hat{\mathcal{G}}_{D_t^{val}}(\mathcal{F})= \mathbb{E} \sup_{\mathbf{w}\in \mathcal{F}} \frac{1}{m} \sum_{i=1}^{m}g_{ti} \mathbf{w}^{\mathsf{T}} h(x_{ti}^{(s)})  \\
		&\leq \frac{\|\mathbf{w}\|}{m} \mathbb{E} \left\|\sum_{i=1}^{m}g_{ti} h(x_{ti}^{(s)}) \right\| \leq \frac{\|\mathbf{w}\|}{m}\sqrt{\sum_{i=1}^{m} \left\| h(x_{ti}^{(s)})\right\|^2 }\\
		& \leq \frac{\|\mathbf{w}\|}{\sqrt{m}} \cdot \max_{x_{ti}^{(s)} \in D_t^{tr}} \|h(x_{ti}^{(s)} )\|.
	\end{split}
\end{align}

Thus, the transfer error defined in Eq.(\ref{eqtestbound}) now can be written as
\begin{align}\label{fewtest}
	\begin{split}
		&  \ \ \ \ \ \ \ \ R_{test}(\hat{f}_{\mu},\hat{h}) - R_{test}(f^*_{\mu},h^*) \\
		& \ \ \ \ \leq \frac{M}{\sigma_{d_L}(\mathbf{K})} \left( 768L \log(4nT) L(\mathcal{F}) 2d_L \sqrt{\log(nT)}  \frac{R \left(\sqrt{2\log(2)L}+1\right)\prod\limits_{i=1}^L B_i}{\sqrt{nT}} \right. \\&\phantom{=\;\;}\left.
		+ \frac{6L\|\mathbf{w}\|}{\sqrt{m}T} \!\sum_{t=1}^T \!\max_{x_{ti}^{(s)} \!\in\! D_t^{val}} \|h(x_{ti}^{(s)} )\| + 6B\! \sqrt{\frac{\log \frac{2}{\delta}}{2nT}}  + 6B\!\sqrt{\frac{\log \frac{2}{\delta}}{m}} + \frac{48L\sup_{h,x} M\|h(x)\|}{n^2T^2}\right) \\
		&\ \ \ \ \ + \frac{6L\|\mathbf{w}\|}{\sqrt{m_{\mu}}} \cdot \max_{x_{i}^{(s)} \in D_{\mu}^{val}} \|h(x_{i}^{(s)} )\|+ 6B\sqrt{\frac{\log \frac{2}{\delta}}{m_{\mu}}},
	\end{split}
\end{align}
where
$\sigma_1(\mathbf{X}), \cdots,\sigma_r(\mathbf{X})$ denote the singular values of a rank $r$ matrix $\mathbf{X}$ at a decreasing order.
Note that we assume that there exists no distribution shift between the training and validation sets, and thus the term $d_{\mathcal{F}} (D_t^{(tr)},D_t^{(val)})$ is zero and can be omitted.

\subsection{Theory-inspired Meta-Regularization} \label{regular}
It can be seen that the complexity of the meta-learner is normal without extra term needed to be restricted, while the complexity of task-specific meta-learner has additional term $\max_{x_{ti}^{(s)} \in D_t^{val}} \|h(x_{ti}^{(s)} )\|$, which is dependent on the output range of the meta-learner. The transfer error bound in Eq.(\ref{fewtest}) can then naturally inspire the following three controlling strategies for improving the generalization capability of the extracted meta-learner $h$.

(\romannumeral 1) Control the output range of the meta-learner $h$. Conventional models usually set all activation functions of meta-learner easily as ReLU. If we revise the last activation function as Tanh (i.e., $\phi_{L}=\frac{e^z-e^{-z}}{e^z+e^{-z}}$), then we have
\begin{align*}
	\| h(\mathbf{x})\| = \| \mathbf{r}_L(\mathbf{x})\|  = \|\phi_L (\mathbf{W}_{L} \mathbf{r}_{L-1}(\mathbf{x}))\| \leq \sqrt{d_L}.
\end{align*}
Generally, $\sqrt{d_L}$ is smaller than $D \prod_{k=1}^{L} \|\mathbf{W}_k\|_F$, and the complexity can thus be expected to substantially decrease, and the generalization of the calculated meta-learner is hopeful to be improved (more analysis and discussion could be found in Appendix \ref{tanh}).

(\romannumeral 2) Minimize $\|\mathbf{w}\|$ of the learners. The terms $\frac{3L\|\mathbf{w}\|}{\sqrt{m}T}\sum_{t=1}^T \max_{x_{ti}^{(s)} \in D_t^{val}} \|h(x_{ti}^{(s)} )\|$ and $\frac{3L\|\mathbf{w}\|}{\sqrt{m_{\mu}}}$ imply that minimizing the $\|\mathbf{w}\|$ also tends to decrease the transfer error in Eq.(\ref{fewtest}).

(\romannumeral 3) Maximize the $\sigma_{d_L}(\mathbf{K})$. The term $\frac{M}{\sigma_{d_L}(\mathbf{K})}$ accounts for the gravity of the meta-training error influencing the final transfer error in Eq.(\ref{fewtest}), and thus maximizing the $\sigma_{d_L}(\mathbf{K})$ also inclines to reduce the transfer error.

It should be emphasized that the above training strategy (\romannumeral 1) corresponds to a meta-regularized control manner imposed on the meta-learner, while the training strategies (\romannumeral 2) and (\romannumeral 3) put controls on the task-specific learners. And they could lead to different training behaviors as can be observed in the next subsection.

The training strategy (\romannumeral 1) is easy to be implemented  by directly replacing the last activation function of learners from conventional ReLU to Tanh function.
The training strategy (\romannumeral 2) can be achieved by adding a $\|\mathbf{w}\|$ regularizer into the meta-training objective in Eq. (\ref{ffew}) as:
\begin{align*}
	\begin{split}
		&\bm{W}^* = \arg\min_{\bm{W}} \frac{1}{nT}\sum_{t=1}^T \sum_{j=1}^n \ell(\mathbf{w}_t^{*\mathsf{T}} h(\mathbf{x}_{tj}^{(q)}),y_{tj}^{(q)}) \\
		s.t.,\  &\mathbf{w}_t^{*} = \arg\min_{\mathbf{w}_t} \frac{1}{m}\sum_{i=1}^m \ell(\mathbf{w}_t^{\mathsf{T}} h(\mathbf{x}_{ti}^{(s)}),y_{ti}^{(s)}) + \lambda_1\|\mathbf{w}_t\|^2, t\in[T].
	\end{split}
\end{align*}
In the meta-test stage, it solves the following objective for a new query task given $\hat{h}$:
\begin{align*}
	\mathbf{w}^{*} = \arg\min_{\mathbf{w}} \frac{1}{M_0}\sum_{i=1}^{M_0} \ell(\mathbf{w}^{\mathsf{T}} \hat{h}(\mathbf{x}_{i}^{(s)}),y_{i}^{(s)}) + \lambda_2\|\mathbf{w}\|^2.
\end{align*}
The training strategy (\romannumeral 3) can be realized by solving the following meta-training objective:
\begin{align*}
	\begin{split}
		&{\bm{W}}^* = \arg\min_{\bm{W}} \frac{1}{nT}\sum_{t=1}^T \sum_{j=1}^n \ell(\mathbf{w}_t^{*\mathsf{T}} h(\mathbf{x}_{tj}^{(q)}),y_{tj}^{(q)}) \\
		s.t.,\  &\mathbf{P}^* = \arg\min_{\mathbf{P}} \frac{1}{mT}\sum_{t=1}^T\sum_{i=1}^m \ell(\mathbf{w}_t^{\mathsf{T}} h(\mathbf{x}_{ti}^{(s)}),y_{ti}^{(s)})- \lambda_3 \sigma_{d_L}(\mathbf{P}^{\mathsf{T}} \mathbf{P}/T).
	\end{split}
\end{align*}
And the $\lambda_1,\lambda_2,\lambda_3$ are the hyper-paramters of the above regularization problems.

\begin{figure*}[t]\vspace{-2mm}
	\centering
	\subfigure[$m=1,n=5$]{\includegraphics[width=1\textwidth]{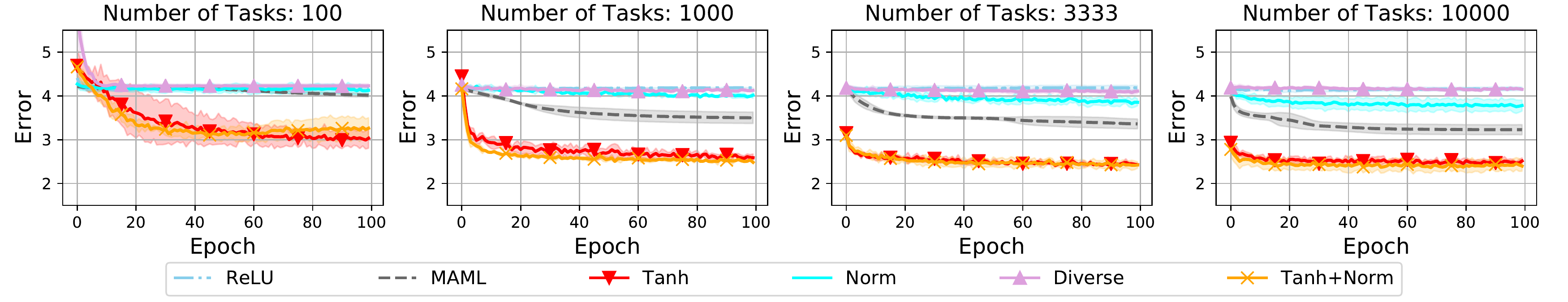} } \\ \vspace{-3mm}
	\subfigure[$m=5,n=1$]{\includegraphics[width=1\textwidth]{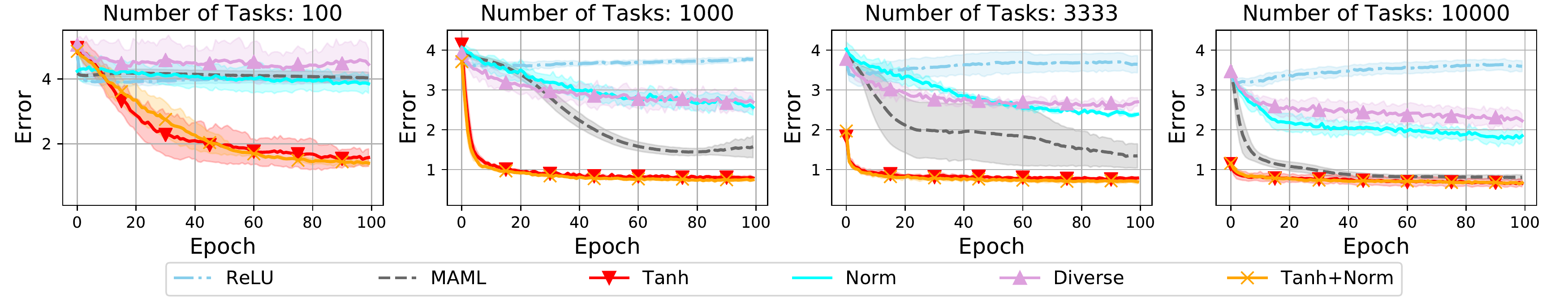} }\\ \vspace{-3mm}
	\subfigure[$m=5,n=5$]{\includegraphics[width=1\textwidth]{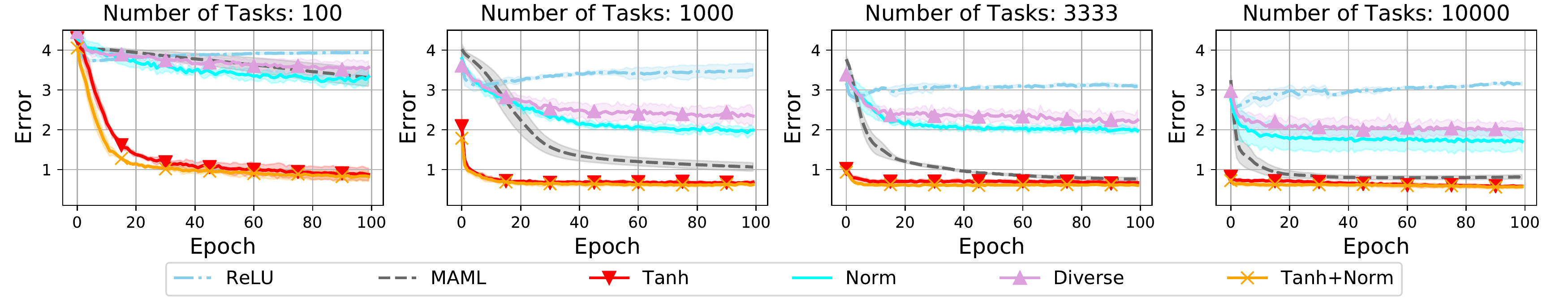} }
	\vspace{-4mm}\caption{Transfer error changing curves of different training strategies with $T=100,1000,3333,10000$ tasks and different $m,n$ values.
		The variance of each curve over 3 independent runs has also depicted along the curve to show the stability of each method.
	}\label{fig-5-5} \vspace{-5mm}
\end{figure*}


\subsection{Numerical Experiments}\label{numerical}
In this section, we test the effectiveness of the theory-induced training strategies on few-shot regression benchmark.
We follow the experimental setting of MAML \citep{finn2017model}. Each task involves regression from the input to the output of a sine wave, where the amplitude and phase of the sinusoid are varied among tasks. Thus the problem aims to approximate a family of sine functions $f(x)=\alpha\sin(\beta x)$. The task distribution $\eta$ is the joint distribution $p(\alpha,\beta)$ of the amplitude parameter $\alpha$ and the phase parameter $\beta$. We set $p(\alpha) = U[0.1,5]$ and $p(\beta)=U[0,\pi]$. All the meta-training and meta-test tasks are randomly generated from the task distribution $p(\alpha,\beta)$. The function class $\mathcal{H}$ is set as an MLP with two hidden layers of size 40 with ReLU activation function, and $\mathcal{F}$ is a linear layer with bias False. Both the input and the output layers have dimensionality 1. The loss is the mean-squared error between the prediction $\mathbf{w}^{\mathsf{T}} h(x)$ and true value.
The transfer error is averaged over 600 meta-test tasks with varying shots and queries. It uses episodic training strategy for meta-training,
i.e., the meta-algorithm observes a set of training tasks sequentially and applies stochastic gradient descent with one task per batch.


We implement the MAML \citep{finn2017model} as the baseline method in Eq.(\ref{ffew})  solved by Bilevel Programming \cite{franceschi2018bilevel}.  We denote the later by ``ReLU'', since the last activation function of the meta-learner is ReLU.
The `ReLU' and `MAML' are implemented based on the code located at \url{https://github.com/jiaxinchen666/meta-theory} released by the paper \citep{chen2020closer}. And `Tanh', `Norm', `Diverse' and `Tanh+Norm' denote the training strategies (\romannumeral 1), (\romannumeral 2), (\romannumeral 3), and  both (\romannumeral 1), (\romannumeral 2) applied to the baseline method, respectively.
We set $\lambda_1=\lambda_2=1$ for training strategy (\romannumeral 2) and $\lambda_3=10$ for training strategy (\romannumeral 3). The task-specific optimizer is set as Adam optimizer with learning rate 0.01, and the meta-optimizer is set as Adam optimizer with learning rate 0.001. The other hyper-parameters of task-specific optimizer and meta-optimizer are the default settings of Adam optimizer.

Fig.\ref{fig-5-5} shows the transfer error of different training strategies with varied numbers of training tasks, support samples and query samples.
It can be seen that the proposed training strategies can help consistently improve the performance of the baseline method on different meta-training tasks or support/query samples. As shown, training strategy (\romannumeral 1) achieves an evidently larger improvement compared with (\romannumeral 2), (\romannumeral 3). Furthermore, we combine training strategies (\romannumeral 1) and (\romannumeral 2), hoping to achieve further improvement compared with only one training strategy (\romannumeral 1). However, there exists an unsubstantial improvement when using both training strategies (\romannumeral 1) and (\romannumeral 2).
This implies that a proper meta-regularization technique for meta-learner can bring more essential improvements compared with the regularization techniques for task-specific learners. Specifically, when $m$ is small ($m=1$), the training strategies (\romannumeral 2), (\romannumeral 3) bring little improvement compared with `ReLU', while training strategy (\romannumeral 1) produces consistent improvement with different training tasks and $m,n$ values.

Note that the training strategies (\romannumeral 2) and (\romannumeral 3) put controls on the task-specific learners, which are important to improve the performance for the given training tasks. While the training strategy (\romannumeral 1) adds a meta-regularized control imposed on the meta-learner, which is verified to be more important to improve the performance for transferring to the new query tasks. Such phenomenon is rational and has been observed comprehensively in our experiments. We thus will only study the control strategies for the meta-learner in the latter sections.

\section{Application \uppercase\expandafter{\romannumeral 2}: Few-shot Classification}   \label{classification}

In this section, we instantiate our meta learning framework for the few-shot classification problem.

\subsection{Basic Setting}
For the few-shot classification issue, usually one considers the $K$-way $N$-shot setting, in which each task contains $NK$ examples from $K$ classes with $N$ examples for each class. Thus the task-specific predictor machine $\mathcal{F}$ is often a $K$-class linear classifier. Let $\mathcal{X}=\mathbb{R}^d, \mathcal{Y}=\{0,1,\cdots, K-1\}$ and $\|x\| \leq R, \forall x \in \mathcal{X}$.
The output of the meta-learner is the feature representation of each sample, and the meta learning machine $\mathcal{H}$ is parameterized as a depth-$L$ vector-valued convolutional neural network (CNN) to extract the common feature representation for various classification tasks. Concretely, $h(x)$ can be written as:
\begin{align}\label{eqcnn}
	h(x) = \phi_{L}(\mathbf{W}_{L}(\phi_{L-1}(\mathbf{W}_{L-1}\cdots \phi_1 (\mathbf{W}_1 x)))),
\end{align}
where $\mathbf{W}_k,k\in [L]$ is the parameter matrix, and $\phi_k, k\in[L]$ is the 1-Lipschitz activation function. Formally, we additionally require both the norm of the weight matrix of each layer and the distance between the weights and the starting point weights are bounded, i.e.,
\begin{align} \label{eqss}
	\|\mathbf{W}_j\|_F \leq B_j, \|\mathbf{W}_j^0\|_F \leq B_j, \|\mathbf{W}_j - \mathbf{W}_j^0\|_F \leq D_j, \ \ j\in [L],
\end{align}
where $\mathbf{W}_j^0$ is the initialized parameter matrix of the model $h$. The distance to initialization has been observed to have substantial influence to generalization in deep learning \citep{bartlett2017spectrally,neyshabur2018role,nagarajan2019generalization}, and we introduce it here to develop the control strategy of the meta-learner.
Specifically, $\mathcal{F}$ and $\mathcal{H}$ are chosen as:
\begin{align}\label{eqfewclassification}
	\begin{split}
		\mathcal{F}&=\{f | f(z)= \bm{A}^{\mathsf{T}} z, \bm{A}\in \mathbb{R}^{d_L\times K}, \|\bm{A}\| \leq M  \},\\
		\mathcal{H} &= \{h|h(x) \in \mathbb{R}^{d_L} \text{ as defined in (\ref{eqcnn})}, x \in \mathcal{X}\}.
	\end{split}
\end{align}
The conditional distribution of classification is
\begin{align}\label{distribution}
	P(y=k|f (h(x))) = \frac{\exp(a_k)}{\sum_{j} \exp(a_j)} := \text{Softmax}(f(h(x))), a_k = (f(h(x)))_k, k\in [K].
\end{align}
The loss functions $\ell$ is chosen as the cross-entropy loss $\ell(\bm{A}^{\mathsf{T}} h(x),y)= -\sum_{k=1}^K y_k \log(a_k)$, where $y$ is the one-hot encoding of the ground-truth label, and $a_k$ is defined as in Eq. (\ref{distribution}).

Following the setting in \citep{finn2017model,lee2019metaopt}, we assume that there are $T$ tasks $\Gamma = \{D_t\}_{t=1}^T$ available for learning, and $D_t =( D_t^{tr},D_t^{val})$, where $D_t^{tr} = \{z_{ti}^{(s)}\}_{i=1}^{m}, D_t^{val} = \{z_{tj}^{(q)}\}_{j=1}^{n}$.
We set identical sample numbers for training/validation data sets for each task.
The few-shot classification model is then written as
\begin{align}\label{ffewcls}
	\begin{split}
		&\bm{W}^* = \arg\min_{\bm{W}} \frac{1}{nT}\sum_{t=1}^T \sum_{j=1}^n \ell(\bm{A}_t^{*\mathsf{T}} h(x_{tj}^{(q)}),y_{tj}^{(q)}) \\
		s.t.,\  &\bm{A}_t^{*} = \arg\min_{\bm{A}_t} \frac{1}{m}\sum_{i=1}^m \ell(\bm{A}_t^{T} h(x_{ti}^{(s)}),y_{ti}^{(s)}), t\in [T],
	\end{split}
\end{align}
where $\bm{W} = \{\mathbf{W}_{k}, k\in [L]\}$ is the collection of the parameter matrices of $h$, and $\bm{A}, h$ is chosen from $\mathcal{F},\mathcal{H}$ defined in Eq. (\ref{eqfewclassification}). We denote $\mathbf{Q} = (\bm{A}_1,\cdots,\bm{A}_T)^{\mathsf{T}} \in \mathbb{R}^{T\times d_L \times K }$, and $\mathbf{Q^*} = (\bm{A}^*_1,\cdots,\bm{A}^*_T)^{\mathsf{T}} \in \mathbb{R}^{T\times d_L \times K }$ as its theoretical optimal solution.

\subsection{Theoretical Analysis}\label{theory2}
Here, we will instantiate the general Theorem \ref{theomtest} for few-shot classification model defined in Eq. (\ref{ffewcls}).
The $Dis(D^{val})$ can be computed as
\begin{align*}
	Dis(D^{val}) &= \sup_{h,h'} \rho_{2,D^{val}} (\mathbf{f}^{(h)},\mathbf{f}^{(h')}) \leq 4 \sup_{h,x \in \mathcal{X}} \|\bm{A}^{\mathsf{T}} h(x)\|_2\\
	& \leq \sup_{h,x} 4M\|h(x)\| \leq 4MD \cdot \prod_{k=1}^{L} \|\mathbf{W}_k\|_F.
\end{align*}
Now we can verify that Assumptions \ref{assumption1} - \ref{assumption4} hold. The following proposition implies the Lipschitz continuity for the cross-entropy function.
\begin{proposition} \label{prop2}
	The loss function $\ell(f(h(x)),y)$ is 1-Lipschitz with respect to $f(h(x))$, where $\ell$ is cross-entropy loss.
\end{proposition}
According to proposition \ref{prop2}, we have
\begin{align*}
	\ell(f(h(x)),y)| \leq \|f(h(x))\| = \|\bm{A}^{\mathsf{T}} h(x)\| \leq M \|h(x)\| \leq M D \prod_{k=1}^{L} \|\mathbf{W}_k\|_F,
\end{align*}
and thus the loss function is bounded.
The following result then verifies Assumption \ref{assumption4}:
\begin{proposition}\label{prop3}
	Consider the few-shot classification model defined in Eq.(\ref{ffewcls}), and the loss function $\ell(\cdot,\cdot)$ chosen as the cross-entropy loss. The meta learning machine $\mathcal{H}$ and task-specific predictor machine are designed as in Eq.(\ref{eqfewclassification}). Then the $\alpha,\beta$ in Assumption \ref{assumption4} are $ \frac{M}{\sum_{k=1}^K  \sigma_{d_L}((\mathbf{K})_k)}$ and $0$, where $(\mathbf{K})_k=(\mathbf{Q^*})_k^{\mathsf{T}}(\mathbf{Q^*})_k/T$, and $(\cdot)_k$ denote the $k$-th element, $\sigma_{d_L}(\mathbf{K}_k)$ denote the $d_L$-th singular value of matrix $\mathbf{K}_k$ at a decreasing order
\end{proposition}

It should be indicated that the value of $\alpha$ reflects the similarity of task-specific learners for given training task set. And minimizing $\alpha$ attempts to increase the diversity of learned task-specific learners, which can cover the space as much as possible captured by the $h$ needed to be predicted on new tasks, and thus help improve the generalization ability of the extracted meta-learner for new query tasks.

Now we can compute the leading-order terms in Eq.(\ref{eqtestbound}) for the parameterized meta-learner and learners defined in Eq.(\ref{eqfewclassification}). We firstly show the Rademacher complexity of $\mathcal{H}$ in the following theorem.
\begin{theorem}[Theorem 2 in \citep{Gouk2021metaopt}] \label{CNN}
	Let $\mathcal{H}$ be the class of real-valued DNN as defined in Eq.(\ref{eqcnn}) and (\ref{eqss}) over $\mathcal{X}=\{\mathbf{x}: \|\mathbf{x}\|\leq R\}$, where each parameter matrix $\mathbf{W}_i$, $i\in[L]$ has Frobenius norm at most $B_i$, and the activation function $\phi_i,i\in[L]$ is 1-Lipschitz, with $\phi_i(0)=0$, and applied element-wise. Then we have:
	\begin{align*}
		\hat{\mathfrak{R}}_{N}(\mathcal{H}) \leq \frac{2\sqrt{2} d_L R \sum_{j=1}^L \frac{D_j}{2B_j \prod_{i=1}^j \sqrt{c_i} } \prod_{j=1}^L 2B_j \sqrt{c_j} }{\sqrt{N}},
	\end{align*}
	where $c_i$ is the number of columns in $\mathbf{W}_i$.
\end{theorem}

(1) The Gaussian complexity of meta-learner $h$ can be computed as
\begin{align*}
	\begin{split}
		\hat{\mathcal{G}}_{D^{val}}(\mathcal{H})  &= \mathbb{E} \sup_{h\in \mathcal{H}} \frac{1}{nT} \sum_{t=1}^{T} \sum_{j=1}^{n}\sum_{k=1}^{d_L} g_{tjk} h_k(x_{tj}^{(q)}) \\
		&\leq\sum_{k=1}^{d_L}  \hat{\mathcal{G}}_{D^{val}}(h_k)  \leq 2\sqrt{\log(nT)}\sum_{k=1}^{d_L}  \hat{\mathfrak{R}}_{D^{val}}(h_k) \\
		&\leq 2\sqrt{\log(nT)}  \cdot \frac{2\sqrt{2} d_L R \sum_{j=1}^L \frac{D_j}{2B_j \prod_{i=1}^j \sqrt{c_i} } \prod_{j=1}^L 2B_j \sqrt{c_j} }{\sqrt{nT}}.
	\end{split}
\end{align*}

(2) The Gaussian complexity of the task-specific learner can be computed as:
\begin{align*}
	\begin{split}
		&\hat{\mathcal{G}}_{D_t^{val}}(\mathcal{F})= \mathbb{E} \sup_{\bm{A}_t\in \mathcal{F}} \frac{1}{m} \sum_{i=1}^{m} \sum_{k=1}^K  g_{tik} ((\bm{A}_t)_{k}^{\mathsf{T}} h(x_{ti}^{(s)}))  \\
		& =\frac{1}{m} \mathbb{E} \sup_{\bm{A}_t\in \mathcal{F}} \sum_{k=1}^K ((\bm{A}_t)_{k}^{\mathsf{T}}  \sum_{i=1}^{m}   g_{tik}  h(x_{ti}^{(s)}))  \\
		&\leq \frac{M}{m} \mathbb{E}  \sum_{k=1}^K \left\|\sum_{i=1}^{m}  g_{tik}  h(x_{ti}^{(s)}))  \right\| \\
		& = \frac{MK}{\sqrt{m}} \sqrt{ \frac{1}{m}\sum_{i=1}^{m}\left\|  h(x_{ti}^{(s)}) \right\|_2^2} \\
		&  \leq \frac{MK}{\sqrt{m}} \cdot \max_{x_{ti}^{(s)} \in {D}_t^{tr}} \|h(x_{ti}^{(s)} )\|.
	\end{split}
\end{align*}
Different from few-shot regression, there exists an additional multiplicative factor with term $K$, which is the number of classes. This implies the complexity of the task-specific learner increases as the number of classes increases, complying with general experience in common practice.

Thus, the transfer error defined in Eq.(\ref{eqtestbound}) now can be written as
\begin{align*}
	\begin{split}
		&  \ \ \ \ \ \ \ \ R_{test}(\hat{f}_{\mu},\hat{h}) - R_{test}(f^*_{\mu},h^*) \\
		& \ \ \ \ \leq \frac{M}{\sum_{k=1}^K  \sigma_{d_L}((\mathbf{K})_k)} \left( 768L \log(4nT) L(\mathcal{F})\cdot \frac{2\sqrt{2} d_L R \sum_{j=1}^L \frac{D_j}{2B_j \prod_{i=1}^j \sqrt{c_i} } \prod_{j=1}^L 2B_j \sqrt{c_j} }{\sqrt{nT}} \right. \\&\phantom{=\;\;}\left.
		+ \frac{6LMK}{\sqrt{m}T}\sum_{t=1}^T  \max_{x_{ti}^{(s)} \in D_t^{val}} \|h(x_{ti}^{(s)} )\| + 6B \sqrt{\frac{\log \frac{2}{\delta}}{2nT}}  + 6B\sqrt{\frac{\log \frac{2}{\delta}}{m}} + \frac{48L\sup_{h,x} M\|h(x)\|}{n^2T^2}\right) \\
		&\ \ \ \ \ + \frac{6LMK}{\sqrt{m_{\mu}}} \cdot \max_{x_{i}^{(s)} \in D_{\mu}^{val}} \|h(x_{i}^{(s)} )\|+ 6B\sqrt{\frac{\log \frac{2}{\delta}}{m_{\mu}}}.
	\end{split}
\end{align*}
Note that we assume there exists no distribution shift between the training and validation sets, and thus the term $d_{\mathcal{F}} (D_t^{(tr)},D_t^{(val)})$ is zero and omitted.

\subsection{Theory-Inspired Meta-Regularization}
As aforementioned, we develop the following two controlling strategies for meta learner $h$ to improve its methodology transferable generalization capability among tasks.

(\romannumeral 1) Control the output range of the meta-learner $h$.
All activation functions of CNN defined in Eq.(\ref{eqcnn}) are usually assumed to be ReLU. We take the same control strategy as in Section \ref{regular},
and revise the last activation function as Tanh (i.e., $\phi_{L}=\frac{e^z-e^{-z}}{e^z+e^{-z}}$).

(\romannumeral 2) Minimize the distance between the weights from the starting point weights in Eq.(\ref{eqss}). This control strategy can decrease the complexity of meta-learner $h$.

The training strategy (\romannumeral 2) can be achieved by adding the penalty terms corresponding to each layer into the meta-training objective in Eq. (\ref{ffew}),
\begin{align*}
	\begin{split}
		&\bm{W}^* = \arg\min_{\bm{W}} \frac{1}{nT}\sum_{t=1}^T \sum_{j=1}^n \ell(\bm{A}_t^{*\mathsf{T}} h(\mathbf{x}_{tj}^{(q)}),y_{tj}^{(q)}) + \lambda \sum_{j=1}^L \|\mathbf{W}_j - \mathbf{W}_j^0\|_F^2\\
		s.t.,\  &\bm{A}_t^{*} = \arg\min_{\bm{A}_t} \frac{1}{m}\sum_{i=1}^m \ell(\bm{A}_t^{\mathsf{T}} h(\mathbf{x}_{ti}^{(s)}),y_{ti}^{(s)}) , t\in[T],
	\end{split}
\end{align*}
where $\lambda$ is the hyper-paramters of the meta-regularization. Note that this regularization scheme has been used to help improve the performance of transfer learning proposed by \citep{xuhong2018explicit,li2018delta}. We firstly employ this regularization scheme for few-shot classification problems.
A detailed comparison with current meta learning methods \citep{zhou2019efficient,rajeswaran2019meta} could be found in Appendix \ref{l2sp}.

\subsection{Numerical Experiments}\label{numericalcls}
In this section, we test the effectiveness of the theory-induced training strategies on few-shot classification benchmark.

\subsubsection{Implementation details}
We parameterize the meta-learner $h$ as a ResNet-12 network and a standard 4-layer convolutional network followed by \citep{snell2017prototypical,lee2019metaopt}. Also, we use DropBlock regularization \citep{ghiasi2018dropblock} for the ResNet-12 network. The compared methods include ProtoNet, MetaOptNet-RR and MetaOptNet-SVM, whose task-specific learners are nearest-neighbor classifier \citep{snell2017prototypical}, ridge regression classifier \citep{bertinetto2018meta} and SVM classifier \citep{lee2019metaopt}, respectively. We denote ``Tanh'', ``$L^2$-SP'' and  ``Tanh + $L^2$-SP'' as training strategies (\romannumeral 1), (\romannumeral 2) and combining training strategy (\romannumeral 1) and (\romannumeral 2).

We follow the setting in \citep{lee2019metaopt} to conduct the few-shot classification experiments. To optimize the meta-learners, we use SGD with Nesterov momentum of 0.9 and weight decay of 0.0005. Each mini-batch consists of 8 episodes. The model was meta-trained for 60 epochs, with each epoch consisting of 1000 episodes. The learning rate was initially set to 0.1, and then changed to 0.006, 0.0012, and 0.00024 at epochs 20, 40 and 50, respectively. During the meta-training stage, we adopt horizontal flip, random crop, and color (brightness, contrast, and saturation) jitter data augmentation techniques. We use 5-way classification in both meta-training and meta-test stages. Each class contains 6 test (query) samples during meta-training and 15 test samples during meta-testing. Our meta-trained model was chosen based on 5-way 5-shot test accuracy on the meta-validation set. Meanwhile, during meta-training, we set training shot to 15 for miniImageNet with ResNet-12; 5 for miniImageNet with 4-layer CNN; 10 for tieredImageNet; 5 for CIFAR-FS; and 15 for FC100. We set the hyper-parameter $\lambda$ as 0.1, and keep the default hyper-parameter setting for the compared baselines in the original papers. Our implemention is based on the code provided on \url{https://github.com/kjunelee/MetaOptNet}.

\subsubsection{Experiments on CIFAR derivatives}

The \textbf{CIFAR-FS} dataset \citep{bertinetto2018meta} is a recently proposed few-shot image classification benchmark, consisting of all 100 classes from CIFAR-100. The classes are randomly split into 64, 16 and 20 for meta-training, meta-validation, and meta-testing, respectively. Each class contains 600 images of size $32\times 32$.

The \textbf{FC100} dataset  is another dataset derived from CIFAR-100 \citep{oreshkin2018tadam}, containing 100 classes which are grouped into 20 superclasses. These classes are partitioned into 60 classes from 12 superclasses for meta-training, 20 classes from 4 superclasses for meta-validation, and 20 classes from 4 superclasses for meta-testing. The goal is to minimize semantic overlap between classes. All images are also of size $32\times 32$.

\textbf{Results}. Table \ref{table4} summarizes the results on the 5-way CIFAR-FS and FC100 classification tasks with different shots. It can be seen that our theory-induced training strategy does help improve generalization error from SOTA results in most cases. Specifically, we shows the results where we vary the meta-learner for two different embedding architectures. In both cases that the dimension of the feature representation is low (1600, a standard 4-layer convolutional network), and the dimension is much higher (16000, ResNet12), our proposed meta-regularization schemes yield better few-shot accuracy than baselines. This implies that our proposed meta-regularization scheme is model-agnostic, in the sense that it can be directly applied to meta-regularize the learning for different meta-learners. Meanwhile, even on the harder FC100 dataset, our strategies can still help increase the accuracy for new query few-shot tasks, which highlights the advantage of our problem-agnostic meta-regularization schemes.

\begin{table*}[t]
	\caption{Average few-shot classification accuracies (\%) with 95\%
		confidence intervals on CIFAR-FS and FC-100 meta-test splits. }\label{table4}
	\centering
	\begin{tiny}
		\begin{tabular}{lllll}
			\toprule
			&  	\multicolumn{2}{c}{\textbf{CIFAR-FS 5-way}} & \multicolumn{2}{c}{\textbf{FC100 5-way}}\\
			\cline{2-3} \cline{4-5}
			\textbf{model}    &         \multicolumn{1}{c} {\textbf{1-shot } }  & \multicolumn{1}{c} {\textbf{5-shot } }  &  \multicolumn{1}{c} {\textbf{1-shot } }  & \multicolumn{1}{c} {\textbf{5-shot } }        \\
			\hline \hline
			\multicolumn{5}{l}{\textbf{4-layer conv (feature dimension=1600)}}   \\
			ProtoNet  &        59.89 $\pm$0.70     &      80.14 $\pm$ 0.53        &        36.13 $\pm$ 0.53     &    50.69 $\pm$0.55         \\
			ProtoNet(Tanh)  &     63.55 $\pm$ 0.75 \textcolor{red}{4.66$\uparrow$} &  80.30 $\pm$ 0.52   \textcolor{red}{0.16$\uparrow$}           & 36.64 $\pm$ 0.55  \textcolor{red}{0.51$\uparrow$}    &     50.53 $\pm$ 0.56    \textcolor{green}{0.16$\downarrow$}    \\
			ProtoNet($L^2$-SP)  &     60.71 $\pm$ 0.71 \textcolor{red}{0.82$\uparrow$} &  80.01 $\pm$ 0.53   \textcolor{green}{0.13$\downarrow$}           & 36.52 $\pm$ 0.54  \textcolor{red}{0.39$\uparrow$}    &     51.10 $\pm$ 0.56    \textcolor{red}{0.01$\uparrow$}    \\
			ProtoNet(Tanh + $L^2$-SP)  &     63.58 $\pm$ 0.75 \textcolor{red}{4.66$\uparrow$} &  80.20 $\pm$ 0.52   \textcolor{red}{0.06$\uparrow$}           & 36.74 $\pm$ 0.55  \textcolor{red}{0.61$\uparrow$}    &     50.77 $\pm$ 0.54    \textcolor{red}{0.08$\uparrow$}    \\ \hline
			MetaOptNet-RR &    63.02 $\pm$ 0.71    &      79.60 $\pm$ 0.52 &     35.75 $\pm$ 0.52     &   51.20 $\pm$ 0.55                              \\
			MetaOptNet-RR(Tanh) &   63.45 $\pm$ 0.70  \textcolor{red}{0.43$\uparrow$}    &      80.11 $\pm$ 0.53  \textcolor{red}{0.51$\uparrow$}         & 38.61 $\pm$ 0.56 \textcolor{red}{2.86$\uparrow$}   &       52.29 $\pm$ 0.56  \textcolor{red}{1.09$\uparrow$}          \\
			MetaOptNet-RR($L^2$-SP) &   63.66 $\pm$ 0.72   \textcolor{red}{0.64$\uparrow$}    &      79.49 $\pm$ 0.52 \textcolor{green}{0.11$\downarrow$}         & 36.55 $\pm$ 0.52 \textcolor{red}{0.80$\uparrow$}   &       51.27 $\pm$ 0.53  \textcolor{red}{0.07$\uparrow$}          \\
			MetaOptNet-RR(Tanh + $L^2$-SP) &   63.81 $\pm$ 0.71  \textcolor{red}{0.69$\uparrow$}    &      80.44 $\pm$ 0.50  \textcolor{red}{0.84$\uparrow$}         & 37.28 $\pm$ 0.53 \textcolor{red}{1.53$\uparrow$}   &       52.31 $\pm$ 0.52  \textcolor{red}{1.11$\uparrow$}          \\ \hline
			MetaOptNet-SVM &       61.77 $\pm$ 0.73 &        79.08 $\pm$ 0.52     &  35.82 $\pm$ 0.54    &   50.60 $\pm$ 0.54        \\
			MetaOptNet-SVM(Tanh) &           63.18 $\pm$ 0.71 \textcolor{red}{1.41$\uparrow$} & 79.87 $\pm$ 0.53  \textcolor{red}{0.49$\uparrow$}    &          37.10 $\pm$ 0.58 \textcolor{red}{1.28$\uparrow$}  &  51.62 $\pm$ 0.57 \textcolor{red}{1.02$\uparrow$}\\
			MetaOptNet-SVM($L^2$-SP) &           62.70 $\pm$ 0.71 \textcolor{red}{0.93$\uparrow$} & 79.95 $\pm$ 0.53  \textcolor{red}{0.87$\uparrow$}    &          37.13 $\pm$ 0.68 \textcolor{red}{1.31$\uparrow$}  &  51.42 $\pm$ 0.53 \textcolor{red}{0.82$\uparrow$}\\
			MetaOptNet-SVM(Tanh + $L^2$-SP) &           63.24 $\pm$ 0.72 \textcolor{red}{1.47$\uparrow$} & 80.20 $\pm$ 0.51  \textcolor{red}{1.12$\uparrow$}    &          37.29 $\pm$ 0.58 \textcolor{red}{1.47$\uparrow$}  &  51.71 $\pm$ 0.55 \textcolor{red}{1.11$\uparrow$}\\
			\hline\hline
			\multicolumn{5}{l}{\textbf{ ResNet-12 (feature dimension=16000)}}    \\
			ProtoNet  &      68.00 $\pm$ 0.74        & 83.50 $\pm$ 0.51           &          38.43 $\pm$ 0.58    &        52.56 $\pm$ 0.56           \\
			ProtoNet(Tanh)  &   69.71 $\pm$ 0.76   \textcolor{red}{1.71$\uparrow$}        &   83.62 $\pm$ 0.51  \textcolor{red}{0.12$\uparrow$}     &         40.12 $\pm$ 0.58     \textcolor{red}{1.69$\uparrow$}  &           55.02 $\pm$ 0.54     \textcolor{red}{2.06$\uparrow$} \\
			ProtoNet($L^2$-SP)  &   69.01 $\pm$ 0.73   \textcolor{red}{1.01$\uparrow$}        &   83.57 $\pm$ 0.51  \textcolor{red}{0.07$\uparrow$}     &         39.16 $\pm$ 0.57    \textcolor{red}{0.73$\uparrow$}  &           53.60 $\pm$ 0.56     \textcolor{red}{1.04$\uparrow$} \\
			ProtoNet(Tanh + $L^2$-SP)  &   71.22 $\pm$ 0.74   \textcolor{red}{3.22$\uparrow$}        &   83.89 $\pm$ 0.51  \textcolor{red}{0.39$\uparrow$}     &         40.70 $\pm$ 0.58     \textcolor{red}{2.43$\uparrow$}  &           55.33 $\pm$ 0.54     \textcolor{red}{2.77$\uparrow$} \\ \hline
			MetaOptNet-RR &     68.58 $\pm$ 0.73           &   84.75 $\pm$ 0.50         &       38.98 $\pm$ 0.56   &           54.46 $\pm$ 0.55        \\
			MetaOptNet-RR(Tanh) &    71.64 $\pm$ 0.72   \textcolor{red}{3.06$\uparrow$}       &     84.81 $\pm$ 0.49    \textcolor{red}{0.06$\uparrow$}        & 40.13 $\pm$ 0.58  \textcolor{red}{1.15$\uparrow$}&     54.48 $\pm$ 0.54 \textcolor{red}{0.02$\uparrow$}\\
			MetaOptNet-RR($L^2$-SP) &    71.40 $\pm$ 0.70   \textcolor{red}{2.82$\uparrow$}       &     84.91 $\pm$ 0.50    \textcolor{red}{0.16$\uparrow$}        & 39.79 $\pm$ 0.57  \textcolor{red}{0.81$\uparrow$}&     54.80 $\pm$ 0.55 \textcolor{red}{0.34$\uparrow$}\\
			MetaOptNet-RR(Tanh + $L^2$-SP) &    72.15 $\pm$ 0.72   \textcolor{red}{3.57$\uparrow$}       &     85.03 $\pm$ 0.49    \textcolor{red}{0.28$\uparrow$}        & 40.86 $\pm$ 0.57  \textcolor{red}{1.88$\uparrow$}&     55.58 $\pm$ 0.54 \textcolor{red}{1.12$\uparrow$}\\  \hline
			MetaOptNet-SVM &     68.81 $\pm$ 0.74          &    83.80 $\pm$ 0.51             &         40.24 $\pm$ 0.58  &           54.71 $\pm$ 0.56       \\
			MetaOptNet-SVM(Tanh) &  71.12 $\pm$ 0.71    \textcolor{red}{2.31$\uparrow$}            &     85.19 $\pm$ 0.49  \textcolor{red}{1.39$\uparrow$}              &                  39.99 $\pm$ 0.57  \textcolor{green}{0.25$\downarrow$} &                   53.92 $\pm$ 0.56  \textcolor{green}{0.79$\downarrow$}       \\
			MetaOptNet-SVM($L^2$-SP) &  70.84 $\pm$ 0.72     \textcolor{red}{2.03$\uparrow$}            &     84.22 $\pm$ 0.48  \textcolor{red}{0.42$\uparrow$}              &                  40.41 $\pm$ 0.57  \textcolor{red}{0.17$\uparrow$} &                   55.53 $\pm$ 0.55  \textcolor{red}{0.82$\uparrow$}       \\
			MetaOptNet-SVM(Tanh + $L^2$-SP) &  71.52 $\pm$ 0.73    \textcolor{red}{2.71$\uparrow$}            &     84.77 $\pm$ 0.48  \textcolor{red}{0.97$\uparrow$}              &                  40.83 $\pm$ 0.58 \textcolor{red}{0.59$\uparrow$} &                   55.75 $\pm$ 0.56  \textcolor{red}{1.04$\uparrow$}       \\
			\bottomrule
		\end{tabular} \vspace{-2mm}
	\end{tiny}
\end{table*}

\subsubsection{Experiments on ImageNet derivatives}
The \textbf{miniImageNet} dataset \citep{vinyals2016matching} is a standard benchmark for few-shot image classification benchmark, consisting of 100 randomly chosen classes from ILSVRC-2012 \citep {russakovsky2015imagenet}. The meta-training, meta-validation, and meta-testing sets contain 64, 16 and 20 classes randomly split from 100 classes, respectively. Each class contains 600 images of size $84\times 84$. We use the commonly-used split proposed by \citep{ravi2016optimization}.

The \textbf{tieredImageNet} benchmark \citep{ren2018meta} is a larger subset of ILSVRC-2012 \citep {russakovsky2015imagenet}, composed of 608 classes grouped into 34 high-level categories. These
categories are then split into 3 disjoint sets: 20 categories for meta-training, 6 for meta-validation, and 8 for meta-test. This corresponds to 351, 97 and 160 classes for meta-training, meta-validation, and meta-testing, respectively. This dataset aims to minimize the semantic similarity between the splits similar to FC100. All images are of size $84\times 84$.

\textbf{Results}. Table \ref{table3} summarizes the results on the 5-way classification tasks with different shots on miniImageNet and tieredImageNet benchmarks.
Compared with CIFAR derivatives dataset, this benchmark contains more classes and more natural images. It can be observed that our proposed meta-regularization schemes can also help improve generalization error in most cases from SOTA baselines method without adding these schemes with different embedding architectures of meta-learners or different semantic similarity between the splits.

It should be emphasized that we have just directly replaced the last activation function (ReLU) of feature extractor (meta-learner) as Tanh, or add the meta-regularization scheme to the meta loss, and keep the original configurations without implementing
any further fine-tuning executions and introducing extra computation cost. This way allows us to clearly see the effect of proposed meta-regularization schemes on the compared baselines. The empirical results verify that our theory-induced meta-regularization schemes are rational and arguably simple to improve the baseline performance.

\begin{table*}[t]
	\caption{Average few-shot classification accuracies (\%) with 95\%
		confidence intervals on miniImageNet and tieredImageNet meta-test splits. }\label{table3}
	\centering
	\begin{tiny}
		\begin{tabular}{lllll}
			\toprule
			&  	\multicolumn{2}{c}{\textbf{miniImageNet 5-way}}   & \multicolumn{2}{c}{\textbf{tieredImageNet 5-way}} \\
			\cline{2-3} \cline{4-5}
			\textbf{model}    &        \multicolumn{1}{c} {\textbf{1-shot } }  & \multicolumn{1}{c} {\textbf{5-shot } }  &  \multicolumn{1}{c} {\textbf{1-shot } }  & \multicolumn{1}{c} {\textbf{5-shot } }       \\
			\hline \hline
			\multicolumn{5}{l}{\textbf{4-layer conv (feature dimension=1600)}}   \\
			ProtoNet  &    47.69 $\pm$ 0.64   &  70.51 $\pm$ 0.52   &  50.10 $\pm$ 0.70  & 71.85 $\pm$ 0.57   \\
			ProtoNet(Tanh)  &     53.13 $\pm$ 0.64 \textcolor{red}{5.44$\uparrow$}  &  71.22 $\pm$ 0.52 \textcolor{red}{0.71$\uparrow$} &       54.07 $\pm$ 0.69 \textcolor{red}{3.93$\uparrow$}     &     71.05 $\pm$ 0.57    \textcolor{green}{0.60$\downarrow$}      \\
			ProtoNet($L^2$-SP)  &     48.15 $\pm$ 0.63 \textcolor{red}{0.46$\uparrow$}  &   70.67 $\pm$ 0.52 \textcolor{red}{0.16$\uparrow$} &       50.84 $\pm$ 0.71 \textcolor{red}{0.74$\uparrow$}     &     72.50 $\pm$ 0.50   \textcolor{red}{0.64$\uparrow$}      \\
			ProtoNet(Tanh + $L^2$-SP)  &     54.43 $\pm$ 0.71 \textcolor{red}{6.74$\uparrow$}  &   70.88  $\pm$ 0.52 \textcolor{red}{0.37$\uparrow$} &       54.06 $\pm$ 0.69 \textcolor{red}{3.93$\uparrow$}     &     71.25 $\pm$ 0.57    \textcolor{green}{0.40$\downarrow$}      \\\hline
			MetaOptNet-RR &   52.66 $\pm$ 0.63 & 69.80 $\pm$ 0.53  &     54.04 $\pm$ 0.68     &       72.09 $\pm$ 0.56  \\
			MetaOptNet-RR(Tanh) & 52.70 $\pm$ 0.62 \textcolor{red}{0.04$\uparrow$} &   70.72 $\pm$ 0.52  \textcolor{red}{0.98$\uparrow$}&     54.23 $\pm$ 0.67 \textcolor{red}{0.19$\uparrow$} &  72.35 $\pm$ 0.55   \textcolor{red}{0.26$\uparrow$}        \\
			MetaOptNet-RR($L^2$-SP) & 52.38 $\pm$ 0.62  \textcolor{green}{0.28$\downarrow$} &   68.95 $\pm$ 0.52  \textcolor{green}{0.85$\downarrow$}&     55.66 $\pm$ 0.68 \textcolor{red}{1.62$\uparrow$} &  72.15 $\pm$ 0.57   \textcolor{red}{0.06$\uparrow$}        \\
			MetaOptNet-RR(Tanh + $L^2$-SP) & 52.47 $\pm$ 0.62 \textcolor{green}{0.19$\downarrow$} &   70.17 $\pm$ 0.51  \textcolor{red}{0.37$\uparrow$}&     54.45 $\pm$ 0.69 \textcolor{red}{0.41$\uparrow$} &  72.40 $\pm$ 0.55   \textcolor{red}{0.31$\uparrow$}        \\ \hline
			MetaOptNet-SVM &    52.50 $\pm$ 0.62 &  69.66 $\pm$ 0.52  &     54.38 $\pm$ 0.70       &            71.48 $\pm$ 0.56       \\
			MetaOptNet-SVM(Tanh) &    52.75 $\pm$ 0.62 \textcolor{red}{0.25$\uparrow$} &70.79 $\pm$ 0.50   \textcolor{red}{1.13$\uparrow$}     &      54.66 $\pm$ 0.69  \textcolor{red}{0.28$\uparrow$}  &      71.57 $\pm$ 0.57 \textcolor{red}{0.09$\uparrow$} \\
			MetaOptNet-SVM($L^2$-SP) &    52.20 $\pm$ 0.60  \textcolor{green}{0.30$\downarrow$} & 69.04 $\pm$ 0.52   \textcolor{green}{0.62$\downarrow$}     &      55.52 $\pm$ 0.71  \textcolor{red}{1.14$\uparrow$}  &      71.49 $\pm$ 0.57 \textcolor{red}{0.01$\uparrow$} \\
			MetaOptNet-SVM(Tanh + $L^2$-SP) &    52.54 $\pm$ 0.61 \textcolor{red}{0.04$\uparrow$} &70.27 $\pm$ 0.52   \textcolor{red}{0.61$\uparrow$}     &      54.66 $\pm$ 0.69  \textcolor{red}{0.28$\uparrow$}  &      71.62 $\pm$ 0.50 \textcolor{red}{0.14$\uparrow$} \\
			\hline\hline
			\multicolumn{5}{l}{\textbf{ ResNet-12 (feature dimension=16000)}}    \\
			ProtoNet  &   56.13 $\pm$ 0.67  &    74.57 $\pm$ 0.53     &  63.02 $\pm$ 0.72     &    80.11 $\pm$ 0.56    \\
			ProtoNet(Tanh)  &     57.55 $\pm$ 0.67 \textcolor{red}{1.42$\uparrow$}  &  74.62 $\pm$ 0.67 \textcolor{red}{0.05$\uparrow$}   &   65.11 $\pm$ 0.73 \textcolor{red}{2.09$\uparrow$}  &     79.69 $\pm$ 1.39  \textcolor{green}{0.42$\downarrow$}   \\
			ProtoNet($L^2$-SP)  &     56.51 $\pm$ 0.65 \textcolor{red}{0.38$\uparrow$}  &  75.15 $\pm$ 0.53 \textcolor{red}{0.58$\uparrow$}   &   63.48 $\pm$ 0.72 \textcolor{red}{0.46$\uparrow$}  &     80.65 $\pm$ 0.54  \textcolor{red}{0.54$\uparrow$}   \\
			ProtoNet(Tanh + $L^2$-SP)  &     58.10 $\pm$ 0.66 \textcolor{red}{1.97$\uparrow$}  &  74.87 $\pm$ 0.51 \textcolor{red}{0.30$\uparrow$}   &   65.68 $\pm$ 0.73 \textcolor{red}{2.66$\uparrow$}  &     80.59 $\pm$ 0.51  \textcolor{red}{0.48$\uparrow$}   \\ \hline
			MetaOptNet-RR &   60.27 $\pm$ 0.67 &76.19 $\pm$ 0.50    &    66.35 $\pm$ 0.71       &          81.25 $\pm$ 0.53 \\
			MetaOptNet-RR(Tanh) & 60.44 $\pm$ 0.64  \textcolor{red}{0.17$\uparrow$} &  76.66 $\pm$ 0.47 \textcolor{red}{0.47$\uparrow$}  &    66.97 $\pm$ 0.71 \textcolor{red}{0.62$\uparrow$}      &        81.42 $\pm$ 0.52  \textcolor{red}{0.17$\uparrow$}       \\
			MetaOptNet-RR($L^2$-SP) & 60.05 $\pm$ 0.65  \textcolor{green}{0.22$\downarrow$} &  77.02 $\pm$ 0.48 \textcolor{red}{0.83$\uparrow$}  &    66.77 $\pm$ 0.71 \textcolor{red}{0.42$\uparrow$}      &        81.48 $\pm$ 0.53  \textcolor{red}{0.23$\uparrow$}       \\
			MetaOptNet-RR(Tanh + $L^2$-SP) & 60.41 $\pm$ 0.65  \textcolor{red}{0.14$\uparrow$} &  76.98 $\pm$ 0.48 \textcolor{red}{0.79$\uparrow$}  &    66.85 $\pm$ 0.71 \textcolor{red}{0.50$\uparrow$}      &        81.49 $\pm$ 0.52  \textcolor{red}{0.24$\uparrow$}       \\ \hline
			MetaOptNet-SVM &   60.58 $\pm$ 0.66& 75.99 $\pm$ 0.51          &     66.40 $\pm$ 0.72    &         81.18 $\pm$ 0.54     \\
			MetaOptNet-SVM(Tanh) &    60.59 $\pm$ 0.65 \textcolor{red}{0.01$\uparrow$}  &  77.05 $\pm$ 0.47 \textcolor{red}{1.06$\uparrow$}     &  66.43 $\pm$ 0.70 \textcolor{red}{0.03$\uparrow$}  &       81.30 $\pm$ 0.52 \textcolor{red}{0.12$\uparrow$}   \\
			MetaOptNet-SVM($L^2$-SP) &    61.17 $\pm$ 0.64 \textcolor{red}{0.59$\uparrow$}  &  77.40 $\pm$ 0.49 \textcolor{red}{1.41$\uparrow$}     &  66.45 $\pm$ 0.71 \textcolor{red}{0.05$\uparrow$}  &       81.52 $\pm$ 0.53 \textcolor{red}{0.34$\uparrow$}   \\
			MetaOptNet-SVM(Tanh + $L^2$-SP) &    60.99 $\pm$ 0.65 \textcolor{red}{0.41$\uparrow$}  &  77.43 $\pm$ 0.48 \textcolor{red}{1.44$\uparrow$}     &  66.49 $\pm$ 0.09 \textcolor{red}{0.09$\uparrow$}  &       81.43 $\pm$ 0.52 \textcolor{red}{0.25$\uparrow$}   \\
			\bottomrule
		\end{tabular} \vspace{-2mm}
	\end{tiny}
\end{table*}

\subsubsection{Cross-domain few-shot classification}
To further justify the effectiveness of our proposed meta-regularization schemes in producing robust features for unseen tasks, we further perform experiments under cross-domain few-shot classification settings, where the meta-test tasks are substantially different from meta-train tasks. We report the results in Table \ref{table9}, using the same experiment settings, firstly introduced in \citep{chen2020closer}, where miniImageNet is used as meta-train set and CUB dataset \citep{wah2011caltech} as meta-test set.

Table \ref{table9} shows the cross-domain performance of the baselines and our proposed regularization schemes, which exhibits the similar tendency to few-shot classification results from Table \ref{table3}. As is known, the CUB is a fine-grained classification dataset with more intra-class variations, and has a large domain shift to the miniImageNet dataset. Therefore, the extracted learning methodology (feature representation) from the meta-training dataset is mostly irrelevant to the meta-test datasets, which makes the learning tasks from new different domains relatively difficult, as suggested in \citep{baik2020learning}. Even so, our proposed meta-regularization schemes can also evidently improves the performance of baselines. This further validates the effectiveness and validness of the developed theory and the theory-induced meta-regularization schemes to learn a more robust feature extractor with better adaptability to new tasks.

\section{Application \uppercase\expandafter{\romannumeral 3}: Domain Generalization}\label{domain}
In this section, we instantiate the proposed meta learning framework for the domain generalization problem.

\subsection{Basic Setting and Theoretical Analysis}
Here we consider the domain generalization (DG) setting. We present the generally used meta learning setting \citep{li2018learning,li2019feature} for heterogeneous DG. Usually, assuming that we have $T$ domains (datasets) $\Gamma = \{D_1, D_2, \cdots, D_T\}$ for meta-training, each domain (task) has both meta-training and meta-validation data with $D_t = (D_t^{tr},D_t^{val})$, where $D_t^{tr} = \{z_{ti}^{(s)}\}_{i=1}^{m_t}, D_t^{val} = \{z_{tj}^{(q)}\}_{j=1}^{n_t}$.
Note that the label space is not shared between training/support and validation/query domains, but it is straightforwardly applicable to conventional (homogeneous) DG when assuming that the label space is shared between different domains.

There exist several differences between DG and few-shot classification problem. (1) The former often assumes that the meta-training and meta-validation sets of each domain (task) share the same label space, but have distribution shift. However, the latter often assumes that they are with the same distribution. (2) The latter pays more attention to solving new query tasks with limited data, while the former not only deals with limited data in the meta-test domain, but also attempts to eliminate the distribution shift between meta-training and meta-test domain.

We take the same setting for $\mathcal{F}$ and $\mathcal{H}$ in few-shot classification, and borrow the theoretical analysis therein. Thus the transfer error defined in Eq.(\ref{eqtestbound}) can be written as
\begin{align*}
	\begin{split}
		&  \ \ \ \ \ \ \ \ R_{test}(\hat{f}_{\mu},\hat{h}) - R_{test}(f^*_{\mu},h^*) \\
		& \ \ \ \ \leq \frac{M}{\sum_{k=1}^K  \sigma_{d_L}((\mathbf{K})_k)} \left( 768L \log(4nT) L(\mathcal{F})\cdot \frac{2\sqrt{2} d_L R \sum_{j=1}^L \frac{D_j}{2B_j \prod_{i=1}^j \sqrt{c_i} } \prod_{j=1}^L 2B_j \sqrt{c_j} }{\sqrt{nT}} \right. \\&\phantom{=\;\;}\left.
		+ \frac{6LMK}{\sqrt{m}T}\sum_{t=1}^T  \max_{x_{ti}^{(s)} \in D_t^{val}} \|h(x_{ti}^{(s)} )\| + \frac{4}{T} \sum_{t=1}^{T} d_{\mathcal{F}} (D_t^{(tr)},D_t^{(val)})+ 6B \sqrt{\frac{\log \frac{2}{\delta}}{2nT}}  + 6B\sqrt{\frac{\log \frac{2}{\delta}}{m}} \right.\\
		&\phantom{=\;\;}\left. + \frac{48L\sup_{h,x} M\|h(x)\|}{n^2T^2}\right) + \frac{6LMK}{\sqrt{m_{\mu}}} \!\cdot \!\max_{x_{i}^{(s)} \in D_{\mu}^{val}} \!\|h(x_{i}^{(s)} )\|+\mathbb{E}_{\mu\sim \eta}d_{\mathcal{F}} (D_{\mu}^{(tr)},D_{\mu}^{(val)})+ 6B\sqrt{\frac{\log \frac{2}{\delta}}{m_{\mu}}}.
	\end{split}
\end{align*}
Note that there exists an additional distribution shift term $d_{\mathcal{F}} (D_t^{(tr)},D_t^{(val)})$ in the error bound, characterizing the distribution shift between the meta-training and meta-validation domains, as well as the term $d_{\mathcal{F}} (D_{\mu}^{(tr)},D_{\mu}^{(val)})$ characterizing the distribution shift between the training and test sets of meta-test set.
It is easy to see that proposed meta-regularization schemes in the few-shot classification can also be applied to DG, i.e.,
(\romannumeral 1) Control the output range of the meta-learner $h$;
(\romannumeral 2) Minimize the distance of the weights from the starting point weights as Eq.(\ref{eqss}).

\begin{table*}[t]
	\caption{Average few-shot classification accuracies (\%) with 95\%
		confidence intervals on 5-way cross-domain classification. }\label{table9}
	\centering
	\begin{footnotesize}
		\begin{tabular}{lll}
			\toprule
			&  	\multicolumn{2}{c}{\textbf{miniImageNet $\rightarrow$ CUB}}    \\
			\cline{2-3}
			\textbf{model}    &        \multicolumn{1}{c} {\textbf{1-shot } }  & \multicolumn{1}{c} {\textbf{5-shot } }      \\
			\hline \hline
			\multicolumn{3}{l}{\textbf{4-layer conv (feature dimension=1600)}}   \\
			ProtoNet  &    36.20 $\pm$ 0.54   &  55.44 $\pm$ 0.52     \\
			ProtoNet(Tanh)  &     39.33 $\pm$ 0.55 \textcolor{red}{3.13$\uparrow$}  & 56.30 $\pm$ 0.52 \textcolor{red}{0.86$\uparrow$}       \\
			ProtoNet($L^2$-SP)  &     36.40 $\pm$ 1.80 \textcolor{red}{0.20$\uparrow$}  &   58.40 $\pm$ 0.52 \textcolor{red}{2.96$\uparrow$}      \\
			ProtoNet(Tanh + $L^2$-SP)  &     39.23 $\pm$ 0.54 \textcolor{red}{3.03$\uparrow$}  &   57.46 $\pm$ 0.53		 \textcolor{red}{2.02$\uparrow$}     \\\hline
			MetaOptNet-RR &   41.27 $\pm$ 0.55			 & 58.15$\pm$0.52			  \\
			MetaOptNet-RR(Tanh) & 41.34 $\pm$ 0.56		 \textcolor{red}{0.07$\uparrow$} &   59.82 $\pm$ 0.51  \textcolor{red}{1.67$\uparrow$}      \\
			MetaOptNet-RR($L^2$-SP) & 41.23 $\pm$ 0.56  \textcolor{green}{0.04$\downarrow$} &   59.11 $\pm$ 0.51		  \textcolor{red}{0.96$\uparrow$}     \\
			MetaOptNet-RR(Tanh + $L^2$-SP) & 42.13 $\pm$ 0.55	 \textcolor{red}{0.86$\uparrow$} &   60.27 $\pm$ 0.52	  \textcolor{red}{2.12$\uparrow$}       \\ \hline
			MetaOptNet-SVM &    41.47 $\pm$ 0.56	 &  58.21 $\pm$ 0.53       \\
			MetaOptNet-SVM(Tanh) &    41.89 $\pm$ 0.56		 \textcolor{red}{0.42$\uparrow$} & 61.26 $\pm$ 0.52	 \textcolor{red}{3.05$\uparrow$}     \\
			MetaOptNet-SVM($L^2$-SP) &    41.00 $\pm$ 0.58	  \textcolor{green}{0.47$\downarrow$} & 58.81 $\pm$ 0.51	   \textcolor{red}{0.60$\uparrow$}     \\
			MetaOptNet-SVM(Tanh + $L^2$-SP) &    41.49 $\pm$ 0.54 \textcolor{red}{0.02$\uparrow$} & 59.44 $\pm$ 0.50	   \textcolor{red}{1.23$\uparrow$}      \\
			\hline\hline
			\multicolumn{3}{l}{\textbf{ ResNet-12 (feature dimension=16000)}}    \\
			ProtoNet  &   41.86 $\pm$ 0.59  &    58.41 $\pm$ 0.56      \\
			ProtoNet(Tanh)  &     42.23 $\pm$ 0.63 \textcolor{red}{0.37$\uparrow$}  &  61.61 $\pm$ 0.56
			\textcolor{red}{3.20$\uparrow$}      \\
			ProtoNet($L^2$-SP)  &     43.60 $\pm$ 0.60 \textcolor{red}{1.74$\uparrow$}  &  62.59 $\pm$ 0.55
			\textcolor{red}{4.18$\uparrow$}    \\
			ProtoNet(Tanh + $L^2$-SP)  &    43.53 $\pm$ 0.63  \textcolor{red}{1.67$\uparrow$}  &  60.03 $\pm$ 2.56 \textcolor{red}{1.62$\uparrow$}     \\ \hline
			MetaOptNet-RR &   44.43 $\pm$ 0.59  & 64.12 $\pm$ 0.51			    \\
			MetaOptNet-RR(Tanh) & 45.14 $\pm$ 0.58  \textcolor{red}{0.71$\uparrow$} &  64.23 $\pm$ 0.52 \textcolor{red}{0.11$\uparrow$}      \\
			MetaOptNet-RR($L^2$-SP) & 45.00 $\pm$ 0.60  \textcolor{red}{0.57$\uparrow$} &  64.44 $\pm$ 0.53
			\textcolor{red}{0.32$\uparrow$}        \\
			MetaOptNet-RR(Tanh + $L^2$-SP) & 45.24 $\pm$ 0.60		  \textcolor{red}{0.81$\uparrow$} &  64.95 $\pm$ 0.50		   \textcolor{red}{0.83$\uparrow$}      \\ \hline
			MetaOptNet-SVM &   44.60 $\pm$ 0.59		 & 64.53 $\pm$ 0.53		             \\
			MetaOptNet-SVM(Tanh) &    45.09 $\pm$ 0.57	 \textcolor{red}{0.49$\uparrow$}  &  65.31 $\pm$ 0.52
			\textcolor{red}{0.78$\uparrow$}      \\
			MetaOptNet-SVM($L^2$-SP) &    44.68 $\pm$ 0.59 \textcolor{red}{0.08$\uparrow$}  &  64.85 $\pm$ 0.53
			\textcolor{red}{0.32$\uparrow$}       \\
			MetaOptNet-SVM(Tanh + $L^2$-SP) &    44.69 $\pm$ 0.46 \textcolor{red}{0.09$\uparrow$}  &  64.90 $\pm$ 0.52
			\textcolor{red}{0.37$\uparrow$}      \\
			\bottomrule
		\end{tabular} \vspace{-2mm}
	\end{footnotesize}
\end{table*}

Considering the difference from few-shot classification, DG should require proper training strategy to eliminate the distribution shift. Recently, \cite{li2019feature} proposed a Feature-Critic Networks technique to produce an auxiliary loss function to guide learning on the meta-training set to produce more robust and effective feature extractor on the meta-validation set. From this view, Feature-Critic Networks attempts to simulate the distribution shift between the meta-training and meta-validation sets, i.e., $d_{\mathcal{F}} (D_t^{(tr)},D_t^{(val)})$. Due to its rationality and efficacy, in the following part, we employ Feature-Critic Networks as the baseline method, and further verify the effectiveness of proposed meta-regularization schemes for DG problem.


\subsection{Numerical Experiments}
In this section, we verify the effectiveness of the theory-induced training strategies on homogeneous and heterogeneous DG benchmark, respectively.

\subsubsection{Homogeneous DG experiments}
\textbf{Dataset.} The \textbf{PACS} dataset is a recent object recognition benchmark for domain generalisation \citep{li2017deeper}.
PACS contains 9991 images of size $224 \times 224$ from four different domains - Photo (P), Art painting (A), Cartoon (C) and Sketch (S). It has 7 categories across these domains: dog, elephant, giraffe, guitar, house, horse and person. It can be downloaded at \url{http://sketchx.eecs.qmul.ac.uk/}. We follow the standard protocol and perform leave-one-domain-out evaluation.

\textbf{Implementation details.} We parameterize the shared feature extrator (meta-learner) $h$ as a pre-trained {AlexNet} \citep{krizhevsky2012imagenet} network followed by \citep{li2019feature}, and the task-specific learners as linear classifier. The competed methods include Reptile \citep{nichol2018first}, CrossGrad \citep{shankar2018generalizing}, MetaReg \citep{balaji2018metareg} and FC \citep{li2019feature}.
We follow the experimental setting in \citep{li2019feature}. Specifically, the feature extrator and task-specific learners are trained with M-SGD optimizer (batch size/per meta-trian domain=32, batch size/per meta-test domain=16, lr=0.0005, weight decay=0.00005, momentum=0.9) for 45K iterations.
The Feature-Critic (set embedding variant) is adopted as MLP type, and is trained with the M-SGD optimizer (lr=0.001, weight decay=0.00005, momentum=0.9).
We denote ``Tanh'', ``$L^2$-SP'' and  ``Tanh + $L^2$-SP'' as the training strategies (\romannumeral 1), (\romannumeral 2) and combining training strategy (\romannumeral 1) and (\romannumeral 2).
Our implementation is based on the code provided on \url{https://github.com/liyiying/Feature_Critic/}.

\begin{table*}[t]
	\caption{Cross-domain recognition accuracy (\%) on PACS using train split in \citep{li2017deeper} for training. }\label{table5}
	\centering
	\begin{tiny}
		\begin{tabular}{ccccccccc}
			\toprule
			Target	&  Reptile &   CrossGrad &     MetaReg    &  AGG & FC & FC+Tanh  &FC+$L^2$-SP & FC+Tanh+$L^2$-SP    \\ \hline
			A  &  63.4 &   61.0 &    63.5     &  62.2 & 63.0   &  \textbf{64.2}  & 63.4  &  63.5  \\
			C &  67.5 &   67.2 &  69.5  & 66.2   &  67.2  & \textbf{70.6}   & 68.2  & 68.3\\
			P&  \textbf{88.7} &   87.6 &   87.4  &  87.0  & 87.9   &  87.5  & \textbf{88.0} & 87.6\\
			S&  55.9 &   55.9 &  59.1  &  54.9  & 59.5   & 60.9   & \textbf{62.1}  & 61.4 \\ \hline
			Ave.&  68.9 &   67.9 &  69.9  &  67.6  &  69.4  &  \textbf{70.8}  &  70.4 & 70.2 \\
			\bottomrule
		\end{tabular} \vspace{-2mm}
	\end{tiny}
\end{table*}

\textbf{Results.} The comparison with state-of-the-art methods on PACS dataset is shown in Table \ref{table5}. The results of Reptile, CrossGrad and MetaReg are directly taken from \citep{li2019feature}. We re-implement the result of AGG \citep{li2019feature} and FC. Note that compared with those reported in \citep{li2019feature}, the depicted results of these two methods have a slight drop on A, C, P domain, but a slight rise on S domain. These minor differences are due to our used higher Pytorch version. From the table, it can be easily observed that our proposed meta-regularization schemes consistently outperform other competing methods. Especially, as compared with the SOTA method FC, aiming to learn representations that generalise to new domains through training a supervised loss and explicitly simulates domain shift, our proposed meta-regularization scheme puts emphasis on decreasing the complexity of the feature extractor (meta-learner). Therefore, it can be combinationally used with FC to compensate its performance.

\subsubsection{Heterogeneous DG experiments}

\textbf{Datasets.} The \textbf{Visual Decathlon} (VD) dataset consists of ten well-known datasets from multiple visual domains \citep{rebuffi2017learning}. \textbf{FGVC-Aircraft Benchmark} \citep{maji2013fine} contains 10,000 images of aircraft, with 100 images for each of 100 different aircraft model variants. \textbf{CIFAR100} \citep{krizhevsky2009learning} contains colour images for 100 object categories. \textbf{Daimler Mono Pedestrian Classification
	Benchmark (DPed)} \citep{munder2006experimental} consists of 50,000 grayscale pedestrian and non-pedestrian images. \textbf{Describable Texture Dataset (DTD)} \citep{cimpoi2014describing} is a texture database, consisting of 5640 images, organized according to a list of 47 categories such as bubbly, cracked, marbled. \textbf{The German Traffic Sign Recognition (GTSR) Benchmark} \citep{stallkamp2012man} contains cropped images for 43 common traffic sign categories in different image resolutions. \textbf{Flowers102} \citep{nilsback2008automated} is a fine-grained classification task which contains 102 flower categories from the UK. \textbf{ILSVRC12 (ImageNet)} \citep{russakovsky2015imagenet} is the largest dataset in our benchmark,  containing 1000 categories and 1.2 million images. \textbf{Omniglot} \citep{lake2015human} consists of 1623 different handwritten characters from 50 different alphabets. \textbf{The Street View House Numbers (SVHN)} \citep{netzer2011reading} is a real-world digit recognition
dataset with around 70,000 images. \textbf{UCF101} \citep{soomro2012ucf101} is an action recognition dataset of realistic human action videos, collected from YouTube. It contains 13,320 videos for 101 action categories. Each video has been summarized  into an image based on a ranking principle using the Dynamic Image encoding of \citep{bilen2016dynamic}.
The images have been pre-processed to the size of $72 \times 72$.
Following the benchmark for DG in \citep{li2019feature}, we take the six larger datasets (CIFAR-100, DPed, GTSR, Omniglot, SVHN and ImageNet) as source domains and hold out the four smaller datasets (FGVC-Aircraft, DTD, Flowers102 and UCF101) as target domains.

\textbf{Implementation details.} We parameterize the shared feature extractor (meta-learner) $h$ as a pre-trained {ResNet-18} \citep{he2016deep} network followed by \citep{li2019feature}, and the task-specific learners as the SVM classifier. The competing methods include Reptile \citep{nichol2018first}, CrossGrad \citep{shankar2018generalizing}, MetaReg \citep{balaji2018metareg, li2019feature} and FC \citep{li2019feature}.
We also follow the experimental setting in \citep{li2019feature}, and train all components end-to-end using the AMSGrad (batch-size/per meta-train domain=64, batch-size/per meta-test domain=32, lr=0.0005, weight decay=0.0001) for 30k iterations where the learning rate decayed in 5K, 12K, 15K, 20K iterations by a factor 5, 10, 50, 100, respectively. Similar to MetaReg \citep{balaji2018metareg}, after the parameters are trained via meta learning, we fine-tune the network on all
source datasets for the final 10k iterations. The Feature-Critic (set embedding variant) is adopted as the MLP type, and is trained with AMSGrad optimizer (lr=0.0001, weight decay=0.00001).

\begin{table*}[t]
	\caption{Cross-domain recognition accuracy (\%) of four held out target datasets on VD using train split in \citep{li2019feature} for training. }\label{table6}
	\centering
	\begin{tiny}
		\begin{tabular}{cccccccccc}
			\toprule
			Target	&  Reptile &   CrossGrad &     MR  & MR-FL &  AGG & FC & FC+Tanh  &FC+$L^2$-SP & FC+Tanh+$L^2$-SP    \\ \hline
			FGVC-Aircraft &  19.62 &  19.92 &    20.91 &18.18    &  18.84 &   19.02    &   19.35  & \textbf{19.80}     & 19.60    \\
			DTD &  37.39&   36.54 &  32.34 & 35.69               &   37.52&  39.01   &  39.20    &  39.04     & \textbf{40.37}    \\
			Flowers102 &  58.26 &   57.84  &   35.49 & 53.04     &    57.64 & 58.13     &  58.62     &  58.43      &   \textbf{60.78}  \\
			UCF101 &  49.85 &   45.80 &  47.34  &48.10           &   47.88 & 51.12     &   51.38    &   51.48     &   \textbf{52.12}   \\ \hline
			Ave.&  41.28 &  40.03 &  34.02  &  38.75             &   40.47     &  41.82     &  42.14     &   42.19     &  \textbf{43.22}   \\
			\bottomrule
		\end{tabular} \vspace{0mm}
	\end{tiny}
\end{table*}

\begin{figure*}[t]
	\centering
	\includegraphics[width=0.85\textwidth]{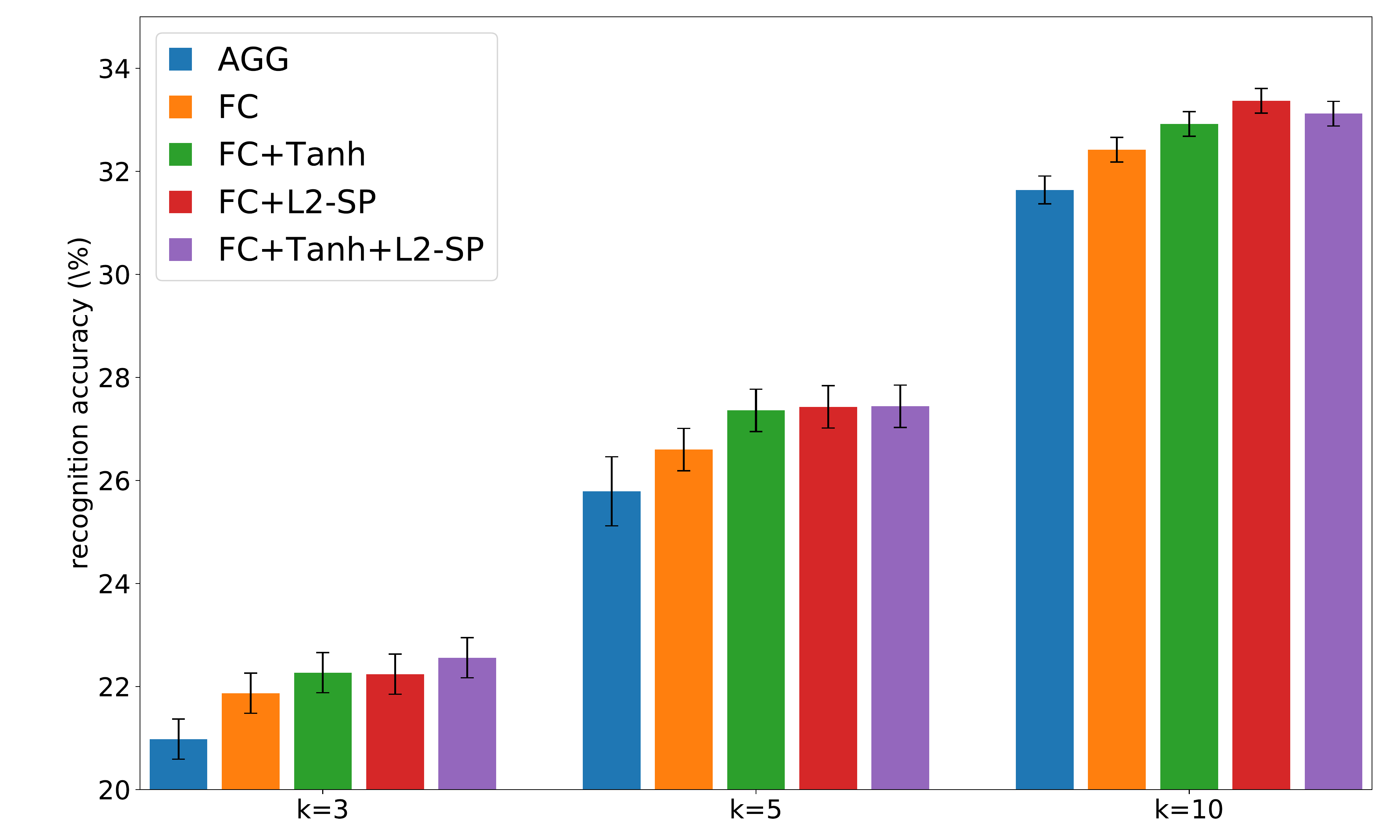}
	\vspace{-4mm}\caption{Recognition accuracies (\%) of all competing methods averaged over 5 test runs on VD K-shot learning
	}\label{fig4}
\end{figure*}

\textbf{Results.} Table \ref{table6} shows the classification accuracy on four hold-out target domain. The results of Reptile, CrossGrad, MR, MR-FL are directly taken from \citep{li2019feature}. We re-implemented the result of AGG \citep{li2019feature} and FC. Different from homogeneous DG, heterogeneous DG assumes that the label spaces among different domains are different, and thus the task-specific learners can not be shared. Thus learning a robust off-the-shelf feature extractor is very important for final performance of heterogeneous DG.
From the table, it can be seen that our proposed meta-regularization schemes produce more robust and effective feature extractor for unseen domains compared with other competing methods. Especially, although our methods are easily rebuilt from FC, they can attain evident improvements. Furthermore, we repeat the evaluation assuming that few-shot (3-shot, 5-shot, 10-shot) samples of training splits are available for SVM training in the meta-test stage. Fig. \ref{fig4} reports the target domain test accuracies under these settings. It can be easily observed that our proposed meta-regularization schemes provide a consistent improvement over the baseline, and produce a superior off-the-shelf feature representation. It can thus be substantiated that the proposed meta-regularization strategy is hopeful to generally improve the training quality of current meta learning methods to produce more robust and effective feature extractor.

\section{Online Methodology-Learning Strategy}

Here we further consider a sequential setting where an agent is faced with tasks one after another. Specifically, let's denote $T$ be the number of tasks and for each task $t \in \{1,\cdots,T\}$, and $D_t = (D_t^{tr}, D_t^{val})$ be the corresponding task sequence. Our goal is to find an estimator of $h\in \mathcal{H}$ that improves  incrementally as the number of observed tasks $T$ increases. Algorithm \ref{alg3:online meta-training} shows the online methodology-learning process, in which we take inspiration from the form of the ``follow the leader algorithm" (FTL) presented in the following Reference \cite{shalev2012online}. Correspondingly, we call our advanced algorithm in the meta-level as the ``follow the meta-leader algorithm" (FTML) as follows.

\begin{algorithm}[t]
	\caption{Online Methodology-Learning Algorithm: FTML}
	\label{alg3:online meta-training}
	\begin{algorithmic}
		\State {\bfseries Input:} $T$ number of tasks $\{D_t = (D_t^{tr}, D_t^{val})\}$.
		\State {\bfseries Initialization:} $h_0 \in \mathcal{H}$
		\State {\bfseries For $t=0$ to $T-1$:}
		\State \ \ \ \ \ {\bfseries (1)} The meta-learner incrementally receives a task dataset $D_t$;
		
		\State \ \ \ \ \ {\bfseries (2)} Run the inner online algorithm with meta-learner $h_t$ on $D_t^{tr}$, returning the predictor model $f_t^{h_t} = f_{t} - \alpha \frac{\partial \mathcal{L}^{task}(f, h_t, D_t^{tr})}{\partial f}$;
		
		\State \ \ \ \ \ {\bfseries (3)} Incrementally incur the errors $\ell_t(h_t) = \bm{L} (f_t^{h_t}, D_t^{val})$;
		
		\State \ \ \ \ \ {\bfseries (4)} Update the meta-learner via
		\begin{align}
			h_{t+1} = \arg\min_{h} \sum_{k=1}^t \ell_k (h).
		\end{align}
		\State {\bfseries End For}
	\end{algorithmic}
\end{algorithm}

The FTML algorithm template can be interpreted as the agent playing the best meta-learner in hindsight if the learning process was to stop at the round $t$. We will show that under standard assumptions on the losses of analyzing bi-level optimization, the online learning-the-methodology algorithm has the corresponding regret guarantees. Note that we may not have full access to $\ell_t(h_t)$ when it is the population risk, and we only have a finite sized dataset. We can draw upon stochastic optimization strategies to solve the optimization problem.

We make the following assumptions about each loss function in the learning problem for all $t$. Let  $z=(f,h)$ denote all parameters. Assumption 4 (1)-(3) are largely standard in online learning \citep{shalev2012online,cesa2006prediction}, and Assumption 4 (4)-(5)  are largely used in bi-level optimization analysis, e.g.,  \citep{ji2021bilevel}.

Let's make the following assumptions about each loss function in the learning problem for all $t$. Let  $z=(f,h)$ denote all parameters. We first list the following necessary assumptions. Note that in the following, assumptions (1)-(3) are largely standard in online learning (see References \citep{shalev2012online,cesa2006prediction}), and assumptions (4)-(5) are always used in bi-level optimization analysis, e.g., \citep{ji2021bilevel}.

\begin{assumption} \label{assumptiononline}
	(1) The meta-loss $\bm{L}$ has gradients bounded by $G$, i.e., $\|\nabla \bm{L} (f)\| \leq G$.
	
	(2) The meta-loss $\bm{L}$ is $\beta$-smooth, i.e., $\|\nabla \bm{L} (f) - \nabla \bm{L} (f') \| \leq \beta \|f-f'\|$.
	
	(3) The meta-loss $\bm{L}$ is $\mu$-strongly convex, i.e., $\|\nabla \bm{L} (f) - \nabla \bm{L} (f') \| \geq \mu \|f-f'\|$.
	
	(4) Suppose the task-specific loss $\bm{L}^{task}$ is $\rho$-strongly convex, i.e., $\left\|\frac{\partial \bm{L}^{task}(z)}{\partial f} - \frac{\partial \bm{L}^{task}(z')}{\partial f}\right\| \geq \rho \|z-z'\|, \forall z, z'$. 	
	
	(5) Suppose the task-specific loss $\bm{L}^{task}$ is $\tau$-smooth, i.e., $\left\|\frac{\partial \bm{L}^{task}(z)}{\partial f} - \frac{\partial \bm{L}^{task}(z')}{\partial f}\right\| \leq \tau \|z-z'\|, \forall z, z'$.

\end{assumption}

Based on the aforementioned assumptions, we can then deduce the following theorem:

\begin{theorem} \label{regret}
	Suppose that for all $t$, meta loss and task-specific loss satisfy Assumption \ref{assumptiononline}. Let $h_T$ be the output of Algorithm \ref{alg3:online meta-training}, then FTML enjoys the following regret guarantee
	\begin{align}
		\mathcal{R}_T = \sum_{t=1}^T \ell_t(h_t) - \min_h \sum_{t=1}^T \ell_t(h) \leq \frac{4 G^2 (1+\log(T)) }{\tau}.
	\end{align}
\end{theorem}

The above theorem provides the online methodology-learning bound, which tells us that there exists a large family of online meta-learning algorithms that can enjoy sub-linear regret under some mild conditions. Note that our theorem enjoys the same regret guarantees (up to constant factors) as FTL does in the comparable setting (with strongly convex losses) \citep{shalev2012online}.

\section{Conclusion}
This study has introduced a SLeM framework for understanding and formulating meta learning, in which a meta-learner is extracted to get the hyper-parameter prediction function for machine learning over training task set. The meta-learner is represented as an explicit function mapping from task information to the hyper-parameters involved in the learning process, facilitating the meta-learned meta-learner able to be readily transferred to new tasks to adaptively set their hyper-parametric configurations. Besides, the corresponding learning theory is developed for this SLeM meta learning framework. Very similar to the SRM principle used to improve generalization capability of the extracted learner in conventional machine learning, this theory can help conduct some useful meta-regularization strategies for ameliorating the generalization capability of the extracted meta-learner as a diverse-task-transferable learning methodology. We have further substantiated the beneficial effects brought by these meta-regularization strategies in typical meta learning applications, including few-shot regression, few-shot classification, and domain generalization. Especially, the new meta-regularization schemes can be easily embedded into the current meta learning programs by directly replacing the form of the output activation function in the learning model, or revising the parameter updating step for the learner or meta-learner from the unregularized solution to the regularized one. These meta-regularized strategies are thus hopeful to be easily and extensively used in more comprehensively meta learning tasks.

In future research, we'll further ameliorate the presented learning theory of SLeM, and especially try to deduce much tighter upper bound for evaluating the task transfer error to reveal more intrinsic generalization insight of this SLeM meta learning framework. Besides, we'll investigate more recent bi-level optimization techniques specifically considering the non-convexity of the inner and outer losses of our meta-learning framework, and ameliorate our theoretical results  to improve the learning bounds with relatively loose conditions. Furthermore, we'll make endeavor to explore more helpful meta-regularization schemes useful for more comprehensive and diverse meta learning problems, e.g., \citep{shu2023dac}. Specifically, we'll continue to explore how to effectively design the structure and rectify the parameter learning of meta-learner by enforcing meta-regularization on the model through certain meta-regularization tricks, so as to improve its generalization capability among variant tasks. We believe this will be potentially beneficial to the meta learning field, just like the significance of the SRM principle in the conventional machine learning field.


\acks{We thank the anonymous reviewers for insightful comments and suggestions. This research was supported in part by the National Key Research and Development
	 Program of China under Grant 2020YFA0713900, in part by the National Natural Science Foundation of China (NSFC) Project under Grants 61721002, and in part by the Macao Science and Technology Development Fund under Grant 061/2020/A2.}


\newpage

\appendix

\section{Theoretical Tools}\label{tools}
In this part, we will show some important auxiliary theoretical results and assumptions preparing for the proofs of the main theorems.

\subsection{Some Properties of Gaussian Complexity}
Here, we present some properties of Gaussian complexity, and the proofs can refer to the references therein.
\begin{proposition}[Ledoux-Talagrand contraction principle \citep{ledoux2013probability}]\label{proposition1}
	Let $\mathcal{X}$ be any set, $\mathcal{F}$ be a class of functions: $f:\mathcal{X}\rightarrow \mathbb{R}^d$. Then for $N$ data points, $\mathbf{X} = (x_1,\cdots,x_N)^{\mathsf{T}}$, and a fixed, centered $L$-Lipschitz function $\phi:\mathbb{R}^d\rightarrow \mathbb{R}$, we have
	\begin{align*}
		\hat{\mathcal{G}}_{\mathbf{X}}(\phi(\mathcal{F})) \leq L\hat{\mathcal{G}}_{\mathbf{X}}(\mathcal{F}); \ \ \ \hat{\mathfrak{R}}_{\mathbf{X}}(\phi(\mathcal{F})) \leq L\hat{\mathfrak{R}}_{\mathbf{X}}(\mathcal{F}),
	\end{align*}
	where $\hat{\mathcal{G}}_{\mathbf{X}}(\mathcal{F})$ and $\hat{\mathfrak{R}}_{\mathbf{X}}(\mathcal{F})$ denote the empirical Gaussian complexity and empirical Rademacher complexity, respectively.
\end{proposition}

\begin{proposition}[The relationship between Gaussian complexity and Rademacher complexity \citep{wainwright2019high}] \label{proposition2}
	The	empirical Rademacher complexity	can be lower/upper bounded by empirical Gaussian complexity as follows:
	\begin{align*}
		\hat{\mathcal{G}}_{\mathbf{X}}(\mathcal{F}) \leq 2\sqrt{\log(N)}\cdot \hat{\mathfrak{R}}_{\mathbf{X}}(\mathcal{F}), \  \hat{\mathfrak{R}}_{\mathbf{X}}  (\mathcal{F}) \leq \sqrt{\frac{\pi}{2}} \hat{\mathcal{G}}_{\mathbf{X}}(\mathcal{F}) \leq \frac{3}{2} \hat{\mathcal{G}}_{\mathbf{X}}(\mathcal{F}),
	\end{align*}
	where $\mathbf{X} = (x_1,\cdots,x_N)^{\mathsf{T}}$.
\end{proposition}

\subsection{Task-averaged Estimation Error}
For the single-task setting, the uniform bounds on the estimation error based on the Rademacher complexity can be found in \citep{mohri2018foundations}. Here we show the estimation error for the multi-task setting.
\begin{theorem}\label{thaverage}
	Let $\mathcal{F}$ be a family of function mappings $f_t:\mathcal{X}\rightarrow [0,B]$, and let $\mu_1,\mu_2,\cdots,\mu_T$ be probability measures on $\mathcal{X}$ with $\mathbf{X} =(\mathbf{X}_1,\cdots,\mathbf{X}_T) \sim \prod_{t=1}^{T} (\mu_t)^{n_t}$, where $\mathbf{X}_t=(x_{t1}, \cdots,x_{t,n_t})$ for each $t\in[T]$. Let $\mathbf{f} = (f_1,\cdots,f_T)$, we define $\hat{\mathbb{E}}_{\mathbf{X}}[\mathbf{f}] = \frac{1}{T}\sum_{t=1}^T \frac{1}{n_t}\sum_{j=1}^{n_t}  f_t(x_{ti}) $, and $\mathbb{E}[\mathbf{f}]= \mathbb{E}_{\mathbf{X}} [\hat{\mathbb{E}}_{\mathbf{X}}[\mathbf{f}]]$. Then, for any $\delta>0$, with probability at least $1-\delta$ we have
	\begin{align*}
		\mathbb{E}[\mathbf{f}]\leq \hat{\mathbb{E}}_{S}[\mathbf{f}]  + 2\hat{\mathfrak{R}}_{\mathbf{X}}(\mathcal{F}^{\otimes T}) + 3\frac{B}{T} \sqrt{\sum_{t=1}^T \frac{1}{n_t}} \sqrt{\frac{\log \frac{2}{\delta}}{2}},
	\end{align*}
	where $\hat{\mathfrak{R}}_{\mathbf{X}}(\mathcal{F}^{\otimes T})=\mathbf{E}_{\sigma} [\sup_{\mathbf{f}\in \mathcal{F}^{\otimes T}}\frac{1}{T}\sum_{t=1}^T \frac{1}{n_t}\sum_{j=1}^{n_t} \sigma_{tj} f_t(x_{ti})] $, and $\mathcal{F}^{\otimes T} =\underbrace{\mathcal{F}\times \mathcal{F} \cdots \mathcal{F}}_{T} $.
\end{theorem}

\begin{proof}
	We define the function $\Phi: \mathbf{X} \rightarrow \mathbb{R}$ as
	\begin{align*}
		\Phi(\mathbf{X})=\sup_{\mathbf{f} \in \mathcal{F}^{\otimes T}}(\mathbb{E}[\mathbf{f}]-\hat{\mathbb{E}}_{\mathbf{X}}[\mathbf{f}]).
	\end{align*}
	Let $\mathbf{X}$ and $\mathbf{X}'$ be two samples differing by exactly one point, say $x_{tj}$ in $\mathbf{X}$ and $x_{tj}'$ in $\mathbf{X}'$. Then, since the difference of suprema does not exceed the supremum of the difference, we have
	\begin{align*}
		|\Phi(\mathbf{X}')-\Phi(\mathbf{X})| \leq &\left|\sup_{\mathbf{f} \in \mathcal{F}^{\otimes T}} (\hat{\mathbb{E}}_{\mathbf{X}}[\mathbf{f}]-\hat{\mathbb{E}}_{\mathbf{X}'}[\mathbf{f}])\right|\\
		= &\frac{1}{T n_t} \left|\sup_{f_t \in \mathcal{F}}f_t(\mathbf{x}_{tj}')-f_t(\mathbf{x}_{tj})\right| 	\leq  \frac{B}{T n_t}.
	\end{align*}
	Then, based on the McDiarmid's inequality \citep{mohri2018foundations}, we have:
	\begin{align*}
		\mathbb{P}\left(\Phi(\mathbf{X})-\mathbb{E}_{\mathbf{X}}[\Phi(\mathbf{X})]\geq \epsilon\right) \leq &\exp\left(\frac{-2\epsilon^2}{\sum_{t=1}^T \sum_{j=1}^{n_t} (\frac{B}{Tn_t})^2}\right)\\
		= & \exp \left(\frac{-2 T^2 \epsilon^2 }{B^2 \sum_{t=1}^T \frac{1}{n_t}}\right).
	\end{align*}
	For $\delta>0$, setting the right-hand side above to be $\delta/2$, with probability at least $1-\delta/2$, the following holds:
	\begin{align*}
		\Phi(\mathbf{X})\leq \mathbb{E}_{\mathbf{X}}[\Phi(\mathbf{X})] + \frac{B}{T} \sqrt{\sum_{t=1}^T \frac{1}{n_t}} \sqrt{\frac{\log \frac{2}{\delta}}{2}}.
	\end{align*}
	We bound the expectation of the right-hand side as follows:
	\begin{align}
		&\mathbb{E}_{\mathbf{X}}[\Phi(\mathbf{X})] = \mathbb{E}_{\mathbf{X}}\left[\sup_{\mathbf{f} \in \mathcal{F}^{\otimes T}}\mathbb{E}[\mathbf{f}]-\hat{\mathbb{E}}_{\mathbf{X}}[\mathbf{f}]\right]\\ \label{eqr1}
		= & \mathbb{E}_{\mathbf{X}}\left[\sup_{\mathbf{f} \in \mathcal{F}^{\otimes T}} \mathbb{E}_{\mathbf{X}'} [\hat{\mathbb{E}}_{\mathbf{X}'}[\mathbf{f}]-\hat{\mathbb{E}}_{\mathbf{X}}[\mathbf{f}]]\right] \\ \label{eqr2}
		\leq& \mathbb{E}_{\mathbf{X},\mathbf{X}'}\left[\sup_{\mathbf{f} \in \mathcal{F}^{\otimes T}}\hat{\mathbb{E}}_{\mathbf{X}'}[\mathbf{f}]-\hat{\mathbb{E}}_{\mathbf{X}}[\mathbf{f}]\right]\\
		=& \mathbb{E}_{\mathbf{X},\mathbf{X}'}\left[ \sup_{\mathbf{f} \in \mathcal{F}^{\otimes T}} \frac{1}{T}\sum_{t=1}^T \frac{1}{n_t}\sum_{j=1}^{n_t} (f_t(\mathbf{x}_{tj}')-f_t(\mathbf{x}_{tj})) \right] \\ \label{eqr3}
		=&  \mathbb{E}_{\mathbf{\sigma},\mathbf{X},\mathbf{X}'}\left[ \sup_{\mathbf{f} \in \mathcal{F}^{\otimes T}} \frac{1}{T}\sum_{t=1}^T \frac{1}{n_t}\sum_{j=1}^{n_t} \sigma_{tj}(f_t(\mathbf{x}_{tj}')-f_t(\mathbf{x}_{tj})) \right] \\ \label{eqr4}
		\leq & \mathbb{E}_{\mathbf{\sigma},\mathbf{X}'} \left[ \sup_{\mathbf{f} \in \mathcal{F}^{\otimes T}} \frac{1}{T}\sum_{t=1}^T \frac{1}{n_t}\sum_{j=1}^{n_t} \sigma_{tj}f_t(\mathbf{x}_{tj}')\right]
		+\mathbb{E}_{\mathbf{\sigma},\mathbf{X}} \left[ \sup_{\mathbf{f} \in \mathcal{F}^{\otimes T}} \frac{1}{T}\sum_{t=1}^T \frac{1}{n_t}\sum_{j=1}^{n_t} -\sigma_{tj}f_t(\mathbf{x}_{tj})\right]\\ \label{eqr5}
		=& 2\mathbb{E}_{\mathbf{\sigma},\mathbf{X}} \left[ \sup_{\mathbf{f} \in \mathcal{F}^{\otimes T}} \frac{1}{T}\sum_{t=1}^T \frac{1}{n_t}\sum_{j=1}^{n_t} \sigma_{tj}f_t(\mathbf{x}_{tj})\right]=2 \hat{\mathfrak{R}}_{\mathbf{X}}(\mathcal{F}^{\otimes T}).
	\end{align}
	We introduce a ghost sample (denoted by $S'$), drawn from the same distribution as our original sample (denoted by $S$), and thus $\mathbb{E}[\mathbf{f}]=\mathbb{E}_{S'} [\hat{\mathbb{E}}_{S'}[\mathbf{f}]$ (Eq.(\ref{eqr1})). The inequality (\ref{eqr2}) holds due to the sub-additivity of the supremum function.
	In Eq. (\ref{eqr3}), we introduce Rademacher variables $\sigma_{tj}$, which are uniformly distributed independent random variables taking values in $\{-1,1\}$. Eq. (\ref{eqr4}) holds by the sub-additivity of the supremum function, and Eq. (\ref{eqr5}) stems from the definition of Rademacher complexity and the fact that the variables $\sigma_{tj}$ and $-\sigma_{tj}$ are
	distributed in the same way.
	
	By using again the McDiarmid's inequality, with probability $1-\delta/2$ the following holds
	\begin{align*}
		\mathfrak{R}_{\mathbf{X}}(\mathcal{F}^{\otimes T}) \leq \hat{\mathfrak{R}}_{\mathbf{X}}(\mathcal{F}^{\otimes T}) + \frac{B}{T} \sqrt{\sum_{t=1}^T \frac{1}{n_t}} \sqrt{\frac{\log \frac{2}{\delta}}{2}}.
	\end{align*}
	Finally, we use the union bound to combine above inequalities, which yields that with probability at least $1-\delta$ it holds:
	\begin{align}
		\mathbb{E}[\mathbf{f}]\leq \hat{\mathbb{E}}_{S}[\mathbf{f}]  + 2\hat{\mathfrak{R}}_{\mathbf{X}}(\mathcal{F}^{\otimes T}) + 3\frac{B}{T} \sqrt{\sum_{t=1}^T \frac{1}{n_t}} \sqrt{\frac{\log \frac{2}{\delta}}{2}}.
	\end{align}
\end{proof}

\begin{remark}
	When $n_t = n $ for $t=1,2,\cdots,T$, the conclusion becomes
	\begin{align*}
		\mathbb{E}[\mathbf{f}]\leq \hat{\mathbb{E}}_{\mathbf{X}}[\mathbf{f}]  + 2\hat{\mathfrak{R}}_{\mathbf{X}}(\mathcal{F}^{\otimes T}) + 3B \sqrt{\frac{\log \frac{2}{\delta}}{2nT}},
	\end{align*}
	where $\hat{\mathbb{E}}_{\mathbf{X}}[\mathbf{f}] =\frac{1}{nT}\sum_{t=1}^T \sum_{j=1}^{n}  f_t(\mathbf{x}_{ti}) $. This recovers the result in Theorem 9 of \citep{maurer2016benefit}.
\end{remark}
\begin{remark}
	When $T=1$, it recovers the result of the single task case \citep{mohri2018foundations}.
\end{remark}

\subsection{A Chain Rule for Gaussian Complexity $\mathcal{F}^{\otimes T}(\mathcal{H})$}
Suppose we have the training task dataset ${\bm{D}} = \{D_t, t \in [T]\}$, with $D_t = (D_t^{tr},D_t^{val})$, where $D_t^{tr} = \{ (x_{ti}^{(s)},y_{ti}^{(s)})\}_{i=1}^{m_t}, D_t^{val} = \{(x_{tj}^{(q)},y_{tj}^{(q)})\}_{j=1}^{n_t}$.
In this section, we aim to decouple the complexity of learning the class $\mathcal{F}^{\otimes T}(\mathcal{H})$ based on the chain rule tools \citep{maurer2016chain,wainwright2019high,tripuraneni2020theory}.
We let $f_t^{(h)}=\mathcal{LM}(D_t^{tr};h)$ in the following for simplicity.

The empirical Gaussian complexity of function class $\mathcal{F}^{\otimes T}(\mathcal{H})$ can be written as
\begin{align}
	\hat{\mathcal{G}}_{D^{val}}(\mathcal{F}^{\otimes T}_{\mathcal{H}}) = \mathbb{E} \sup_{f\in\mathcal{F}^{\otimes T},h\in \mathcal{H}} \frac{1}{T} \sum_{t=1}^{T} \frac{1}{n_t}\sum_{j=1}^{n_t} g_{tj} f_t^{(h)}(x^{(q)}_{tj}); \ g_{tj}\sim \mathcal{N}(0,1),
\end{align}
where $D^{val} = \{D_t^{val}\}_{t=1}^T$.
Generally, the Ledoux-Talagrand contraction principle in Proposition \ref{proposition1} can be applied to decoupling the complexity. However, we learn the meta-learner $h$ and the task-specific learners $f_t, t\in [T]$ simultaneously after observing the data, where learners are not fixed.

To present the result, we firstly list some important theorems for the proof.
\begin{theorem}[Dudley's entropy integral bound \citep{wainwright2019high}]\label{themdudley}
	Let $\{\mathbf{X}_{\vartheta}, \vartheta\in \mathbb{T}\}$  be a zero-mean sub-Gaussian process with respect to induced pseudometric $\rho_{2,\mathbf{X}}$, and $Dis(\mathbf{X})=\sup_{\vartheta,\vartheta' \in \mathbb{T}} \rho_{2,\mathbf{X}}(\vartheta,\vartheta')$. Then for any $\delta \in [0,E]$, we have
	\begin{align*}
		E\left[\sup_{\vartheta,\vartheta' \in \mathbb{T}} (\mathbf{X}_{\vartheta}-\mathbf{X}_{\vartheta'})\right]\leq 2\mathbf{E}\left[\sup_{\gamma,\gamma' \in \mathbb{T} ,\rho_{2,\mathbf{X}}(\gamma,\gamma')\leq \delta } (\mathbf{X}_{\gamma}-\mathbf{X}_{\gamma'})\right] + 32 \int_{\delta/4}^E \sqrt{\log N_{2,\mathbf{X}}(\mu,\rho_{2,\mathbf{X}},\mathbb{T})}d\mu,
	\end{align*}
	where $N_{2,\mathbf{X}}(\mu,\rho_{2,\mathbf{X}},\mathbb{T})$ denotes the $\mu$-covering number  of $\mathbb{T}$ in the metric $\rho_{2,\mathbf{X}}$.
\end{theorem}

\begin{theorem}[Sudakov minoration \citep{wainwright2019high}] \label{themsudakov}
	Let $\{\mathbf{X}_{\vartheta}, \vartheta\in \mathbb{T}\}$  be a zero-mean sub-Gaussian process defined on the non-empty set $\mathbb{T}$. Then
	\begin{align*}
		\mathbb{E}\left[\sup_{\vartheta\in \mathbb{T}}\mathbf{X}_{\vartheta}\right]\geq \sup_{\delta>0} \frac{\delta}{2}	\sqrt{\log M_{2,\mathbf{X}}(\delta,\rho_{2,\mathbf{X}},\mathbb{T})},
	\end{align*}
	where $M_{2,\mathbf{X}}(\delta,\rho_{2,\mathbf{X}},\mathbb{T})$ is the $\delta$-packing number of $\mathbb{T}$ in the metric $\rho_{2,\mathbf{X}}$. 
\end{theorem}

Now, the following gives the decomposition theorm for Gaussian complexity $\hat{\mathcal{G}}_{D^{val}}(\mathcal{F}^{\otimes T}_{\mathcal{H}})$.

\begin{theorem}\label{theorem4}
	Let the function class $\mathcal{F}$ consist of functions that are $\ell_2$-Lipschitz with constant $L(\mathcal{F})$. Define $Dis(D^{val}) = \sup_{h,h'} \rho_{2,D^{val}} (\mathbf{f}^{(h)},\mathbf{f}^{(h')})$, $\mathbf{f}^{(h)} = (f_1^{(h)},f_2^{(h)},\cdots, f_T^{(h)})$, and then the empirical Gaussian complexity of function class $\mathcal{F}^{\otimes T}_{\mathcal{H}}$ satisfies
	\begin{align*}
		\hat{\mathcal{G}}_{D^{val}}(\mathcal{F}^{\otimes T}_{\mathcal{H}}) \leq \frac{2{Dis(D^{val})}}{(\sum_{t=1}^T n_t)^2} \sqrt{\sum_{t=1}^T \frac{1}{\beta_t T} } + 128\log(4\sum_{t=1}^T n_t)\cdot L(\mathcal{F})\cdot \hat{\mathcal{G}}_{D^{val}}(\mathcal{H}),
	\end{align*}
	where $\rho_{2,D^{val}} (\mathbf{f}^{(h)},\mathbf{f}^{(h')})=\frac{1}{\sum_{t=1}^T n_t} \sum_{i=1}^{T}\sum_{j=1}^{n_t} (f_t^{(h)}(\mathbf{z}_{tj}^{(q)})-f_t^{(h')}(\mathbf{z}_{tj}^{(q)}))^2$, $\mathbf{f}^{(h)} = (f_1^{(h)}, f_2^{(h)}, \cdots, f_T^{(h)})$, $f_t^{(h)} = \mathcal{LM}(D_t^{tr};h)$, and $L(\mathcal{F})$ is the Lipschitz constant of $f^{(h)}$ with respect to $h$.
\end{theorem}

\begin{proof}
	For ease of notation we define $N = nT = \sum_{t=1}^{T} n_t, n_t=\beta_tn$ in the following.
	We define the mean-zero stochastic process $Z_{\mathbf{f}^{(h)}}= \frac{1}{\sqrt{\sum_{t=1}^T n_t}} \sum_{i=1}^{T} \sum_{j=1}^{n_t}  g_{tj}/\beta_t \cdot f_t^{(h)}(x_{tj}^{(q)})$
	for a sequence of data points $x_{tj}^{(q)}$. Note that the process $Z_{\mathbf{f}^{(h)}}$ has sub-Gaussian increments, and then $Z_{\mathbf{f}^{(h)}}-Z_{\mathbf{f}^{(h')}}$ is a sub-Gaussian random variable with parameter $\rho_{2,D^{val}} (\mathbf{f}^{(h)},\mathbf{f}^{(h')})= \frac{1}{\sum_{t=1}^T n_t}  \sum_{i=1}^{T} \sum_{j=1}^{n_t} (f_t^{(h)}(x_{tj}^{(q)})-f_t^{(h)}(x_{tj}^{(q)}))^2$. Since $Z_{\mathbf{f}^{(h)}}$ is a mean-zero stochastic process, we have
	\begin{align*}
		\mathbb{E}[\sup_{\mathbf{f}^{(h)}\in \mathcal{F}^{\otimes T}_{\mathcal{H}}} Z_{\mathbf{f}^{(h)}}] = \mathbb{E}[\sup_{\mathbf{f}^{(h)}\in \mathcal{F}^{\otimes T}_{\mathcal{H}}} Z_{\mathbf{f}^{(h)}}-Z_{\mathbf{f}^{(h')}}] \leq \mathbb{E}[\sup_{\mathbf{f}^{(h)},\mathbf{f}^{(h')}\in \mathcal{F}^{\otimes T}_{\mathcal{H}}} Z_{\mathbf{f}^{(h)}}-Z_{\mathbf{f}^{(h')}}].
	\end{align*}
	According to the Dudley entropy integral bound (Theorem \ref{themdudley}), we have
	\begin{align*}
		& \ \ \ \ \ \mathbb{E}\left[\sup_{\mathbf{f}^{(h)},\mathbf{f}^{(h')}\in \mathcal{F}^{\otimes T}_{\mathcal{H}}} Z_{\mathbf{f}^{(h)}}-Z_{\mathbf{f}^{(h')}}\right] \\
		\leq &2\mathbb{E}\left[\sup_{\substack{\mathbf{f}^{(h)},\mathbf{f}^{(h')}\in \mathcal{F}^{\otimes T}_{\mathcal{H}} \\ \rho_{2,\mathbf{Z}}(\mathbf{f}^{(h)},\mathbf{f}^{(h')})\leq \delta}} Z_{\mathbf{f}^{(h)}}-Z_{\mathbf{f}^{(h')}}\right] + 32\int_{\delta/4}^{Dis(D^{val})} \sqrt{\log N_{D^{val}}(\mu;\rho_{2,D^{val}},\mathcal{F}^{\otimes T}_{\mathcal{H}})}d\mu.
	\end{align*}
	The first term in the right hand can be computed as:
	\begin{align*}
		& \ \mathbb{E} \sup_{\substack{\mathbf{f}^{(h)},\mathbf{f}^{(h')}\in \mathcal{F}^{\otimes T}_{\mathcal{H}} \\		
				\rho_{2,D^{val}}(\mathbf{f}^{(h)},\mathbf{f}^{(h')})\leq \delta}} Z_{\mathbf{f}^{(h)}}-Z_{\mathbf{f}^{(h')}} \\
		= & \mathbb{E} \sup_{\substack{\mathbf{f}(h),\mathbf{f}(h')\in \mathcal{F}^{\otimes T}_{\mathcal{H}} \\		
				\rho_{2,D^{val}}(\mathbf{f}^{(h)},\mathbf{f}^{(h')})\leq \delta}} \frac{1}{\sqrt{\sum_{t=1}^T n_t}} \sum_{i=1}^{T} \sum_{j=1}^{n_t}  g_{tj}/\beta_t \cdot (f_t^{(h)}(x^{(q)}_{tj})-f_t^{(h')}(x_{tj}^{(q)}))\\
		\leq &  \mathbb{E} \sup_{\mathbf{v}:\|\mathbf{v}\|_2\leq \delta} \sqrt{ \sum_{t=1}^T\sum_{j=1}^{n_t}  (g_{ij}/\beta_t})^2 \cdot \|\mathbf{v}\|  \\
		\leq & \sqrt{ \sum_{t=1}^T\sum_{j=1}^{n_t} \frac{1}{\beta_t^2}} \cdot \delta \leq \delta \sqrt{n\sum_{t=1}^T \frac{1}{\beta_t} } .
	\end{align*}
	Let $C_{\mathcal{H}}$ be a covering of the function space $\mathcal{H}$ in the empirical $\ell_2$-norm at scale $\epsilon_1$ with respect to $D^{val}$. We will claim that $C_{\mathcal{F}^{\otimes T}_\mathcal{H}}=\cup_{h\in \mathcal{H}} C_{\mathcal{F}_h^{\otimes T}}$ is an $\epsilon_1 \cdot L(\mathcal{F})$-covering for the function space $\mathcal{F}^{\otimes T}_\mathcal{H}$ in the empirical $\ell_2$-norm. To see this, for any $h\in\mathcal{H}, \mathbf{f}\in \mathcal{F}^{\otimes T}$, we let $h'\in C_{\mathcal{H}}$ be $\epsilon_1$-close to $h$ in $\mathcal{H}$. By construction we have $\mathbf{f}^{(h')} \in C_{\mathcal{F}^{\otimes T}_{\mathcal{H}}}$. The following inequality estabalishes the claim,
	\begin{align*}
		\rho_{2,D^{val}}(\mathbf{f}^{(h)}, \mathbf{f}^{(h')}) =&{\frac{1}{\sum_{t=1}^{T} n_t}\sum_{t=1}^{T}\sum_{j=1}^{n_t} (f_t^{(h)}(x^{(q)}_{tj})-f_t^{(h')}(x^{(q)}_{tj}))^2}\\
		\leq & L(\mathcal{F})\cdot  {\frac{1}{\sum_{t=1}^T n_t }\sum_{t=1}^{T}\sum_{j=1}^{n_t}  (h(D_t^{val})-h'(D_t^{val}))^2}\\
		=& L(\mathcal{F})\cdot\rho_{2,D^{val}}(h, h') \leq \epsilon_1\cdot L(\mathcal{F}),
	\end{align*}
	where the first inequality holds since the Lipschitz property of the learning method $\mathcal{LM}$ with respect to $h$.
	Now, the cardinaity of $C_{\mathcal{F}_{\mathcal{H}}^{\otimes T}}$ can be bounded as
	\begin{align*}
		|C_{\mathcal{F}_{\mathcal{H}}^{\otimes T}}|=\sum_{h\in C_{\mathcal{H}}} | C_{\mathcal{F}_h}^{\otimes T}|\leq |C_{\mathcal{H}}|\cdot \max_{h\in C_{\mathcal{H}}} |C_{\mathcal{F}_h}^{\otimes T}|.
	\end{align*}
	Thus, it follows that
	\begin{align*}
		\log N_{2,D^{val}}(\epsilon_1\cdot L(\mathcal{F}), \rho_{2,D^{val}},C_{\mathcal{F}_h}^{\otimes T}) \leq \log N_{2,D^{val}} (\epsilon_1,\rho_{2,D^{val}},\mathcal{H}).
	\end{align*}
	We define $\epsilon_1 = \frac{\epsilon}{L(\mathcal{F})}$ to show that
	\begin{align*}
		\int_{\delta/4}^{Dis(D^{val})} \sqrt{\log N_{2,D^{val}}(\epsilon, \rho_{2,D^{val}}, C_{\mathcal{F}^{\otimes T}_{\mathcal{H}}})d\epsilon} \leq \int_{\delta/4}^{Dis(D^{val})} \sqrt{\log N_{2,D^{val}}(\epsilon/(L(\mathcal{F})), \rho_{2,D^{val}},\mathcal{H})d\epsilon}.
	\end{align*}
	Observing that the covering number can be bounded by packing number $M(\epsilon, \rho_{2,D^{val}},\mathcal{H})$, i.e., $N(\epsilon, \rho_{2,D^{val}},\mathcal{H}) \leq M(\epsilon, \rho_{2,D^{val}},\mathcal{H})$, we can then employ Sudakov minoration Theorem (Theorem \ref{themsudakov}) to upper bound the covering number by Gaussian complexity. For the covering number of $\mathcal{H}$, $N_{2,D^{val}}(\epsilon/(L(\mathcal{F})), \rho_{2,D^{val}},\mathcal{H})$, we can apply the theorem with mean-zero Gaussian process $\frac{1}{\sqrt{\sum_{t=1}^T n_t}} \sum_{t=1}^{T}\sum_{j=1}^{n_t} g_{tj} \cdot h(D_t^{val})$, with $g_{tj}\sim \mathcal{N}(0,1)$,
	\begin{align*}
		\log N_{2,D^{val}}(\epsilon/(L(\mathcal{F})), \rho_{2,D^{val}},\mathcal{H})d\epsilon \leq 4\left(\frac{\sqrt{\sum_{t=1}^T n_t}\hat{\mathcal{G}}_{D^{val}}(\mathcal{H})}{\epsilon/(L(\mathcal{F}))}\right)^2=4 \left(\frac{\sqrt{\sum_{t=1}^T n_t} L(\mathcal{F})\hat{\mathcal{G}}_{D^{val}}(\mathcal{H})}{\epsilon}\right)^2.
	\end{align*}
	Finally, we can show that
	\begin{align*}
		&\ \ \ \ \  \ \ \ \ \ \hat{\mathcal{G}}_{D^{val}}(\mathcal{F}^{\otimes T}_{\mathcal{H}}) =  \mathbb{E} [\frac{1}{\sqrt{n T}}\sup_{\mathbf{f}\in\mathcal{F}^{\otimes T},h\in \mathcal{H}} Z_{\mathbf{f}(h)}]  \\
		&\leq \frac{1}{\sqrt{n T}}\left(2\delta \sqrt{n\sum_{t=1}^T \frac{1}{\beta_t} } +64L(\mathcal{F})\cdot \hat{\mathcal{G}}_{D^{val}}(\mathcal{H}) \cdot \sqrt{\sum_{t=1}^T n_t} \int_{\delta/4}^{Dis(D^{val})} \frac{1}{\mu} d\mu \right) \\
		&\leq 2\delta \sqrt{\sum_{t=1}^T \frac{1}{\beta_t T} }    +64\log(\frac{4{Dis(D^{val})}}{\delta}) \left(L(\mathcal{F})\cdot \hat{\mathcal{G}}_{D^{val}}(\mathcal{H}) \sqrt{\frac{\sum_{t=1}^{T}n_t }{n T}}\right)\\
		&=2\delta \sqrt{\sum_{t=1}^T \frac{1}{\beta_t T} }  +64\log(\frac{4{Dis(D^{val})}}{\delta})\cdot L(\mathcal{F})\cdot \hat{\mathcal{G}}_{D^{val}}(\mathcal{H}).
	\end{align*}
	Let $\delta=\frac{Dis(D^{val})}{(n T)^2}$, we have
	\begin{align*}
		\hat{\mathcal{G}}_{D^{val}}(\mathcal{F}^{\otimes T}_{\mathcal{H}}) \leq & \frac{2{Dis(D^{val})}}{(n T)^2} \sqrt{\sum_{t=1}^T \frac{1}{\beta_t T} } + 128\log(4n T)\cdot L(\mathcal{F})\cdot \hat{\mathcal{G}}_{D^{val}}(\mathcal{H})\\
		= & \frac{2{Dis(D^{val})}}{(\sum_{t=1}^T n_t)^2} \sqrt{\sum_{t=1}^T \frac{1}{\beta_t T} } + 128\log(4\sum_{t=1}^T n_t)\cdot L(\mathcal{F})\cdot \hat{\mathcal{G}}_{D^{val}}(\mathcal{H}).
	\end{align*}
\end{proof}

%

\newpage

\section{Proof of Theorem \ref{the1} : Excess Task Avergae Risk in Meta-Training Stage}\label{average}

\begin{proof}
	Recall that 	
	\begin{align*}
		R_{\Gamma}(h) &=  \frac{1}{T}\sum_{t=1}^{T} \mathbb{E}_{D_t^{val}\sim (\mu_t^q)^{n_t}}\mathbb{E}_{D_t^{tr}\sim (\mu_t^s)^{m_t}}\bm{L} (\mathcal{LM}(D_t^{tr};h), D_t^{val}),\\
		\hat{R}_{\bm{D}}(h) &=  \frac{1}{T}\sum_{t=1}^{T} \bm{L} (\mathcal{LM}(D_t^{tr};h), D_t^{val}),
	\end{align*}
	and we denote $\tilde{R}_{\Gamma}(h),\overline{R}_{\Gamma}(h)$ as
	\begin{align*}
		\tilde{R}_{\Gamma}(h) & =  \frac{1}{T}\sum_{t=1}^{T} \mathbb{E}_{D_t^{val}\sim (\mu_t^q)^{n_t}} \bm{L} (\mathcal{LM}(D_t^{tr};h), D_t^{val})\\
		\overline{R}_{\Gamma}(h) & =  \frac{1}{T}\sum_{t=1}^{T} \mathbb{E}_{D_t^{tr}\sim (\mu_t^s)^{m_t}} \bm{L} (\mathcal{LM}(D_t^{tr};h), D_t^{val}).
	\end{align*}

	For the task average excess risk $R_{train}(\hat{\mathbf{f}}, \hat{h})-R_{train}(\mathbf{f}^*, h^*)$, we have the following decomposition
	\begin{align*}
		&R_{train}( \hat{h})-R_{train}( h^*)\\
		=&R_{\Gamma}(\hat{h})-R_{\Gamma}(h^*)\\
		=&  \underbrace{R_{\Gamma}(\hat{h}) - \overline{R}_{\Gamma}(\hat{h})}_{a}+ \underbrace{\overline{R}_{\Gamma}(\hat{h}) - \hat{R}_{\Gamma}(\hat{h})}_{b} + \underbrace{\hat{R}_{\Gamma}(\hat{h})-\hat{R}_{\Gamma}(h^*)}_{c} + \underbrace{\hat{R}_{\Gamma}(h^*) - \overline{R}_{\Gamma}(h^*)}_{d}  + \underbrace{\overline{R}_{\Gamma}(h^*) - {R}_{\Gamma}(h^*)}_{e}.
	\end{align*}

	For the terms (a) and (e), we have
	\begin{align*}
		&(a)+(e)\\
		\leq & \sup_{h\in \mathcal{H}}2\left|\frac{1}{T}\sum_{t=1}^{T} \mathbb{E}_{D_t^{val}\sim (\mu_t^q)^{n_t}}\mathbb{E}_{D_t^{tr}\sim (\mu_t^s)^{m_t}}\bm{L} (\mathcal{LM}(D_t^{tr};h), D_t^{val}) - \frac{1}{T}\sum_{t=1}^{T} \mathbb{E}_{D_t^{tr}\sim (\mu_t^s)^{m_t}} \bm{L} (\mathcal{LM}(D_t^{tr};h), D_t^{val}) \right| \\
		\leq &4L \hat{\mathfrak{R}}_{D^{val}}(\mathcal{F}^{\otimes T}_{\mathcal{H}}) + 3\frac{B}{T} \sqrt{\sum_{t=1}^T \frac{1}{n_t}} \sqrt{\frac{\log \frac{2}{\delta}}{2}} \leq  6L \hat{\mathcal{G}}_{D^{val}}(\mathcal{F}^{\otimes T}_{\mathcal{H}}) + 6\frac{B}{T} \sqrt{\sum_{t=1}^T \frac{1}{n_t}} \sqrt{\frac{\log \frac{2}{\delta}}{2}} \\
		\leq &  768L\log(4\sum_{t=1}^T n_t)\cdot L(\mathcal{F})\cdot \hat{\mathcal{G}}_{D^{val}}(\mathcal{H}) + \frac{12L{Dis(D^{val})}}{(\sum_{t=1}^T n_t)^2} \sqrt{\sum_{t=1}^T \frac{1}{\beta_t T} } + 6\frac{B}{T} \sqrt{\sum_{t=1}^T \frac{1}{n_t}} \sqrt{\frac{\log \frac{2}{\delta}}{2}},
	\end{align*}
	where $D^{val} = \{D_t^{val}\}_{t=1}^T$, and $|D_t^{val}| = n_t$, for $t \in [T]$. The second inequality we use the task-averaged estimation error in Theorem \ref{thaverage} and the Ledoux-Talagrand contraction principle in Proposition \ref{proposition1}. Proposition \ref{proposition2} is employed for the third inequallity. And the last inequality holds for Theorem \ref{theorem4}.
	
	For the term (b), we have
	\begin{align*}
		&\overline{R}_{\Gamma}(\hat{h}) - \hat{R}_{\Gamma}(\hat{h}) \\
		=& 	  \frac{1}{T}\sum_{t=1}^{T} \mathbb{E}_{D_t^{tr}\sim (\mu_t^s)^{m_t}} \bm{L} (\mathcal{LM}(D_t^{tr};\hat{h}), D_t^{val})-\frac{1}{T}\sum_{t=1}^{T} \bm{L} (\mathcal{LM}(D_t^{tr};\hat{h}), D_t^{val})  \\
		=& 	 \underbrace{\frac{1}{T}\sum_{t=1}^{T} \mathbb{E}_{D_t^{tr}\sim (\mu_t^s)^{m_t}} \bm{L} (\mathcal{LM}(D_t^{tr};\hat{h}), D_t^{val}) -  \frac{1}{T}\sum_{t=1}^{T} \mathbb{E}_{D_t^{tr}\sim (\mu_t^s)^{m_t}} \bm{L} (\mathcal{LM}(D_t^{tr};\hat{h}), D_t^{tr})}_{b1} \\
		+& \underbrace{ \frac{1}{T}\sum_{t=1}^{T} \mathbb{E}_{D_t^{tr}\sim (\mu_t^s)^{m_t}} \bm{L} (\mathcal{LM}(D_t^{tr};\hat{h}), D_t^{tr}) - \frac{1}{T}\sum_{t=1}^{T}  \bm{L} (\mathcal{LM}(D_t^{tr};\hat{h}), D_t^{tr})}_{b2} \\
		+& \underbrace{\frac{1}{T}\sum_{t=1}^{T}  \bm{L} (\mathcal{LM}(D_t^{tr};\hat{h}), D_t^{tr}) - \frac{1}{T}\sum_{t=1}^{T} \bm{L} (\mathcal{LM}(D_t^{tr};\hat{h}), D_t^{val}) }_{b3}.
	\end{align*}
	
	For the terms $(b1)$ and $(b3)$, we have
	\begin{align*}
		(b1)+(b3) \leq \frac{2}{T} \sum_{t=1}^{T} d_{\mathcal{F}}  ({\mu}_t^{s},{\mu}_t^{q}),
	\end{align*}
	where $d_{\mathcal{F}}  ({\mu}_t^{s},{\mu}_t^{q})$ denotes the discrepancy divergence \citep{ben2010theory} between samples from the probability distributions ${\mu}_t^{s}$ and ${\mu}_t^{q}$ with respect to the hypothesis class $\mathcal{F}$,
	\begin{align*}
		d_{\mathcal{F}} ({\mu}_t^{s},{\mu}_t^{q}) =   \sup_{f \in \mathcal{F}} \left| \mathbb{E}_{D_{t}^{val}\sim (\mu_t^q)^{n_t}}\bm{L} (f, D_t^{val}) - \mathbb{E}_{D_{t}^{tr}\sim (\mu_t^s)^{m_t}}\bm{L} (f, D_t^{tr})  \right|.
	\end{align*}
	
	For the term $(b2)$, we denote $\hat{f}_t^{(h^*)} = \mathcal{LM}(D_t^{tr};h^*)$, and then we have
	\begin{align*}
		b2 &= \frac{1}{T}\sum_{t=1}^{T}  \bm{L} (\hat{f}_t^{(h^*)}, D_t^{tr}) - \frac{1}{T}\sum_{t=1}^{T} \mathbb{E}_{D_t^{tr}\sim (\mu_t^s)^{m_t}} \bm{L} (\hat{f}_t^{(h^*)}, D_t^{tr})  \\
		&\leq \frac{1}{T}\sum_{t=1}^{T} \left(2L \hat{\mathfrak{R}}_{D_t^{tr}}(\mathcal{F}) + 3B\sqrt{\frac{\log \frac{2}{\delta}}{m_t}} \right) \leq  \frac{3L}{T}\sum_{t=1}^{T} \hat{\mathcal{G}}_{D_t^{tr}}(\mathcal{F})+ \frac{3B}{T}\sum_{t=1}^{T} \sqrt{\frac{\log \frac{2}{\delta}}{m_t}}.
	\end{align*}
	Thus we have
	\begin{align}
		b \leq \frac{2}{T} \sum_{t=1}^{T}  d_{\mathcal{F}}  ({\mu}_t^{s},{\mu}_t^{q}) +  \frac{3L}{T}\sum_{t=1}^{T} \hat{\mathcal{G}}_{D_t^{tr}}(\mathcal{F})+ \frac{3B}{T}\sum_{t=1}^{T} \sqrt{\frac{\log \frac{2}{\delta}}{m_t}}.
	\end{align}	
	Similar process can also be applied to the term $(d)$, and $(d)$ can be bounded by
	\begin{align}
		d \leq \frac{2}{T} \sum_{t=1}^{T}  d_{\mathcal{F}}  ({\mu}_t^{s},{\mu}_t^{q}) +  \frac{3L}{T}\sum_{t=1}^{T} \hat{\mathcal{G}}_{D_t^{tr}}(\mathcal{F})+ \frac{3B}{T}\sum_{t=1}^{T} \sqrt{\frac{\log \frac{2}{\delta}}{m_t}}.
	\end{align}	
	
	For the term (c), according to $\hat{h} = \arg\min_{h\in \mathcal{H}}\hat{R}_{\bm{D}}(h)$, we have $c \leq 0$.
	
	Combining the above results from the terms $(a)$ to $(e)$, we have
	%
	\begin{align*}
		&R_{train}( \hat{h})-R_{train}( h^*) \leq 768L \log(4\sum_{t=1}^T n_t)\cdot L(\mathcal{F})\cdot \hat{\mathcal{G}}_{\mathbf{\Gamma}^{(q)}}(\mathcal{H})+ \frac{6L}{T}\sum_{t=1}^{T} \hat{\mathcal{G}}_{D_t^{tr}}(\mathcal{F})\\
		&+ \frac{4}{T} \sum_{t=1}^{T} d_{\mathcal{F}} (D_t^{(tr)},D_t^{(val)}) + 6\frac{B}{T} \sqrt{\sum_{t=1}^T \frac{1}{n_t}} \sqrt{\frac{\log \frac{2}{\delta}}{2}}  + \frac{6B}{T}\sum_{t=1}^{T} \sqrt{\frac{\log \frac{2}{\delta}}{m_t}} + \frac{12L{Dis(D^{val})}}{(\sum_{t=1}^T n_t)^2} \sqrt{\sum_{t=1}^T \frac{1}{\beta_t T} }.
	\end{align*}
\end{proof}

%

\section{Proof of Theorem \ref{theomtest}: Excess Transfer Error in Meta-Test Stage}

\begin{proof}
	Recall that
	\begin{tiny}
		\begin{align*}
			& R_{test}(\hat{h}) - R_{test}(h^*) \\
			= \ & \mathbb{E}_{\mu \sim \eta}\mathbb{E}_{D_{\mu}^{val}\sim (\mu^q)^{n_{\mu}}} \bm{L}(\mathcal{LM}(D_{\mu}^{tr}; \hat{h})),  D_{\mu}^{val}) - \mathbb{E}_{\mu \sim \eta}\mathbb{E}_{D_{\mu}^{val}\sim (\mu^q)^{n_{\mu}}} \bm{L}(\mathcal{LM}(D_{\mu}^{tr}; h^*)),  D_{\mu}^{val}) \\
			=\ & \underbrace{\mathbb{E}_{\mu \sim \eta}\mathbb{E}_{D_{\mu}^{val}\sim (\mu^q)^{n_{\mu}}} \bm{L}(\mathcal{LM}(D_{\mu}^{tr}; \hat{h}), D_{\mu}^{val}) - \mathbb{E}_{\mu \sim \eta}\mathbb{E}_{D_{\mu}^{val}\sim (\mu^q)^{n_{\mu}}}\mathbb{E}_{D_{\mu}^{tr}\sim (\mu^s)^{m_{\mu}}} \bm{L}(\mathcal{LM}(D_{\mu}^{tr}; \hat{h}), D_{\mu}^{val})}_{a} \\
			+\ & \underbrace{\mathbb{E}_{\mu \sim \eta}\mathbb{E}_{D_{\mu}^{val}\sim (\mu^q)^{n_{\mu}}}\mathbb{E}_{D_{\mu}^{tr}\sim (\mu^s)^{m_{\mu}}} \bm{L}(\mathcal{LM}(D_{\mu}^{tr}; \hat{h}), D_{\mu}^{val})  - \mathbb{E}_{\mu \sim \eta}\mathbb{E}_{D_{\mu}^{val}\sim (\mu^q)^{n_{\mu}}}\mathbb{E}_{D_{\mu}^{tr}\sim (\mu^s)^{m_{\mu}}} \bm{L}(\mathcal{LM}(D_{\mu}^{tr}; h^*), D_{\mu}^{val})}_{b}\\
			+\ & \underbrace{\mathbb{E}_{\mu \sim \eta}\mathbb{E}_{D_{\mu}^{val}\sim (\mu^q)^{n_{\mu}}}\mathbb{E}_{D_{\mu}^{tr}\sim (\mu^s)^{m_{\mu}}} \bm{L}(\mathcal{LM}(D_{\mu}^{tr}; h^*), D_{\mu}^{val})  - \mathbb{E}_{\mu \sim \eta}\mathbb{E}_{D_{\mu}^{val}\sim (\mu^q)^{n_{\mu}}} \bm{L}(\mathcal{LM}(D_{\mu}^{tr}; h^*)),  D_{\mu}^{val})}_{c}.
		\end{align*}	
	\end{tiny}
	We now bound the first term (a) by
	\begin{tiny}
		\begin{align*}
			(a) =& \underbrace{ \mathbb{E}_{\mu \sim \eta}\mathbb{E}_{D_{\mu}^{val}\sim (\mu^q)^{n_{\mu}}} \bm{L}(\mathcal{LM}(D_{\mu}^{tr}; \hat{h})),  D_{\mu}^{val}) - \mathbb{E}_{\mu \sim \eta} \mathbb{E}_{D_{\mu}^{tr'}\sim (\mu^s)^{m_{\mu}}}\bm{L} (\mathcal{LM}(D_{\mu}^{tr};\hat{h}), D_{\mu}^{tr'})}_{a1} \\
			+ & \underbrace{\mathbb{E}_{\mu\sim \eta} \mathbb{E}_{D_{\mu}^{tr'}\sim (\mu^s)^{m_{\mu}}}\bm{L} (\mathcal{LM}(D_{\mu}^{tr};\hat{h}), D_{\mu}^{tr'}) - \mathbb{E}_{\mu\sim \eta} \bm{L}(\mathcal{LM}(D_{\mu}^{tr}; \hat{h}), D_{\mu}^{tr'})}_{a2} \\
			+& \underbrace{\mathbb{E}_{\mu\sim \eta} \bm{L}(\mathcal{LM}(D_{\mu}^{tr}; \hat{h}), D_{\mu}^{tr'}) - \mathbb{E}_{\mu\sim\eta} \mathbb{E}_{D_{\mu}^{tr}\sim (\mu^s)^{m_{\mu}}}\bm{L}(\mathcal{LM}(D_{\mu}^{tr}; \hat{h}), D_{\mu}^{tr'})}_{a3}\\
			+&\underbrace{ \mathbb{E}_{\eta} \mathbb{E}_{D_{\mu}^{tr}\sim (\mu^s)^{m_{\mu}}}\bm{L}(\mathcal{LM}(D_{\mu}^{tr}; \hat{h}), D_{\mu}^{tr'}) - \mathbb{E}_{\eta} \mathbb{E}_{D_{\mu}^{tr}\sim (\mu^s)^{m_{\mu}}} \mathbb{E}_{D_{\mu}^{tr'}\sim (\mu^s)^{m_{\mu}}} \bm{L}(\mathcal{LM}(D_{\mu}^{tr}; \hat{h}), D_{\mu}^{tr'})  }_{a4}\\
			+ & \underbrace{\mathbb{E}_{\eta} \mathbb{E}_{D_{\mu}^{tr}\sim (\mu^s)^{m_{\mu}}} \mathbb{E}_{D_{\mu}^{tr'}\sim (\mu^s)^{m_{\mu}}} \bm{L}(\mathcal{LM}(D_{\mu}^{tr}; \hat{h}), D_{\mu}^{tr'}) - \mathbb{E}_{\mu \sim \eta} \mathbb{E}_{D_{\mu}^{tr}\sim (\mu^s)^{m_{\mu}}} \mathbb{E}_{D_{\mu}^{val}\sim (\mu^q)^{n_{\mu}}} \bm{L}(\mathcal{LM}(D_{\mu}^{tr}; \hat{h}), D_{\mu}^{val})}_{a5},
		\end{align*}
	\end{tiny}
	where $D_t^{tr'}$ is equivalent to $D_t^{tr}$.
	
	For the terms $(a1)$ and $(a5)$, we have
	\begin{align*}
		(a1)+(a5) \leq  2 \mathbb{E}_{\mu \sim \eta} d_{\mathcal{F}} ({\mu}^{s},{\mu}^{q}),
	\end{align*}
	where $d_{\mathcal{F}} ({\mu}^{s},{\mu}^{q})$ denotes the discrepancy divergence \citep{ben2010theory} between samples from the probability distributions ${\mu}^{s}$ and ${\mu}^{q}$ with respect to the hypothesis class $\mathcal{F}$,
	\begin{align*}
		d_{\mathcal{F}} ({\mu}^{s},{\mu}^{q}) =   \sup_{f \in \mathcal{F}} \left| \mathbb{E}_{D_{\mu}^{val} \sim (\mu^q)^{n_{\mu}}}\bm{L} (f, D_t^{val}) - \mathbb{E}_{D_{\mu}^{tr} \sim (\mu^s)^{m_{\mu}}}\bm{L} (f, D_t^{tr})  \right|.
	\end{align*}
	
	For the term $(a2)$, we have
	\begin{align*}
		(a2) 	\leq 2L\hat{\mathfrak{R}}_{D_{\mu}^{tr}}(\mathcal{F}) + 3B\sqrt{\frac{\log \frac{2}{\delta}}{m_{\mu}}}\leq  3L\hat{\mathcal{G}}_{D_{\mu}^{tr}}(\mathcal{F})+ 3B\sqrt{\frac{\log \frac{2}{\delta}}{m_{\mu}}}.
	\end{align*}
	
	Similarly, For the term $(a4)$, we have
	\begin{align*}
		(a4) \leq  3L\hat{\mathcal{G}}_{D_{\mu}^{tr}}(\mathcal{F})+ 3B\sqrt{\frac{\log \frac{2}{\delta}}{m_{\mu}}}.
	\end{align*}
	
	For the term $(a3)$, suppose that the outer loss can be upper bounded by the inner loss, and then we have $a3<0$ according to the definition of $\mathcal{LM}$ function, i.e., $\mathcal{LM}(D_{\mu}^{tr};\hat{h}(D_{\mu}^{tr}))$ minimizes the $\bm{L}(f,D_{\mu}^{tr})$ equipped with hyper-parameters $\hat{h}(D_{\mu}^{tr})$.
	
	Therefore, combining the above analysis, the term $(a)$ can be upper bounded as
	\begin{align*}
		(a) \leq 2 \mathbb{E}_{\mu \sim \eta} d_{\mathcal{F}} ({\mu}^{s},{\mu}^{q})  + 6L\hat{\mathcal{G}}_{D_{\mu}^{tr}}(\mathcal{F})+ 6B\sqrt{\frac{\log \frac{2}{\delta}}{m_{\mu}}}.
	\end{align*}	
	
	For the term $(b)$, according to the task diversity definition (Assumption \ref{assumption4}), we have
	\begin{align*}
		(b) = R_{\eta}(\hat{h}) - R_{\eta}(h^*) \leq \alpha\left(R_{train}(\hat{h}) - R_{train}(h^*))\right)+\beta.
	\end{align*}
	
	For the term $(c)$, according to definition of $h^*$, we have $(c)\leq 0$.

	Combining above results from term $(a)$ to $(c)$, we have
	\begin{align*}
		& R_{test}(\hat{h}) - R_{test}(h^*)  \\
		\leq & \alpha \left(R_{train}(\hat{\mathbf{f}}, \hat{h})-R_{train}(\mathbf{f}^*, h^*)\right)+ \beta+ 6L\hat{\mathcal{G}}_{D_{\mu}^{tr}}(\mathcal{F})+ 2 \mathbb{E}_{\mu \sim \eta} d_{\mathcal{F}} ({\mu}^{s},{\mu}^{q}) +  6B\sqrt{\frac{\log \frac{2}{\delta}}{m_{\mu}}}.
	\end{align*}
	
\end{proof}

\section{Proof of the Proposition \ref{prop4}}
\begin{proof}
	We denote
	$$\tilde{\mathbf{w}} = \arg\min_{\mathbf{w}} \mathbb{E}_{(x^{(s)},y^{(s)}) \in \mu^s} \ell(\mathbf{w}^{T} \hat{h}(x^{(s)}),y^{(s)})$$ and
	$${\mathbf{w}}^* = \arg\min_{\mathbf{w}} \mathbb{E}_{(x^{(s)},y^{(s)}) \in \mu^s} \ell(\mathbf{w}^{T} h^*(x^{(s)}),y^{(s)}).$$
	Observe that
	\begin{align*}
		&\ \ \ \ \ R_{\eta}(\hat{h}) - R_{\eta}(h^*)  \\
		&=\mathbb{E}_{\mu \sim \eta}\mathbb{E}_{D_{\mu}^{val}\sim (\mu^q)^{n}}\mathbb{E}_{D_{\mu}^{tr}\sim (\mu^s)^{m}} \bm{L}(\mathcal{LM}(D_{\mu}^{tr}; \hat{h}), D_{\mu}^{val})  - \\
		&\ \ \ \ \ \ \ \ \ \ \ \ \ \ \  \mathbb{E}_{\mu \sim \eta}\mathbb{E}_{D_{\mu}^{val}\sim (\mu^q)^{n}}\mathbb{E}_{D_{\mu}^{tr}\sim (\mu^s)^{m}} \bm{L}(\mathcal{LM}(D_{\mu}^{tr}; h^*), D_{\mu}^{val})\\
		&=\mathbb{E}_{\mu\sim\eta} \mathbb{E}_{x\sim \mu^q} \left\{\left|\tilde{\mathbf{w}}^{\mathsf{T}}\hat{h}(x) - \mathbf{w}^{*\mathsf{T}}h^*(x)\right|^2\right\}, \\
		&= \sup_{\mathbf{w}_0} \inf_{{\mathbf{w}}} \mathbb{E}_{\mu\sim\eta} \mathbb{E}_{x\sim \mu^q} \left\{\left|{\mathbf{w}}^{\mathsf{T}}\hat{h}(x) - \mathbf{w}_0^{\mathsf{T}}h^*(x)\right|^2\right\} \\
		&= \sup_{\mathbf{w}_0} \inf_{\tilde{\mathbf{w}}}  \left\{ [\tilde{\mathbf{w}}^{\mathsf{T}} \ -\mathbf{w}_0^{\mathsf{T}}] \Lambda \begin{bmatrix}\tilde{\mathbf{w}} \\  -\mathbf{w}_0 \end{bmatrix}\right\},
	\end{align*}
	where $\Lambda$ is defined as
	\begin{align*}
		\Lambda(\hat{h},h^*) = \begin{bmatrix} \mathbb{E}_{\mu\sim\eta}\mathbb{E}_{x\sim \mu^q}[\hat{h}(x)\hat{h}(x)^{\mathsf{T}}] & \mathbb{E}_{\mu\sim\eta}\mathbb{E}_{x\sim \mu^q}[\hat{h}(x)h^*(x)^{\mathsf{T}}] \\ \mathbb{E}_{\mu\sim\eta} \mathbb{E}_{x\sim \mu^q}[h^*(x)\hat{h}(x)^{\mathsf{T}}] & \mathbb{E}_{\mu\sim\eta}\mathbb{E}_{x\sim \mu^q}[h^*(x)h^*(x)^{\mathsf{T}}] \end{bmatrix} =
		\begin{bmatrix} \mathbf{G}_{\hat{h}\hat{h}} & \mathbf{G}_{\hat{h}h^*} \\ \mathbf{G}_{h^*\hat{h}} & \mathbf{G}_{h^*h^*}  \end{bmatrix}.
	\end{align*}
	According to the partial minimization of a convex quadratic form \citep{Boyd2004convex}, we have
	\begin{align*}
		\inf_{\tilde{\mathbf{w}}}  \left\{ [\tilde{\mathbf{w}}^{\mathsf{T}} \ -\mathbf{w}_0^{\mathsf{T}}] \Lambda \begin{bmatrix}\tilde{\mathbf{w}} \\  -\mathbf{w}_0 \end{bmatrix}\right\} = \mathbf{w}_0^{\mathsf{T}}  \Lambda_S(\hat{h},h^*) \mathbf{w}_0,
	\end{align*}
	where $\Lambda_S(\hat{h},h^*) = \mathbf{G}_{h^*h^*}- \mathbf{G}_{h^*\hat{h}} (\mathbf{G}_{\hat{h}\hat{h}})^{\dagger} \mathbf{G}_{\hat{h}}h^*$ is the generalized Schur complement of the representation of $h^*$ with respect to $\hat{h}$.
	Furthermore, according to the variational characterization of the singular value, we have
	\begin{align*}
		\sup_{\mathbf{w}_0: \|\mathbf{w}_0\| \leq M } \mathbf{w}_0^{\mathsf{T}}  \Lambda_S \mathbf{w}_0  = M \sigma_1(\Lambda_S(\hat{h},h^*)),
	\end{align*}
	where $\sigma_1$ denotes the maximal singular value.

	Moreover, we denote
	$$\tilde{\mathbf{w}}_t = \arg\min_{\mathbf{w}_t} \mathbb{E}_{(x_t^{(s)},y_t^{(s)}) \in \mu_t^s} \ell(\mathbf{w}_t^{T} \hat{h}(x_t^{(s)}),y_t^{(s)}), t\in [T]$$ and 	$${\mathbf{w}}_t^* = \arg\min_{\mathbf{w}_t} \mathbb{E}_{(x_t^{(s)},y_t^{(s)}) \in \mu^s} \ell(\mathbf{w}_t^{T} h^*(x_t^{(s)}),y_t^{(s)}), t\in [T].$$
	Thus we have the following derivation:	
	\begin{align*}
		& \ \ \ \ \ R_{train}(\hat{h}) - R_{train}(h^*) \\
		&= \frac{1}{T}\sum_{t=1}^{T} \mathbb{E}_{D_t^{val}\sim (\mu_t^q)^{n}}\mathbb{E}_{D_t^{tr}\sim (\mu_t^s)^{m}}\bm{L} (\mathcal{LM}(D_t^{tr};\hat{h}), D_t^{val}) - \\
		&\ \ \ \ \ \ \ \ \ \ \ \ \ \ \ \frac{1}{T}\sum_{t=1}^{T} \mathbb{E}_{D_t^{val}\sim (\mu_t^q)^{n}}\mathbb{E}_{D_t^{tr}\sim (\mu_t^s)^{m}}\bm{L} (\mathcal{LM}(D_t^{tr};h^*), D_t^{val})\\
		&=\frac{1}{T} \sum_{t=1}^T  \mathbb{E}_{\mathbf{\mu}_t\sim \eta}\mathbb{E}_{x\sim \mu_t^q}  \left\{\left|\tilde{\mathbf{w}}_t^{\mathsf{T}}\hat{h}(x) - \mathbf{w}_t^{*T}h^*(x)\right|^2\right\}  \\
		&=\frac{1}{T} \sum_{t=1}^T \inf_{{\mathbf{w}}_t} \mathbb{E}_{\mathbf{\mu}_t\sim \eta}\mathbb{E}_{x\sim \mu_t^q} \left\{\left|{\mathbf{w}}_t^{\mathsf{T}}\hat{h}(x) - \mathbf{w}_t^{*T}h^*(x)\right|^2\right\}  \\
		&= \frac{1}{T} \sum_{t=1}^T   \mathbf{w}_t^{*T} \Lambda_S(\hat{h},h^*)  \mathbf{w}_t^{*} = tr(\Lambda_S(\hat{h},h^*) \mathbf{P}^{\mathsf{T}}\mathbf{P}/T) = tr(\Lambda_S(\hat{h},h^*) \mathbf{K}),
	\end{align*}
	where $\mathbf{K}=\mathbf{P}^{\mathsf{T}}\mathbf{P}/T$.
	Since $\Lambda_S \succeq 0$, and $\mathbf{K} \succeq 0$, through the Von-Neumann trace inequality, we have that
	\begin{align*}
		tr(\Lambda_S(\hat{h},h^*) \mathbf{K}) & \geq \sum_{i=1}^{d_L} \sigma_i(\Lambda_S(\hat{h},h^*)) \sigma_{d_L-i+1}(\mathbf{K}) \geq  \sum_{i=1}^{d_L} \sigma_i(\Lambda_S(\hat{h},h^*)) \sigma_{d_L}(\mathbf{K}) \\
		&= tr(\Lambda_S(\hat{h},h^*))\sigma_{d_L}(\mathbf{K}) \geq \sigma_1(\Lambda_S(\hat{h},h^*)) \sigma_{d_L}(\mathbf{K}).
	\end{align*}
	Now, we can deduce that
	\begin{align*}
		R_{\eta}(\hat{h}) - R_{\eta}(h^*) \leq \frac{M}{\sigma_{d_L}(\mathbf{K})} \left\{ R_{train}(\hat{h}) - R_{train}(h^*)\right\}.
	\end{align*}
	Based on the above derivation, we can have the conclusion.	
\end{proof}

\section{Proof of the Proposition \ref{prop2}}

\begin{proof}
	Denote $z_i(x)=f(h(x))_i$, the loss function can be written as
	\begin{align}
		\ell(f(h(x)),y) = -\sum_{i}y_i \log \left(\frac{e^{z_i}}{\sum_{k} e^{z_k}}\right).
	\end{align}
	
	Let $a_i = \frac{e^{z_i}}{\sum_{k} e^{z_k}}$,
	\begin{align}
		\frac{\partial \ell }{\partial z_i}= \sum_{j} \left(\frac{\partial \ell}{\partial a_j}\frac{\partial a_j}{\partial z_i}\right),
	\end{align}	
	where $\frac{\partial \ell}{\partial a_j} = -\frac{y_j}{a_j}$, for $\frac{\partial a_j}{\partial z_i}$,
	if $i=j$, $\frac{\partial a_j}{\partial z_i} = \frac{\partial (\frac{e^{z_i}}{\sum_{k} e^{z_k}}) }{\partial z_i}= \frac{e^{z_i}\sum_k e^{z_k}-(e^{z_i})^2}{(\sum_k e^{z_k})^2} = a_i(1-a_i)$; otherwise, $i\neq j$, $\frac{\partial a_j}{\partial z_i} =\frac{\partial (\frac{e^{z_j}}{\sum_{k} e^{z_k}}) }{\partial z_i} = \frac{-e^{z_i}e^{z_j}}{(\sum_k e^{z_k})^2}=-a_ia_j$.
	Therefore,
	\begin{align*}
		\frac{\partial \ell }{\partial z_i}&= -\sum_{j\neq i} (\frac{y_j}{a_j}(-a_ia_j)) - \frac{y_j}{a_j}(1-a_j) \\
		&=a_i \sum_{j} y_j  -y_i =a_i-y_i,
	\end{align*}
	since $a_i \in [0,1], y_i \in \{0,1\}$, we can obtain $\frac{\partial \ell }{\partial z_i} \in [-1,1]$.
	Therefore, we can demonstrate that the loss function is 1-Lipschitz with respect to $f(h(X))$.
\end{proof}

\section{Proof of the Proposition \ref{prop3}}
We first present some essential theoretical results preparing for proving the Proposition \ref{prop3}.
\begin{lemma}
	Consider the following generalized linear model
	\begin{align} \label{linear}
		p(\bm{y}|\bm{\eta})=h(\bm{y})\exp(\bm{\eta}^T t(\bm{y})-a(\bm{\eta})),
	\end{align}
	Then we have
	\begin{align}
		(\hat{\bm{\eta}}-\bm{\eta})^{\mathsf{T}}\frac{a''(\bm{c}_1)}{2} 	(\hat{\bm{\eta}}-\bm{\eta}) \leq KL(p(\bm{y}|\bm{\eta})|p(\bm{y}|\hat{\bm{\eta}})) \leq (\hat{\bm{\eta}}-\bm{\eta})^{\mathsf{T}}\frac{a''(\bm{c}_2)}{2} 	(\hat{\bm{\eta}}-\bm{\eta}),
	\end{align}
	where $\bm{c}_1 = \inf_{\bm{c}\in [\hat{\bm{\eta}},\bm{\eta}]} a''(\bm{c}), \bm{c}_2=\sup_{\bm{c}\in [\hat{\bm{\eta}},\bm{\eta}]} a''(\bm{c})$.
\end{lemma}

\begin{proof}
	Observe that
	\begin{align*}
		& KL(p(\bm{y}|\bm{\eta})|p(\bm{y}|\bm{\hat{\eta}}))=\int (p(\bm{y}|\bm{\eta})\log \frac{(p(\bm{y}|\bm{\eta})}{(p(\bm{y}|\hat{\bm{\eta}})} d\bm{y} \\
		=&\int p(\bm{y}|\bm{\eta})^{\mathsf{T}} [t(\bm{y})(\bm{\eta} - \hat{\bm{\eta}}) +a(\hat{\bm{\eta}})-a(\bm{\eta})] d\bm{y}\\
		=&  a'(\bm{\eta})^{\mathsf{T}}(\bm{\eta} - \hat{\bm{\eta}})+a(\hat{\bm{\eta}})-a(\bm{\eta}),
	\end{align*}
	where $a'(\bm{\eta}) = \int t(\bm{y})^{\mathsf{T}} p(\bm{y}|\bm{\eta})d\bm{y}$.
	Based on the Taylor's theorem we have that
	\begin{align*}
		a(\hat{\bm{\eta}})=a(\bm{\eta}) +a'(\bm{\eta})^{\mathsf{T}} (\hat{\bm{\eta}}-\bm{\eta}) + (\hat{\bm{\eta}}-\bm{\eta})^{\mathsf{T}}\frac{a''(\bm{c})}{2} (\hat{\bm{\eta}}-\bm{\eta}),
	\end{align*}
	where $\bm{c}\in [\hat{\bm{\eta}},\bm{\eta}]$. Therefore, we can obtain
	\begin{align*}
		KL(p(\bm{y}|\bm{\eta})|p(\bm{y}|\bm{\hat{\eta}})) = (\hat{\bm{\eta}}-\bm{\eta})^{\mathsf{T}}\frac{a''(\bm{c})}{2} (\hat{\bm{\eta}}-\bm{\eta}).
	\end{align*}
	We can then obtain
	\begin{align*}
		(\hat{\bm{\eta}}-\bm{\eta})^{\mathsf{T}}\frac{a''(\bm{c}_1)}{2} (\hat{\bm{\eta}}-\bm{\eta}) \leq KL(p(\bm{y}|\bm{\eta})|p(\bm{y}|\bm{\hat{\eta}})) \leq (\hat{\bm{\eta}}-\bm{\eta})^{\mathsf{T}}\frac{a''(\bm{c}_2)}{2} (\hat{\bm{\eta}}-\bm{\eta}),
	\end{align*}
	where $\bm{c}_1 = \inf_{\bm{c} \in [\hat{\bm{\eta}},\bm{\eta}]} (\hat{\bm{\eta}}-\bm{\eta})^{\mathsf{T}} a''(\bm{c}) (\hat{\bm{\eta}}-\bm{\eta}), \bm{c}_2=\sup_{\bm{c}\in [\hat{\bm{\eta}},\bm{\eta}]} (\hat{\bm{\eta}}-\bm{\eta})^{\mathsf{T}} a''(\bm{c}) (\hat{\bm{\eta}}-\bm{\eta}) $.
\end{proof}
\begin{remark} \rm
	If the data generating model satisfies the conditional likelihood defined in Eq. (\ref{distribution}), for the cross-entropy loss we have
	\begin{align*}
		&\mathbb{E}_{(x,y) \in D_t^{val}} [\ell(f(h(x)),y) - \ell(f(h(x)),y)] \\
		=& \mathbb{E}_x [KL(\text{Multi}(\text{Softmax}(f(h(x))))| KL(\text{Multi}(\text{Softmax}(f(h(x))))],
	\end{align*}
	where $\text{Multi}$ denote the multinomial distribution.
	
	When consider the multinomial distribution, the generalized linear model in Eq. (\ref{linear}) satisfies that $h(\bm{y})=1, t(\bm{y}) =\bm{y}, a(\bm{\eta})= \log(\sum_{k=1}^{K} \exp(\eta_k)), \bm{\eta} = (\eta_1,\cdots,\eta_K)^{\mathsf{T}}$. Then, we have
	\begin{align*}
		a'(\bm{\eta}) = \left( \frac{\exp(\eta_1)}{\sum_{k=1}^{K} \exp(\eta_k)}, \cdots,  \frac{\exp(\eta_{K})}{\sum_{k=1}^{K} \exp(\eta_k)}        \right)^{\mathsf{T}},
	\end{align*}
	\begin{equation*}
		a''(\bm{\eta}) = \left(
		\begin{array}{cccc}
			\frac{\exp(\eta_1)}{S}  - \frac{\exp(\eta_1)^2}{S^2}  & - \frac{\exp(\eta_1)\exp(\eta_2)}{S^2} &   \cdots &  - \frac{\exp(\eta_1)\exp(\eta_{K-1})}{S^2} \\
			-\frac{\exp(\eta_2)\exp(\eta_1)}{S^2}  & \frac{\exp(\eta_2)}{S}  - \frac{\exp(\eta_2)^2}{S^2} &   \cdots &  - \frac{\exp(\eta_2)\exp(\eta_{K})}{S^2}   \\
			\vdots & \vdots & \vdots & \vdots  \\
			-\frac{\exp(\eta_{K})\exp(\eta_1)}{S^2}  &  -\frac{\exp(\eta_{K})\exp(\eta_2)}{S^2}   & \cdots &  \frac{\exp(\eta_K)}{S}  - \frac{\exp(\eta_K)^2}{S^2}  \\
		\end{array}
		\right),
	\end{equation*}
	where $S =[\sum_{k=1}^{K} \exp(\eta_k)]$.
	We denote $p_i = \exp(\eta_i) /S$, Then we have
	\begin{align*}
		& \mathbb{E}_{(x,y) \in D_t^{val}} [\ell(f(h(x)),y) - \ell(f(h(x)),y)] =  (\hat{\bm{\eta}}-\bm{\eta})^{\mathsf{T}}\frac{a''(\bm{c})}{2} (\hat{\bm{\eta}}-\bm{\eta})\\
		= & \sum_{k=1}^{K} \left[p_k-p_k^2\right]\eta_k^2 - \sum_{k=1}^{K} \sum_{j=1,j\neq k}^{K}  p_kp_j \eta_k \eta_j \\
		\geq & \min_{k,j\in[K],k\neq j} \{p_k-p_k^2, p_kp_j\}  \|\hat{\bm{\eta}}-\bm{\eta} \|_2^2	:=  C \|\hat{\bm{\eta}}-\bm{\eta} \|_2^2.
	\end{align*}

\end{remark}

{\noindent \textbf{Proof of Proposition \ref{prop3}}}
\begin{proof}
	We denote
	$$\tilde{\bm{A}} = \arg\min_{\bm{A}} \mathbb{E}_{(x^{(s)},y^{(s)})\sim \mu^s} \ell(\bm{A}^{\mathsf{T}} \hat{h}(x^{(s)}),y^{(s)})$$ and
	$$\bm{A}^{*} = \arg\min_{\bm{A}} \mathbb{E}_{(x^{(s)},y^{(s)})\sim \mu^s} \ell(\bm{A}^{\mathsf{T}} h^*(x^{(s)}),y^{(s)}).$$
	Observe that
	\begin{align*}
		&\ \ \ \ \ R_{\eta}(\hat{h}) - R_{\eta}(h^*)  \\
		&= \frac{1}{T}\sum_{t=1}^{T} \mathbb{E}_{D_t^{val}\sim (\mu_t^q)^{n}}\mathbb{E}_{D_t^{tr}\sim (\mu_t^s)^{m}}\bm{L} (\mathcal{LM}(D_t^{tr};\hat{h}), D_t^{val}) - \\
		&\ \ \ \ \ \ \ \ \ \ \ \ \ \ \ \frac{1}{T}\sum_{t=1}^{T} \mathbb{E}_{D_t^{val}\sim (\mu_t^q)^{n}}\mathbb{E}_{D_t^{tr}\sim (\mu_t^s)^{m}}\bm{L} (\mathcal{LM}(D_t^{tr};h^*), D_t^{val})\\
		&=\mathbb{E}_{\mu\sim\eta} \mathbb{E}_{(x,y)\sim \mu^q} \left\{ \ell(\tilde{\bm{A}}^{\mathsf{T}}\hat{h}(x),y) - \ell({\bm{A}}^{*\mathsf{T}}h^*(x),y)  \right\} \\
		& \leq    \mathbb{E}_{\mu\sim\eta} \mathbb{E}_{x\sim \mu^q} \|\tilde{\bm{A}}^{\mathsf{T}}\hat{h}(x) - {\bm{A}}^{*\mathsf{T}}h^*(x)\|_2^2 \\
		& =    \mathbb{E}_{\mu\sim\eta} \mathbb{E}_{x\sim \mu^q} \left\{ \sum_{k=1}^K \left|(\tilde{\bm{A}})_k^{\mathsf{T}}\hat{h}(x) - ({\bm{A}})_k^{*\mathsf{T}}h^*(x)\right|^2 \right\}\\
		&= \sup_{\bm{A}'} \inf_{{\bm{A}}} \mathbb{E}_{\mu\sim\eta} \mathbb{E}_{x\sim \mu^q}\left\{ \sum_{k=1}^K \left|({\bm{A}})_k^{\mathsf{T}}\hat{h}(x) - ({\bm{A}})_k^{'\mathsf{T}}h^*(x)\right|^2 \right\} \\
		&= \sum_{k=1}^K   \left\{ \sup_{(\bm{A})_k'} \inf_{{(\bm{A})_k}} [({\bm{A}})_k^{\mathsf{T}} \ -({\bm{A}})_k^{'\mathsf{T}}] \Lambda \begin{bmatrix}({\bm{A}})_k \\  -({\bm{A}})_k^{'} \end{bmatrix}\right\},
	\end{align*}
	where $\Lambda$ is defined as
	\begin{align*}
		\Lambda(\hat{h},h^*) = \begin{bmatrix} \mathbb{E}_{\mu\sim\eta}\mathbb{E}_{x\sim \mu^q}[\hat{h}(x)\hat{h}(x)^{\mathsf{T}}] & \mathbb{E}_{\mu\sim\eta}\mathbb{E}_{x\sim \mu^q}[\hat{h}(x)h^*(x)^{\mathsf{T}}] \\ \mathbb{E}_{\mu\sim\eta} \mathbb{E}_{x\sim \mu^q}[h^*(x)\hat{h}(x)^{\mathsf{T}}] & \mathbb{E}_{\mu\sim\eta}\mathbb{E}_{x\sim \mu^q}[h^*(x)h^*(x)^{\mathsf{T}}] \end{bmatrix} =
		\begin{bmatrix} \mathbf{G}_{\hat{h}\hat{h}} & \mathbf{G}_{\hat{h}h^*} \\ \mathbf{G}_{h^*\hat{h}} & \mathbf{G}_{h^*h^*}  \end{bmatrix}.
	\end{align*}
	The first inequality holds since the 1-Lipschitz continuity of the cross-entropy loss.
	According to the partial minimization of a convex quadratic form \citep{Boyd2004convex}, we have
	\begin{align*}
		\inf_{({\bm{A})_k}}  \left\{ [({\bm{A}})_k^{\mathsf{T}} \ -({\bm{A}})_k^{'\mathsf{T}}] \Lambda \begin{bmatrix}({\bm{A}})_k \\  -({\bm{A}})_k^{'} \end{bmatrix}\right\} = ({\bm{A}})_k^{'\mathsf{T}}  \Lambda_S(\hat{h},h^*) ({\bm{A}})_k^{'},
	\end{align*}
	where $\Lambda_S(\hat{h},h^*) = \mathbf{G}_{h^*h^*}- \mathbf{G}_{h^*\hat{h}} (\mathbf{G}_{\hat{h}\hat{h}})^{\dagger} \mathbf{G}_{\hat{h}}h^*$ is the generalized Schur complement of the representation of $h^*$ with respect to $\hat{h}$.
	Furthermore, according to the variational characterization of the singular value, we have
	\begin{align*}
		\sum_{k=1}^K  \sup_{({\bm{A}})_k^{'\mathsf{T}}: \|({\bm{A}})_k^{'\mathsf{T}}\| \leq M_k } ({\bm{A}})_k^{'\mathsf{T}}  \Lambda_S(\hat{h},h^*) ({\bm{A}})_k^{'}  = \sum_{k=1}^K  M_k \sigma_1(\Lambda_S(\hat{h},h^*)) \leq  M \sigma_1(\Lambda_S(\hat{h},h^*)),
	\end{align*}
	where $\sigma_1$ denotes the maximal singular value.
	
	Moreover, we denote
	$$\tilde{\bm{A}}_t = \arg\min_{\bm{A}_t} \mathbb{E}_{(x_t^{(s)},y_t^{(s)})\sim \mu_t^s} \ell(\bm{A}_t^{\mathsf{T}} \hat{h}(x_t^{(s)}),y_t^{(s)}), t \in [T]$$ and
	$$\bm{A}_t^{*} = \arg\min_{\bm{A}_t} \mathbb{E}_{(x_t^{(s)},y_t^{(s)})\sim \mu_t^s} \ell(\bm{A}_t^{\mathsf{T}} h^*(x_t^{(s)}),y_t^{(s)}), t \in [T].$$ Thus we have the following derivation	
	\begin{align*}
		& \ \ \ \ \ R_{train}(\hat{h}) - R_{train}(h^*) \\
		&= \frac{1}{T}\sum_{t=1}^{T} \mathbb{E}_{D_t^{val}\sim (\mu_t^q)^{n}}\mathbb{E}_{D_t^{tr}\sim (\mu_t^s)^{m}}\bm{L} (\mathcal{LM}(D_t^{tr};\hat{h}), D_t^{val}) - \\
		&\ \ \ \ \ \ \ \ \ \ \ \ \ \ \ \frac{1}{T}\sum_{t=1}^{T} \mathbb{E}_{D_t^{val}\sim (\mu_t^q)^{n}}\mathbb{E}_{D_t^{tr}\sim (\mu_t^s)^{m}}\bm{L} (\mathcal{LM}(D_t^{tr};h^*), D_t^{val})\\
		&= \frac{1}{T}\sum_{t=1}^{T}  \mathbb{E}_{(x,y)\sim \mu^q_t} \left\{ \ell(\tilde{\bm{A}}_t^{\mathsf{T}}\hat{h}(x),y) - \ell({\bm{A}}_t^{*\mathsf{T}}h^*(x),y)  \right\} \\
		&\geq\frac{C}{T} \sum_{t=1}^T  \mathbb{E}_{\mathbf{\mu}_t\sim \eta}\mathbb{E}_{x\sim \mu_t^q}  \left\{\|\tilde{\bm{A}}_t^{\mathsf{T}}\hat{h}(x) - {\bm{A}}_t^{*\mathsf{T}}h^*(x)\|_2^2 \right\}  \\
		&=\frac{C}{T} \sum_{t=1}^T \inf_{{\bm{A}}_t} \mathbb{E}_{\mathbf{\mu}_t\sim \eta}\mathbb{E}_{x\sim \mu_t^q} \left\{\|\tilde{\bm{A}}_t^{\mathsf{T}}\hat{h}(x) - {\bm{A}}_t^{*\mathsf{T}}h^*(x)\|_2^2  \right\}  \\
		&=\frac{C}{T} \sum_{t=1}^T \inf_{{\bm{A}}_t} \mathbb{E}_{\mathbf{\mu}_t\sim \eta}\mathbb{E}_{x\sim \mu_t^q} \left\{ \sum_{k=1}^K |(\tilde{\bm{A}}_t)_k^{\mathsf{T}}\hat{h}(x) - ({\bm{A}}_t)_k^{*\mathsf{T}}h^*(x)|^2  \right\}  \\
		&= \frac{C}{T} \sum_{t=1}^T  \sum_{k=1}^K ({\bm{A}}_t)_k^{*\mathsf{T}} \Lambda_S(\hat{h},h^*)  ({\bm{A}}_t)_k^{*} = \sum_{k=1}^K tr(\Lambda_S(\hat{h},h^*) (\mathbf{Q})_k^{\mathsf{T}}(\mathbf{Q})_k/T) = \sum_{k=1}^K tr(\Lambda_S(\hat{h},h^*) (\mathbf{K})_k),
	\end{align*}
	where $(\mathbf{K})_k=(\mathbf{Q})_k^{\mathsf{T}}(\mathbf{Q})_k/T$.
	Since $\Lambda_S \succeq 0$, and $(\mathbf{K})_k \succeq 0$, through the Von-Neumann trace inequality, we have that
	\begin{align*}
		\sum_{k=1}^K tr(\Lambda_S(\hat{h},h^*) (\mathbf{K})_k) & \geq \sum_{k=1}^K \sum_{i=1}^{d_L} \sigma_i(\Lambda_S(\hat{h},h^*)) \sigma_{d_L-i+1}(\mathbf{(\mathbf{K})_k}) \geq  \sum_{k=1}^K\sum_{i=1}^{d_L} \sigma_i(\Lambda_S(\hat{h},h^*)) \sigma_{d_L}((\mathbf{K})_k) \\
		&= \sum_{k=1}^K tr(\Lambda_S(\hat{h},h^*))\sigma_{d_L}((\mathbf{K})_k) \geq \sigma_1(\Lambda_S(\hat{h},h^*)) \sum_{k=1}^K  \sigma_{d_L}((\mathbf{K})_k).
	\end{align*}
	Now, we can deduce that
	\begin{align*}
		R_{\eta}(\hat{h}) - R_{\eta}(h^*) \leq \frac{M}{\sum_{k=1}^K  \sigma_{d_L}((\mathbf{K})_k)} \left\{ R_{train}(\hat{h}) - R_{train}(h^*)\right\}.
	\end{align*}
	Based on the above derivation, the conclusion can then obtained.	
\end{proof}

\section{Proof of Online Methodology-Learning algorithm}
\subsection{Proof of Theorem \ref{regret}}
We firstly present Lemma \ref{lemma1} and then provide the proof of Theorem \ref{regret}.
\begin{lemma} \label{lemma1}
	If we set the step size $ \frac{2\tau }{ \mu \rho L'(\mathcal{F})} \leq \alpha \leq \min\{ \frac{1}{G}, \frac{1}{\beta L(\mathcal{F})}\}$, then $\ell_t(h)$ is convex, where $L'(\mathcal{F}) \|h-h'\| \leq \|f_t^{(h)} - f_t^{(h')}\| \leq L(\mathcal{F}) \|h-h'\| $. Furthermore, it is also $2\tau$-smooth and $\tau$-strongly convex. 
\end{lemma}
\begin{proof}
	Using the chain rule and our definitions, we have
	\begin{align*}
		\nabla \ell_t(h) -\nabla\ell_t(h') & = \frac{\partial f_t^{(h)}}{\partial h} \nabla \bm{L} (f_t^{(h)}, D_t^{val}) -\frac{\partial f_t^{(h')}}{\partial h} \nabla \bm{L} (f_t^{(h')}, D_t^{val})  \\
		& = \left(\frac{\partial f_t^{(h)}}{\partial h} -\frac{\partial f_t^{(h')}}{\partial h}\right) \nabla\bm{L} (f_t^{(h)}, D_t^{val}) + \frac{\partial f_t^{(h')}}{\partial h} \left(\nabla\bm{L} (f_t^{(h)}, D_t^{val}) -  \nabla \bm{L} (f_t^{(h')}, D_t^{val})\right)
	\end{align*}
	Taking the norm on both sides, and noticing that $f_t^{h_t} = f_{t} - \alpha \frac{\partial \mathcal{L}^{task}(f, h_t, D_t^{tr})}{\partial f}$, for the specified $\alpha$, we have:
	\begin{align*}
		& \ \ \ \ \|\nabla \ell_t(h) -\nabla\ell_t(h')\|  \\
		& \leq \left\|\left(\frac{\partial f_t^{(h)}}{\partial h} -\frac{\partial f_t^{(h')}}{\partial h} \right) \nabla\bm{L} (f_t^{(h)}, D_t^{val}) \right\| +  \left\|\frac{\partial f_t^{(h')}}{\partial h} \left(\nabla\bm{L} (f_t^{(h)}, D_t^{val}) -  \nabla \bm{L} (f_t^{(h')}, D_t^{val})\right) \right\| \\
		& \leq \alpha \tau G \|h-h'\|+ \alpha \tau \beta L(\mathcal{F}) \|h-h'\| \\
		& \leq \tau \|h-h'\| + \tau  \|h-h'\| =  2 \tau\|h-h'\|
	\end{align*}
	Similarly, we obtain the following lower bound
	\begin{align*}
		& \ \ \ \ \|\nabla \ell_t(h) -\nabla\ell_t(h')\|   \\
		& \geq \left\|\frac{\partial f_t^{(h')}}{\partial h} \left(\nabla\bm{L} (f_t^{(h)}, D_t^{val}) -  \nabla \bm{L} (f_t^{(h')}, D_t^{val})\right) \right\| - \left\|\left(\frac{\partial f_t^{(h)}}{\partial h} -\frac{\partial f_t^{(h')}}{\partial h} \right) \nabla\bm{L} (f_t^{(h)}, D_t^{val}) \right\| \\
		& \geq \alpha \mu L'(\mathcal{F})\rho \|h-h'\| -  \alpha \tau G \|h-h'\| \\
		& \geq 2  \tau \|h-h'\| - \tau  \|h-h'\|   =\tau   \|h-h'\|,
	\end{align*}
	which completes the proof.
\end{proof}
Then we present the proof of Theorem \ref{regret}.
\begin{proof}
	The FTML algorithm is identical to FTL on the sequence of loss functions $a_t, t \in [T]$, which has a $\frac{4A^2 \log(1+T)}{B}$ regret bound, where $a_t$ is $A$-Lipschitz, and $B$-strongly convex (see \citep{cesa2006prediction} Theorem 3.1). Using $A=G, B=\tau$ completes the proof.
\end{proof}

\section{Comparing Our Meta-Regulization with Related Works }
\subsection{Comparing Tanh with Neural Architecture Search} \label{tanh}
Recall the bounds that we instantiate our general-purpose bounds in Theorem \ref{theomtest} for few-shot regression in Section \ref{application},
\begin{align}\label{fewtest2}
	\begin{split}
		&  \ \ \ \ \ \ \ \ R_{test}(\hat{f}_{\mu},\hat{h}) - R_{test}(f^*_{\mu},h^*) \\
		& \ \ \ \ \leq \frac{M}{\sigma_{d_L}(\mathbf{K})} \left( 768L \log(4nT) L(\mathcal{F}) 2d_L \sqrt{\log(nT)}  \frac{R \left(\sqrt{2\log(2)L}+1\right)\prod\limits_{i=1}^L B_i}{\sqrt{nT}} \right. \\&\phantom{=\;\;}\left.
		+ \frac{6L\|\mathbf{w}\|}{\sqrt{m}T} \!\sum_{t=1}^T \! \max_{x_{ti}^{(s)} \!\in\! D_t^{val}} {\color{red}\|h(x_{ti}^{(s)} )\| } + 6B\! \sqrt{\frac{\log \frac{2}{\delta}}{2nT}}  + 6B\!\sqrt{\frac{\log \frac{2}{\delta}}{m}} + \frac{48L\sup_{h,x} M {\color{red}\|h(x)\|}}{n^2T^2}\right) \\
		&\ \ \ \ \ + \frac{6L\|\mathbf{w}\|}{\sqrt{m_{\mu}}} \cdot \max_{x_{i}^{(s)} \in D_{\mu}^{val}} {\color{red}\|h(x_{i}^{(s)} )\|}+ 6B\sqrt{\frac{\log \frac{2}{\delta}}{m_{\mu}}}.
	\end{split}
\end{align}
It can be seen that the transfer error defined in Eq.(\ref{fewtest2}) contains an important term $\|h(x)\|$, i.e., the output range of the meta-learner $h(x)$ in our few-shot regression setting defined as follows:
\begin{align}\label{eqmlp2}
	h(x) = \phi_{L}(\mathbf{W}_{L}(\phi_{L-1}(\mathbf{W}_{L-1}\cdots \phi_1 (\mathbf{W}_1 x)))).
\end{align}
Thus $\|h(x)\|$ can be bounded by:
\begin{align} \label{eqhbound2}
	\begin{split}
		&\| h(x)\| = \| \mathbf{r}_L(x)\|  = \|\phi_L (\mathbf{W}_{L} \mathbf{r}_{L-1}(x))\| \\
		\leq & \|\mathbf{W}_{L}\mathbf{r}_{L-1}(x)\| \leq \|\mathbf{W}_{L}\| \|\mathbf{r}_{L-1}(x)\| \leq D \prod_{k=1}^{L} \|\mathbf{W}_k\|,
	\end{split}
\end{align}
where $ \mathbf{r}_k(\cdot)$ denotes the vector-valued output of the $k$-layer for $k\in [L]$.
The above bound assumes that the activation function $\phi_k, k \in [L]$ is ReLU. \emph{If we revise the last activation function as Tanh (i.e., $\phi_{L}=\frac{e^z-e^{-z}}{e^z+e^{-z}}$), then we have}
\begin{align*}
	\| h(\mathbf{x})\| = \| \mathbf{r}_L(\mathbf{x})\|  = \|\phi_L (\mathbf{W}_{L} \mathbf{r}_{L-1}(\mathbf{x}))\| \leq \sqrt{d_L}.
\end{align*}
Generally, $\sqrt{d_L}$ is much smaller than $D \prod_{k=1}^{L} \|\mathbf{W}_k\|_F$, and the complexity can thus be expected to substantially decrease. This means that the generalization capability of the calculated meta-learner is hopeful to be improved by replacing the last activation function from ReLU with Tanh.

Recall that the structural risk minimization (SRM) principle in conventional machine learning is often used to improve generalization capability of the extracted learner through inducing some rational meta-regularization strategies. Similarly, changing the last activation function from ReLU as conventional to Tanh can be thought as a meta-regularization strategy for ameliorating the generalization capability of the extracted meta-learner. This meta-regularization strategy is completely inspired from our derived learning theory,
and this should be the first time to use such meta-regularization strategy in meta-learning, to the best of our knowledge.

Besides, we wang to clarify that our proposed method should be evidently different from the approaches that apply neural architecture search to meta-learning, e.g., \citep{lian2019towards,wang2020m}. Specifically, T-NAS \citep{lian2019towards} learns a meta-architecture that is capable of adapting to a new task quickly through a few gradient steps, and M-NAS \citep{wang2020m} proposes meta neural architecture search to learn a controller for architecture generation and then fast adapt to new tasks. Compared with current meta-learning methods using existing backbone networks, both of them attempt to learn a new architecture using neural architecture search techniques for the specific tasks. Comparatively, our method proposes to revise the activation function inspired from statistical learning theory perspective, rather than treat activation function as hyper-parameters to learn or tune from training data. This means that we explore a new orthogonal research direction to understand and improve meta-learning. Though
T-NAS \citep{lian2019towards}, M-NAS \citep{wang2020m} can obtain nice performance improvement, our method is designed upon more comprehensively solid theoretical guarantee and with more efficient computation implementation. In this sense, we believe that it is still rational to say our meta-learning framework and its statistical learning theory, as well as its deduced meta-regularization strategies, should be meaningful to the field.

\subsection{Comparing $L^2$-SP with Meta-MinibatchProx and iMAML} \label{l2sp}

Meta-MinibatchProx proposed in \citep{zhou2019efficient}, as well as iMAML proposed in \citep{rajeswaran2019meta}, mainly focuses on improving the computation efficiency of the well-known MAML algorithm.
The original MAML seeks to find the task-specific model by fine-tuning meta-parameters $\theta$ with stochastic gradient descent, i.e.,
\begin{align}
	\mathcal{A}lg^*(\theta,\mathcal{D}_i^{tr}) = \theta - \alpha \nabla  \bm{L} (\theta,\mathcal{D}_i^{tr}),
\end{align}
where $\alpha$ is the learning rate. To have sufficient learning in the task-specific model learning while also avoiding over-fitting,
Meta-MinibatchProx and iMAML consider the following proximal regularization for the task-specific model:
\begin{align} \label{proximal}
	\mathcal{A}lg^*(\theta,\mathcal{D}_i^{tr} ) = \mathop{\arg\min}_{\phi \in \Phi} \bm{L}(\phi, D_i^{tr}) + \frac{\lambda}{2} \|\phi-\theta\|^2,
\end{align}
where $\theta$ and $\phi$ are the meta-parameters and model-parameters, respectively. The overall bi-level meta-learning objective can be written as
\begin{align}
	\theta^* &= \mathop{\arg\min}_{\theta \in \Theta} F(\theta), \ \text{where}\ F(\theta)= \frac{1}{T} \sum_{i=1}^T \bm{L}_i (\mathcal{A}lg_i^*(\theta), \mathcal{D}_i^{test}),  \label{eqout}\\
	&\mathcal{A}lg_i^*(\theta) = \mathop{\arg\min}_{\phi \in \Phi} G(\phi, \theta), \ \text{where}\ G(\phi, \theta) = \bm{L}_i (\phi,\mathcal{D}_i^{tr}) + \frac{\lambda}{2} \|\phi-\theta\|^2,   \label{eqin}
\end{align}
where $\mathcal{A}lg_i^*(\theta) = \mathcal{A}lg^*(\theta,\mathcal{D}_i^{tr} )$.

Comparatively, our $L^2$-SP meta-regularizer is with evident differences with the above regularizers mainly from the following aspects:
\begin{itemize}
	\item \emph{Formulation.} The proximal regularization in Eq.(\ref{proximal}) encourages the model-parameters to remain close to $\theta$, which regulates the \textit{inner-level model capacity}. While our $L^2$-SP meta-regularizer minimizes the distance between the weights from the starting point weights of meta-model, which regulates the \textit{outer-level meta-model (meta-learner) capacity}. For example, if we use $L^2-SP$, the outer-level objective in Eq.(\ref{eqout},\ref{eqin}) can be written as
	\begin{align} \label{SP}
		\theta^* &= \mathop{\arg\min}_{\theta \in \Theta} \tilde{F}(\theta), \ \text{where}\ \tilde{F}(\theta)= \frac{1}{T} \sum_{i=1}^T \bm{L}_i (\mathcal{A}lg_i^*(\theta), \mathcal{D}_i^{test}) + \lambda \|\theta-\theta_0\|^2, \\
		&\mathcal{A}lg_i^*(\theta) = \mathop{\arg\min}_{\phi \in \Phi} G(\phi, \theta), \ \text{where}\ G(\phi, \theta) = \bm{L}_i (\phi,\mathcal{D}_i^{tr})
	\end{align}
	where $\theta_0$ is the starting point weights of $\theta$.
	
	\item \emph{Theory.}  The proximal regularization explores memory and computation efficient optimization steps via the efficient proximal learning from the \emph{optimization theory perspective}, and bring better results on few-shot learning benchmarks. Comparatively, our method induces the $L^2$-SP meta-regularization from the \emph{statistical learning theory perspective}. Specifically, the upper bound of Rademacher complexity for meta-model in Theorem 5 of our manuscript is given by
	\begin{align}
		\hat{\mathfrak{R}}_{N}(\mathcal{H}) \leq \frac{2\sqrt{2} d_L R \sum_{j=1}^L \frac{D_j}{2B_j \prod_{i=1}^j \sqrt{c_i} } \prod_{j=1}^L 2B_j \sqrt{c_j} }{\sqrt{N}},
	\end{align}
	where $ D_j$ controls the distance between the weights and the starting point weights of $j$-th layer, i.e., $ \|\mathbf{W}_j - \mathbf{W}_j^0\|_F \leq D_j, j\in [L]$. Therefore, to reduce the meta generalization error of meta-learner, we can minimize $\|\mathbf{W}_j - \mathbf{W}_j^0\|_F$ to reduce $D_j$, $j\in [L]$,	
	so as to improve the transferable generalization capability of meta-learner.
	
	\item \emph{Algorithm.} Compared with the original MAML, Meta-MinibatchProx and iMAML need to re-design inner-level algorithm for new formulation. Comparatively, our method can be easily integrated into existing meta-learning methods, and still use their out-level algorithms to update meta-learner, and do not revise the inner-level algorithm.
	
\end{itemize}

\newpage

\bibliography{sample}

\end{document}